\documentclass[journal]{IEEEtran}

\usepackage{amsmath}
\usepackage{amsfonts}
\usepackage{amssymb}

\usepackage{amsthm}
\usepackage[ruled,vlined]{algorithm2e}
\usepackage{mathtools}
\usepackage{tabularx}
\usepackage{graphicx}
\usepackage{subfigure}
\usepackage{enumerate}
\usepackage{float}
\usepackage{url}
\usepackage{verbatim}
\usepackage{cite}
\usepackage[linkcolor=black,citecolor=black,urlcolor=black,colorlinks=true]{hyperref}

\graphicspath{{figures/}}
\IEEEoverridecommandlockouts

\DeclareMathOperator*{\trace}{Tr}
\DeclareMathOperator*{\Diag}{Diag}
\newcommand{\tp}{^{\mathrm{T}}}
\newcommand{\invtp}{^{-\mathrm{T}}}
\newcommand{\grad}{\nabla}
\newcommand{\hessian}{\nabla^2}
\newcommand{\df}[1]{\mathrm{d}{#1}}
\newcommand{\rbrac}[1]{({#1})}
\newcommand{\rBrac}[1]{\left({#1}\right)}

\newcommand{\cbrac}[1]{\{{#1}\}}
\newcommand{\cBrac}[1]{\left\{{#1}\right\}}
\newcommand{\sbrac}[1]{[{#1}]}
\newcommand{\sBrac}[1]{\left[{#1}\right]}
\newcommand{\norm}[1]{\Vert{#1}\Vert}
\newcommand{\Norm}[1]{\left\Vert{#1}\right\Vert}
\newcommand{\abs}[1]{\vert{#1}\vert}

\newcommand{\ceil}[1]{\lceil{#1}\rceil}

\newtheorem{proposition}{Proposition}

\newtheorem{theorem}{Theorem}

\author{Zhepei Wang, Xin Zhou, Chao Xu, and Fei Gao
    \thanks{All authors are with the College of Control Science and Engineering, Zhejiang University, Hangzhou, 310027, China, and also with the Huzhou Institute of Zhejiang University, Huzhou, 313000, China. {\tt\small \{wangzhepei, iszhouxin, cxu, fgaoaa\}@zju.edu.cn} This work was supported by the National Natural Science Foundation of China under Grant 62003299 and Grant 62088101. \textit{(Corresponding authors: Fei Gao, Chao Xu.)}}
}

\title{Geometrically Constrained Trajectory Optimization \\ for Multicopters}

\begin{document}
    \maketitle

\begin{abstract}
In this article, we present an optimization-based framework for multicopter trajectory planning subject to geometrical configuration constraints and user-defined dynamic constraints. The basis of the framework is a novel trajectory representation built upon our novel optimality conditions for unconstrained control effort minimization. We design linear-complexity operations on this representation to conduct spatial-temporal deformation under various planning requirements. Smooth maps are utilized to exactly eliminate geometrical constraints in a lightweight fashion. A variety of state-input constraints are supported by the decoupling of dense constraint evaluation from sparse parameterization, and backward differentiation of flatness map. As a result, this framework transforms a generally constrained multicopter planning problem into an unconstrained optimization that can be solved reliably and efficiently. Our framework bridges the gaps among solution quality, planning efficiency, and constraint fidelity for a multicopter with limited resources and maneuvering capability. Its generality and robustness are both demonstrated by applications to different flight tasks. Extensive simulations and benchmarks are also conducted to show its capability of generating high-quality solutions while retaining the computation speed against other specialized methods by orders of magnitude. The source code of our framework is available at: \url{https://github.com/ZJU-FAST-Lab/GCOPTER}.
\end{abstract}


\section{Introduction}
\IEEEPARstart{M}{ulticopters} rely on robust and efficient trajectory planning for safe yet agile autonomous navigation in complex environments~\cite{Ryll2019EfficientTP, Oleynikova2020OpenMPF, Zhang2020Falco, Campos2021AutonomousMAVs, Zhou2021EgoPLANNER, Foehn2021Alphapilot}. For robotics, precisely incorporating dynamics, smoothness, and safety is essential to generate high-quality motions.
Moreover, lightweight robots, such as multicopters under SWaP (size, weight, and power) constraints, put further hard requirements on the real-time computing using limited onboard resources.
Despite that various successful tools in general-purpose kinodynamic planning or optimal control have been presented, few of them guarantee efficient online planning while also considering general constraints on dynamics for multicopters.
Consequently, existing applications often use oversimplified requirements on trajectories for better computation efficiency, thus limiting the full exploitation of vehicle's capability.

\begin{figure}[ht]
    \begin{center}
        \subfigure[\label{fig:HighSpeedSideView}High-speed flights in the garage.]
        {\includegraphics[width=0.49\columnwidth]{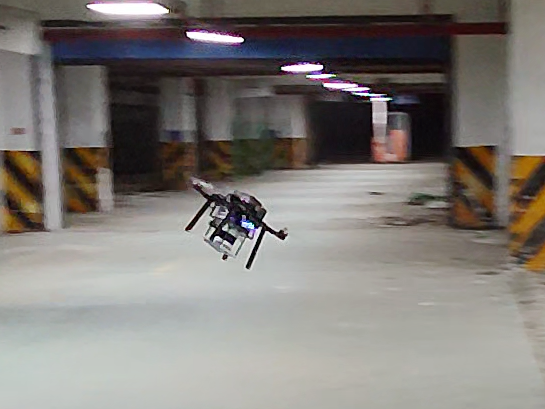}}%
        \hspace{0.02cm}
        \subfigure[\label{fig:Consecutive}$\mathrm{SE}(3)$ motions through windows.]
        {\includegraphics[width=0.49\columnwidth]{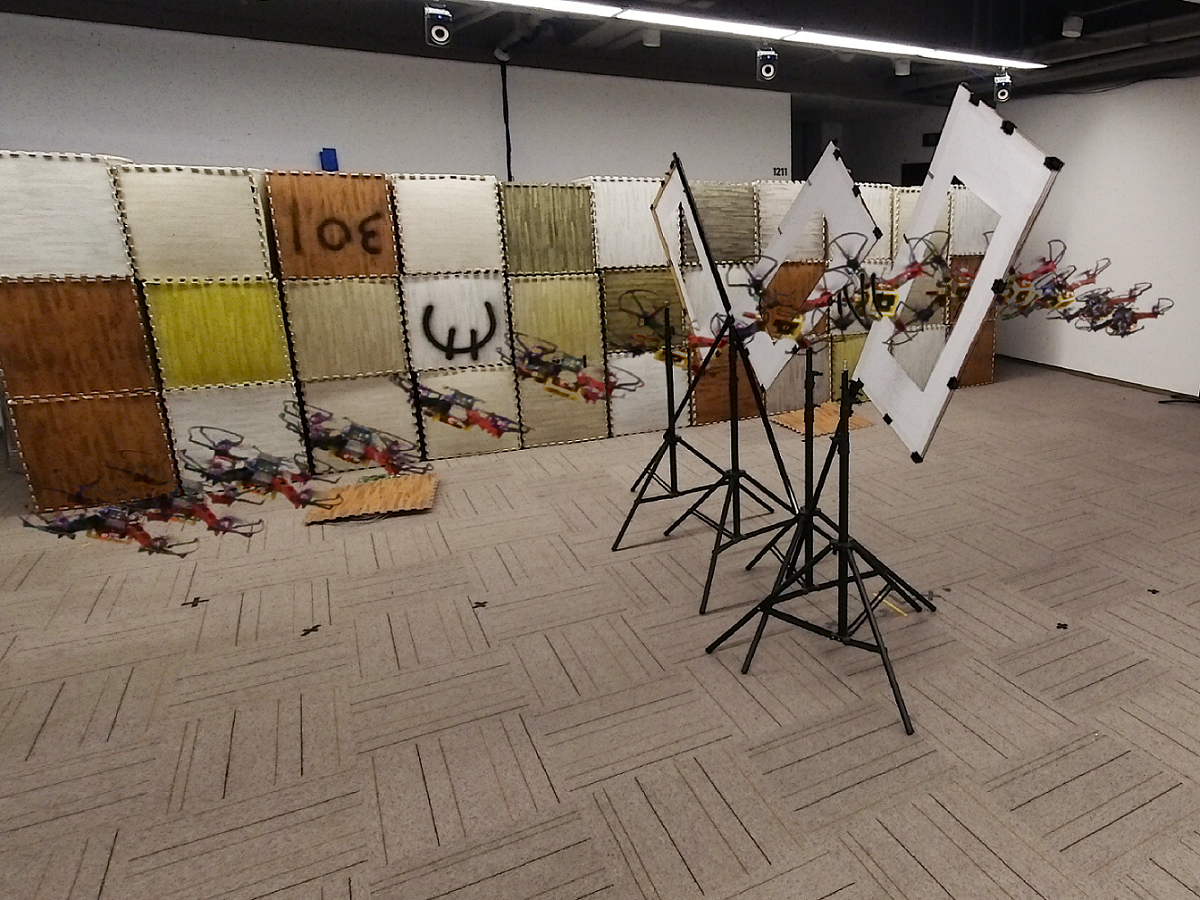}}
        \subfigure[\label{fig:HighSpeedGlobalView}Global trajectory planning results.]
        {\includegraphics[width=0.49\columnwidth]{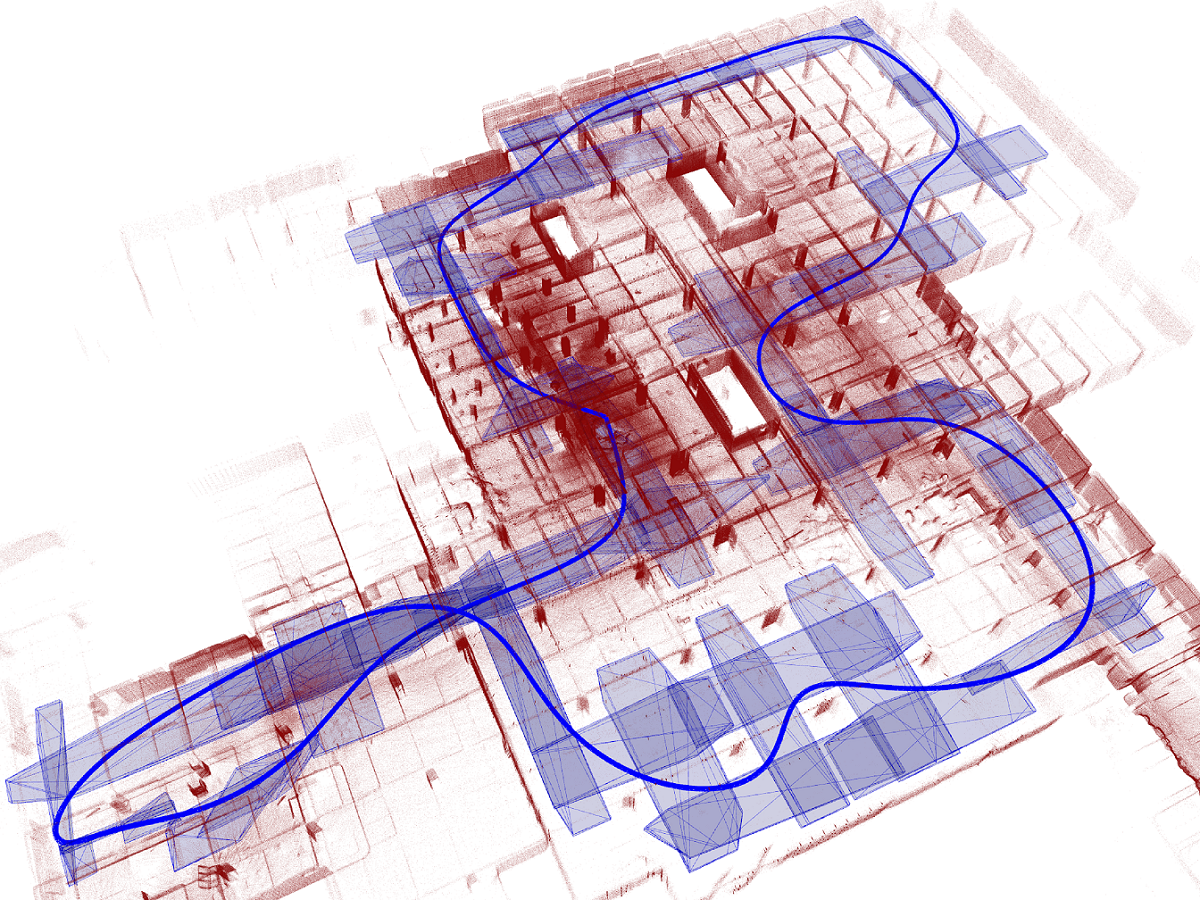}}%
        \hspace{0.02cm}
        \subfigure[\label{fig:ConsecutiveSE3Trajectory}$\mathrm{SE}(3)$ trajectory planning results.]
        {\includegraphics[width=0.49\columnwidth]{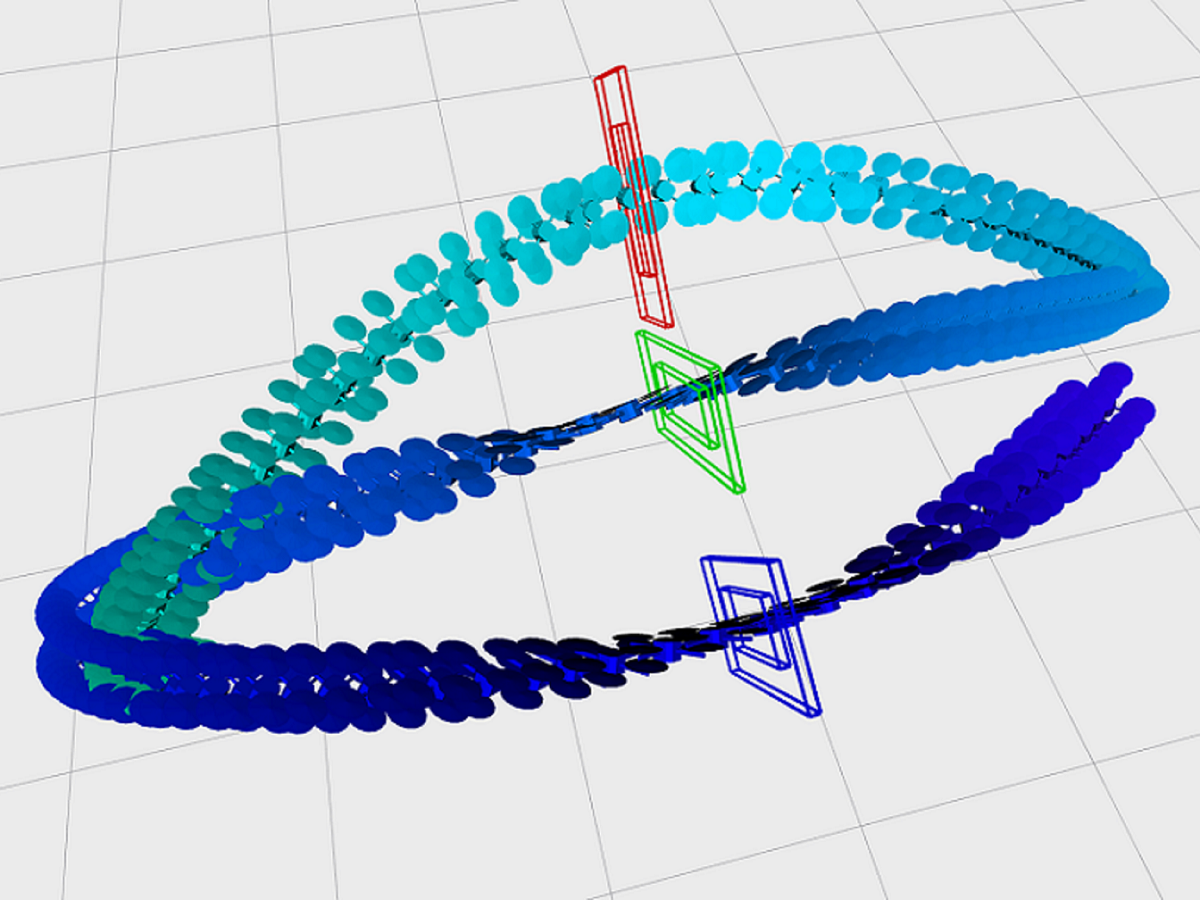}}
    \end{center}
    \caption{\label{fig:ExtremeFlights}Experimental validation of our framework through extreme flights. The left figures show long-distance high-speed flights in an underground garage. A $343.57m$ trajectory is computed in $0.29s$. The right figures show aggressive yet robust $\mathrm{SE}(3)$ maneuvers through narrow windows repeated for $12$ times. The flight speed and tilt angle in our experiments reach $12.0m/s$ and $90^\circ$, respectively. Details can be found in the attached multimedia.}
    \vspace{0.0cm}
\end{figure}

The high-performance planning mentioned above possesses four major algorithmic challenges. First, ensuring safety often involves frequent interactions with a large volume of highly discretized environment data. Second, the nonlinearity of vehicle dynamics brings difficulties to directly enforcing physically acceptable states and inputs when the multicopter is flying at the limit of its capabilities.
Third, high-quality motions conventionally need fine discretization of the dynamic process, where requirements for task resources tend to be unrealistic.
Fourth, methods that use sparse representation for trajectories lack an effective way to optimize the temporal profile while satisfying continuous-time constraints.

In this article, we overcome these challenges by designing a lightweight and flexible optimization framework to meet user-defined requirements based on a novel trajectory class.

As the theoretical foundation of our framework, we present necessary and sufficient optimality conditions to multistage control effort minimization for the concerned linear dynamics, which are given for the first time to the best of our knowledge.
The conditions are easy to use in that the unique optimal solution can be directly constructed with linear complexity in both time and space aspects. More importantly, the existence and uniqueness of conditions further provide crucial information on smoothness of the problem parameter sensitivity.

To ease computation burden without sacrificing trajectory quality, it is essential to use sparse parameterization while keeping the flexibility to suit multicopter dynamics.
Therefore, we design a novel trajectory class based on our optimality conditions. Any element in this class is by default an unconstrained control effort minimizer, thus we name it as MINCO (Minimum Control). MINCO differs from conventional splines that majorly focus on the smoothness of the geometrical shape, such as B-Splines and B\'ezier curves.
Its sparse parameters are designed to directly control both the spatial and temporal profile of a trajectory, which are of equal importance for dynamic feasibility. Besides, a spatial-temporal deformation scheme is also designed such that MINCO can be optimized under any user-defined objective.

Our framework utilizes the geometrical approximation of low-dimensional free space based on results of sampling-based or search-based global methods. The safety is ensured via configuration constraints formed by the union of obstacle-free convex primitives. Constraint elimination schemes are proposed such that MINCO can be freely deformed through unconstrained optimization. The schemes exactly eliminate constraints that are directly defined on decision variables without introducing extra local minima.

Reliable motion planning requires admissible states and inputs, while most existing flatness-based methods only support differential constraints. To ensure high-fidelity feasibility, we propose a systematic way to enforce user-defined state-input constraints for our sparse parameterization without resorting to a fine discretization of trajectories. We exploit the backward differentiation of flatness map such that the constraint violation can be reflected in their gradient w.r.t. sparse parameters. Besides, a differentiable penalty functional is also proposed to enforce general continuous-time constraints.

Our framework focuses on computationally efficient yet high-quality trajectory planning for multicopters where there are complex constraints for safety, critical limits on dynamics, and task-specified requirements.
To validate its effectiveness, we conduct extensive benchmarks against various cutting-edge multicopter trajectory planning methods.
Results show that our method exceeds existing methods for orders of magnitude in efficiency, and retains comparable solution quality against general-purpose optimal-control solvers.
We also conduct versatile simulations and extreme real-world flights to show the practical performance of our approach.

The contributions of this article are as follows.
\begin{itemize}
    \item Optimality conditions in a general form on multi-stage control effort minimization are proposed with a proof of both the necessity and sufficiency for the first time.
    \item A novel trajectory class is designed to meet user-defined objectives while retaining local smoothness by spatial-temporal deformation via linear-complexity operations.
    \item A flexible trajectory planning framework that leverages both constraint elimination and constraint transcription is proposed for multicopter systems with user-defined state-input constraints.
    \item A set of simulations and experiments that validate our method significantly outperforms state-of-the-art works in efficiency, optimality, robustness, and generality.
\end{itemize}

\section{Related Work}
Despite various planning approaches in existing literature, there has yet to emerge a complete framework to accomplish time-critical large-scale trajectory planning for multicopters while incorporating user-defined continuous-time constraints on state and control. Our framework bridges this gap by exploring and exploiting different capabilities from both optimal control and motion planning.

\subsection{Differentially Flat Multicopters}
The concept of differential flatness has been introduced by Fliess et al.~\cite{Fliess1995Flatness} and drawn great attentions in robotics trajectory planning~\cite{Van1998RealTGDFS, Martin2003FlatSETG, Ryu2011DifferentialFBRCMR}.
The property makes it possible to recover the full state and input of a flat system from finite derivatives of its flat outputs. Mellinger and Kumar~\cite{Mellinger2011MinimumST} validate the flatness of quadcopters with aligned propellers, which takes the thrust and three-dimensional torque as inputs. Watterson and Kumar~\cite{Watterson2019HOPF} use Hopf fibration to decompose the quadcopter rotation, thus achieving the minimum singularity number in flatness maps. Ferrin et al.~\cite{Ferrin2011DifferentialFBCHR} show the flatness of a hexacopter whose inputs are desired orientation and thrust. They utilize the flatness to compute the nominal state where a Linear Quadratic Regulator (LQR) is applied. Faessler et al.~\cite{Faessler2018DiffFRD} further consider linear drags that produce extra linear and angular accelerations.
They show the flatness of parallel-rotor multicopters subject to the drag effect.
Moreover, Mu and Chirarattananon~\cite{Mu2019TrajGTP} investigate underactuated multicopters with tilted propellers.
They prove that the flatness holds for a wide range of tricopters, quadcopters, and hexacopters as long as the input rank condition is satisfied.

The flatness of a multicopter, if holds, benefits trajectory generation and tracking control in obtaining the reference state and input without integrating differential equations. Literature mentioned above uses flatness to avoid confronting system dynamics during planning. However, dynamics of a real physical system are only valid for reasonable state and admissible input. Although our framework also utilizes the flatness property, it differs from previous works in that a general form of state-input constraints is formally supported.

\subsection{Sampling-Based Motion Planning}
Sampling-based motion planners focus on global solutions of problems by exploration and exploitation, where the complexity mainly originates from the configuration space. The Probabilistic Roadmap (PRM)~\cite{Kavraki1996ProbabilisticRM} and the Rapidly-exploring Random Tree (RRT)~\cite{Lavalle1998RapidlyRT} are both probabilistically complete since their probability of failure decays to zero exponentially as the sample number goes to infinity~\cite{Kavraki1998AnalysisPRM}. Karaman and Frazzoli~\cite{Karaman2011SamplingbasedAF} propose asymptotically optimal variants of PRM and RRT, known as PRM* and RRT*, which ensure the convergence to globally optimal solutions as the sample number goes to infinity. There are also algorithms~\cite{Janson2015FastMT, Li2016SamplingbasedKDP, Gammell2018InformedRRT} that further improve the efficiency or applicability of randomized motion planning. Our method exploits sampling-based planners to overcome the complexity from environments. It accomplishes the optimization of a dynamically feasible trajectory that is homotopic to a given low-dimensional collision-free path. It is designed to flexibly incorporate system state-input constraints, which is not the strength of sampling-based methods. In this way, the complexity from both the environments and dynamics are divided and conquered.

\subsection{Optimization-Based Motion Planning}
Optimization-based planners focus on local solutions by using high-order information of the problems. They depend on specific environment pre-processing methods such that the obstacle information is encoded into the optimization.

Trajectory optimization has long been studied for general systems in the control community~\cite{Betts2010PracticalNOC}.
Many general-purpose methods are designed for high-quality solutions, such as the collocation-based method GPOPS-II~\cite{Patterson2014GPOPS2}, and the shooting-based one ACADO~\cite{Houska2011Acado}. They transcribe the original problem into a Nonlinear Programming (NLP) using a lot of equalities and variables, then resort to well-established NLP solvers such as SNOPT~\cite{Gill2005SnoptLCO} or IPOPT~\cite{Wachter2006ImplementationIPOPT}.
However, trajectory planning in robotics may impose hard-to-formulate constraints, nonsmoothness, and integer variables.
Besides, general-purpose methods often take a long computation time, making them inappropriate for time-critical tasks.
For example, Bry et al.~\cite{Bry2015AggressiveFO} report that Direct Collocation (DC) with SNOPT takes several minutes to optimize a $4.5m$ trajectory for a $12$-state airplane flying among cylindrical obstacles~\cite{Barry2014Flying}.
Therefore, specialized methods are on calling to overcome these difficulties.

For differentially flat multicopters, motion planning can be transformed into optimization of low-dimensional trajectories of flat outputs.
Mellinger and Kumar~\cite{Mellinger2011MinimumST} use fixed-duration splines to represent quadcopter trajectories.
A Quadratic Programming (QP) is formulated by the quadratic cost of snap and linear constraints of safety.
However, perturbation problems need to be solved in finite difference to estimate the gradient for time allocation. Its actuator constraints are also oversimplified.
Bry et al.~\cite{Bry2015AggressiveFO} propose a closed-form solution for this QP without safety constraints. They heuristically add waypoints from a collision-free path of RRT* to recompute the solution until the safety is satisfied.
This method is admittedly efficient but cannot guarantee high-quality solutions in obstacle-rich environments.
Besides, it involves the inverse of a matrix whose non-singularity is never discussed.
Deits and Tedrake~\cite{Deits2015EfficientMISOS} approximate the free space using polytopes.
The safety of each piece of trajectory is equivalent to a Sum-of-Square (SOS) condition if it entirely lies inside a polytope. They solve interval assignment using Mixed Integer Second Order Conic Programming (MISOCP).
It generates globally optimal trajectories while the computation time is unacceptable.
Gao et al.~\cite{Gao2020TeachRepeatReplanAC} also use the polyhedron-shaped free space representation.
They alternately optimize the geometrical and temporal profile of a trajectory.
The safety is enforced by the convex-hull property of B\'ezier curves, and the dynamic profile is optimized via Time-Optimal Path Parameterization (TOPP)~\cite{Verscheure2009TimeOPT}.
There are also variants that propose improvements over the above methods.
For example, Tordesillas et al.~\cite{Tordesillas2019Faster} improve the efficiency of~\cite{Deits2015EfficientMISOS} by substituting SOS conditions on polynomials with linear constraints on B\'ezier curves at the cost of conservatism~\cite{Tordesillas2020MinvoBASIS}. Sun et al.~\cite{Sun2020BilevelTO} avoid integer variables by optimizing time allocation instead, where the sensitivity of a bilevel optimization is exploited.

These specialized methods utilize the continuous-time trajectory parameterization to avoid the computation burden from the fine discretization. However, they do not support flexibly optimizing its time allocation, decoupling temporal resolutions of constraints from decision variable dimensions, or enforcing high-fidelity constraints except for restrictions on derivative norms. In this article, our framework supports all these features by introducing unified techniques for a novel sparse parameterization. Moreover, the solution quality is comparable with that of the general-purpose optimal-control solvers.

\section{Preliminaries}
\subsection{Differential Flatness}

\begin{figure}[ht]
    \centering
    \includegraphics[width=1.0\columnwidth]{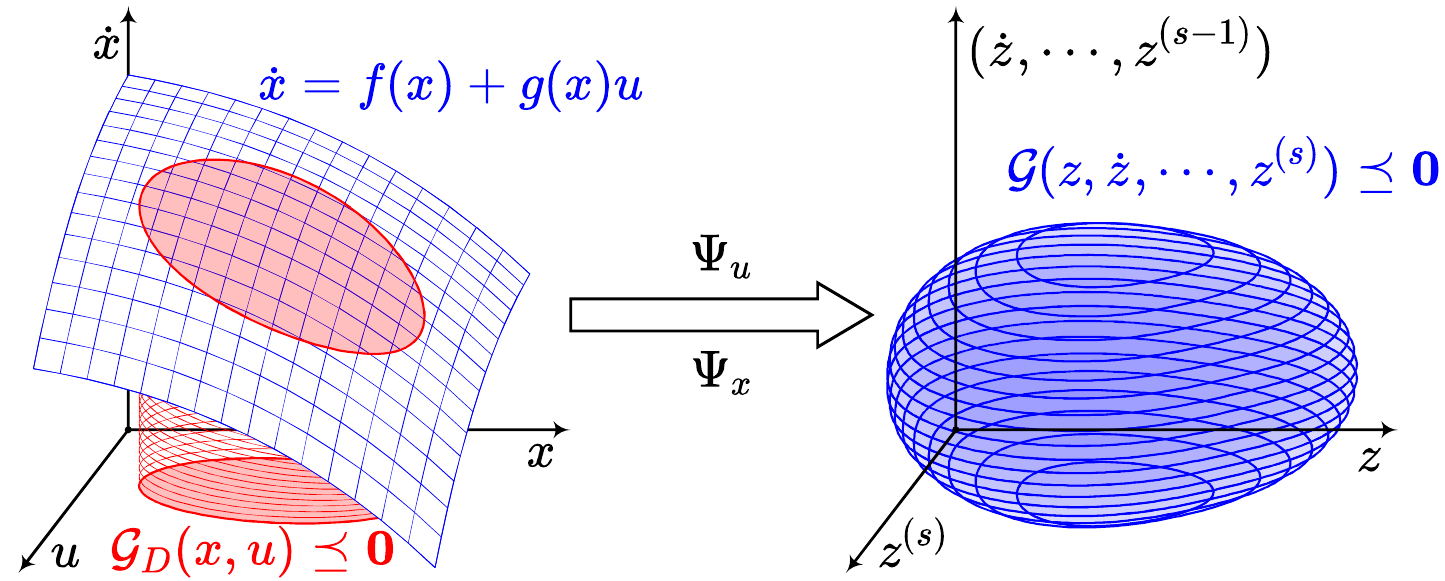}
    \caption{Transform $\Psi_u$ and $\Psi_x$ of a flat system eliminate differential constraints (blue surface) from dynamics in the state-input space (left coordinate). The original state-input constraint $\mathcal{G}_\mathcal{D}$ (red area) is also transformed into a new constraint $\mathcal{G}$ (blue volume) in the flat-output space (right coordinate).\label{fig:StateInputToFlatOutput}}
    \vspace{0.0cm}
\end{figure}

Consider a dynamical system of the following type
\begin{equation}
\label{eq:SystemDynamics}
\dot{x}=f(x)+g(x)u,
\end{equation}
with $f:\mathbb{R}^n\mapsto\mathbb{R}^n$, $g:\mathbb{R}^n\mapsto\mathbb{R}^{n\times m}$, state $x\in\mathbb{R}^n$, and input $u\in\mathbb{R}^m$. The map $g$ is assumed to have rank $m$. The system is said to be \textit{differentially flat}~\cite{Fliess1995Flatness}, if there exists a \textit{flat output} $z\in\mathbb{R}^m$ determined by $x$ and finite derivatives of $u$, such that $x$ and $u$ can both be parameterized by finite derivatives of $z$:
\begin{equation}
\label{eq:StateFromFlatOutput}
x=\Psi_x\rbrac{z,\dot{z},\dots,z^{(s-1)}},
\end{equation}
\begin{equation}
\label{eq:ControlFromFlatOutput}
u=\Psi_u\rbrac{z,\dot{z},\dots,z^{(s)}},
\end{equation}
where $\Psi_x:\rbrac{\mathbb{R}^m}^{s-1}\mapsto\mathbb{R}^n$ and  $\Psi_u:\rbrac{\mathbb{R}^m}^s\mapsto\mathbb{R}^m$ are both induced by $f$ and $g$.  Intuitively, the state and control can be determined from $z$ without explicit integration of the system dynamics (\ref{eq:SystemDynamics}).

Leveraging the flatness of a system, the trajectory generation is convenient when there are only differential constraints in (\ref{eq:SystemDynamics}). If we introduce a new control variable $v=z^{(s)}$ and denote $z^{[s-1]}\in\mathbb{R}^{ms}$ as
\begin{equation}
z^{[s-1]}=\rbrac{z\tp,\dot{z}\tp,\dots,z^{{(s-1)}\tp}}\tp,
\end{equation}
the input $u=\Psi_u\rbrac{z^{[s-1]},v}$ then exactly linearizes the original flat system into $m$ decoupled chains of $s$-integrators. Let $z_i$ denote the $i$-th entry in $z$, $v_i$ the $i$-th entry in $v$ and
$z_i^{[s-1]}=\rbrac{z_i,\dot{z}_i,\dots,z^{(s-1)}_i}\tp$. The $i$-th integrator chain is
\begin{equation}
\label{eq:IntegratorChain}
\dot{z}_i^{[s-1]}=\begin{pmatrix}\mathbf{0} & \mathbf{I}_{s-1} \\0 & \mathbf{0}\tp\end{pmatrix}
z_i^{[s-1]}+
\begin{pmatrix}\mathbf{0} \\1 \end{pmatrix}
v_i,
\end{equation}
where $\mathbf{0}$ and $\mathbf{I}$ are a zero matrix and an identity matrix with appropriate sizes, respectively. Given an initial state and a goal state, boundary values of each integrator chain (\ref{eq:IntegratorChain}) can be algebraically computed. Thus any trajectory integrated from these $m$ integrator chains can be transformed into a feasible trajectory~\cite{Van1998RealTGDFS} for the original flat system via (\ref{eq:StateFromFlatOutput}) and (\ref{eq:ControlFromFlatOutput}).

For dynamics with a small $m$, the flatness maps $\Psi_x$ and $\Psi_u$ further reduce the trajectory dimension and eliminate the differential constraints (\ref{eq:SystemDynamics}), which is illustrated in Fig.~\ref{fig:StateInputToFlatOutput}.
As a side effect, nonlinearity coming from both $\Psi_x$ and $\Psi_u$ brings additional difficulties in trajectory generation for $z$ when there are additional state-input constraints for (\ref{eq:SystemDynamics}).
However, such an effect is relieved if the flat-output space coincides with the configuration space of the relevant planning problem.

\subsection{Direct Optimization in Flat-Output Space}
Fortunately, the differential flatness of multicopters has been well studied and shown to have physically meaningful flat-output space which overlaps with the configuration space. Explicit forms of $\Psi_x$ and $\Psi_u$ are available in~\cite{Mellinger2011MinimumST, Watterson2019HOPF,Ferrin2011DifferentialFBCHR, Faessler2018DiffFRD} for a variety of underactuated multicopters. More importantly, their flat outputs share the same form in general:
\begin{equation}
\label{eq:MulticopterFlatOutput}
z=\rbrac{p_x,p_y,p_z,\psi}\tp
\end{equation}
where $(p_x,p_y,p_z)\tp$ is the translation of the Center of Gravity (CoG) and $\psi$ the yaw angle of the vehicle. The flat output $z$, especially its translation, provides a lot of convenience for the multicopter motion planning with complex spatial constraints.

To generate feasible motions for a multicopter, we first optimize the trajectory $z(t):[0,T]\mapsto\mathbb{R}^m$ in its flat-output space such that most of the spatial constraints are directly enforced. Then, the flatness maps $\Psi_x$ and $\Psi_u$ are applied to transform $z(t)$ into the state-input trajectory $x(t)$ and $u(t)$.

For motion smoothness, the quadratic control effort~\cite{Bertsekas1995DynamicPOC} with time regularization is adopted as a cost functional of $z(t)$. General constraints on multicopters can be classified into configuration constraints and user-defined dynamic constraints. Normally, a collision-free motion implies
\begin{equation}
z(t)\in\mathcal{F},~\forall t\in[0,T],
\end{equation}
where $\mathcal{F}$ is the concerned \textit{obstacle-free} region in configuration space. Besides, user-defined state-input constraints such as actuator limits or task-specific constraints are denoted by
\begin{equation}
\label{eq:OriginalStateControlConstraints}
\mathcal{G}_\mathcal{D}\rbrac{x(t),u(t)}\preceq\mathbf{0},~\forall t\in[0,T].
\end{equation}
Exploiting $\Psi_x$ and $\Psi_u$, the corresponding constraints on $z(t)$ are computed as
\begin{equation}
\label{eq:StateControlConstraintsViaFlatness}
\mathcal{G}_\mathcal{D}\rbrac{\Psi_x\rbrac{z^{[s-1]}(t)},\Psi_u\rbrac{z^{[s]}(t)}}\preceq\mathbf{0},~\forall t\in[0,T].
\end{equation}
Apparently, via the flatness, a constraint on $x$ and $u$ has its equivalent form on the finite derivatives of $z(t)$.
For simplicity, we denote (\ref{eq:StateControlConstraintsViaFlatness}) hereafter by
\begin{equation}
\label{eq:FunctionalConstraints}
\mathcal{G}\rbrac{z(t),\dot{z}(t),\dots,z^{(s)}(t)}\preceq\mathbf{0},~\forall t\in[0,T],
\end{equation}
where $\mathcal{G}$ consists of $n_g$ equivalent constraints.

It is worth noting that we do not make further assumptions on the multicopter dynamics and flatness maps. In other words, the proposed framework supports a wide range of multicopters, including, but not limited to the ones in~\cite{Mellinger2011MinimumST, Watterson2019HOPF,Ferrin2011DifferentialFBCHR, Faessler2018DiffFRD, Mu2019TrajGTP}.

\subsection{Problem Formulation}
Concluding above descriptions gives the following problem:
\begin{subequations}
    \label{eq:TrajectoryOptimization}
    \begin{align}
    \min_{z(t), T}&~{\int_{0}^{T}{v(t)\tp \mathbf{W}v(t)\df{t}}+\rho(T)},\\
    \mathit{s.t.}~&~z^{(s)}(t)=v(t),~\forall t\in[0,T],\\
    &~\mathcal{G}\rbrac{z(t),\dots,z^{(s)}(t)}\preceq\mathbf{0},~\forall t\in[0,T],\\
    &~z(t)\in\mathcal{F},~\forall t\in[0,T],\\
    &~z^{[s-1]}(0)=\bar{z}_o,~z^{[s-1]}(T)=\bar{z}_f,
    \end{align}
\end{subequations}
where $\mathbf{W}\in\mathbb{R}^{m\times m}$ is a positive diagonal matrix, $\rho:[0,\infty)\mapsto[0,\infty]$ the time regularization, $\bar{z}_o\in\mathbb{R}^{ms}$ the initial condition and $\bar{z}_f\in\mathbb{R}^{ms}$ the terminal condition. The control input $v$ is allowed to be discontinuous in a finite number of time instants, as is commonly assumed in existing literature~\cite{Bertsekas1995DynamicPOC}.

The trajectory optimization (\ref{eq:TrajectoryOptimization}) is nontrivial because of the continuous-time constraints $\mathcal{G}$ and the nonconvex set $\mathcal{F}$. We further specify some reasonable conditions to make it a well defined problem. As for time regularization $\rho$, it trades off between the control effort and the expectation of total time,
\begin{equation}
\label{eq:SoftDuration}
\rho_s(T)=\sum_{i=0}^{M_T}{b_iT^i},
\end{equation}
where $b_{M_T}$ is positive.
Common choices are $\rho_s(T)=k_\rho T$ and $\rho_s(T)=k_\rho(T-T_\Sigma)^2$ with an expected time $T_\Sigma$.
Besides, $\rho$ can also be defined to strictly fix the total time:
\begin{equation}
\label{eq:HardDuration}
\rho_f(T)=\begin{cases} 0 & \mathit{if}~T=T_\Sigma,\\ \infty & \mathit{if}~T\neq T_\Sigma. \end{cases}
\end{equation}
As for nonlinear constraints $\mathcal{G}$, they are required to be $C^2$, i.e., twice continuously differentiable. As for the feasible region $\mathcal{F}$ in the configuration space, we approximate it geometrically by the union of $M_\mathcal{P}$ closed convex sets as
\begin{equation}
\label{eq:FreeSpaceConvexDecomposition}
\mathcal{F}\simeq\tilde{\mathcal{F}}=\bigcup_{i=1}^{M_\mathcal{P}}\mathcal{P}_i.
\end{equation}
For simplicity, locally sequential connection is assumed on these convex sets:
\begin{equation}
\label{eq:LocallySequentialConnection}
\begin{cases}
\mathcal{P}_i\cap\mathcal{P}_j=\varnothing & \mathit{if}~\abs{i-j}=2,\\
\mathrm{Int}\rBrac{\mathcal{P}_i\cap\mathcal{P}_j}\neq\varnothing & \mathit{if}~\abs{i-j}\leq1,
\end{cases}
\end{equation}
where $\mathrm{Int}\rbrac{\cdot}$ means the interior of a set.
The translation of $\bar{z}_o$ and $\bar{z}_f$ is inscribed in $\mathcal{P}_1$ and $\mathcal{P}_{M_\mathcal{P}}$, respectively. As for $\tilde{\mathcal{F}}$, we consider the case that each $\mathcal{P}_i$ is a closed $m$-dimensional ball:
\begin{equation}
\mathcal{P}^\mathcal{B}_i=\cBrac{x\in\mathbb{R}^m~\Big|~\Norm{x-o_i}_2\leq r_i},
\end{equation}
or, more generally, a bounded convex polytope described by its $\mathcal{H}$-representation~\cite{Toth2017HandbookDCG} with potentially redundant constraints:
\begin{equation}
\label{eq:HPolytopeDescription}
\mathcal{P}^\mathcal{H}_i=\cBrac{x\in\mathbb{R}^m~\Big|~\mathbf{A}_ix\preceq b_i}.
\end{equation}
For the optimization in (\ref{eq:TrajectoryOptimization}), we aim to construct a computationally efficient solver while retaining the flexibility to handle different task-specific constraints $\mathcal{G}_\mathcal{D}$ in (\ref{eq:OriginalStateControlConstraints}).

\section{Multi-Stage Control Effort Minimization}
In this section, we analyze the multi-stage control effort minimization without functional constraints.
For this problem, we propose easy-to-use optimality conditions for general cases, which are proved to be necessary and sufficient. Leveraging our conditions, the optimal trajectory is directly constructed in linear complexity of time and space, without evaluating the cost functional explicitly or implicitly.
Base on them, a novel trajectory class along with linear-complexity spatial-temporal deformation is designed to meet user-defined objectives in various trajectory planning scenarios.

\subsection{Unconstrained Control Effort Minimization}

When constraint $\mathcal{F}$ exists, adjusting the waypoints~\cite{Bry2015AggressiveFO} or control points~\cite{Tordesillas2019Faster} of a trajectory helps to ensure safety. When constraint $\mathcal{G}$ exists, adjusting the time allocation also helps to enforce physical limits~\cite{Liu2017PlanningDF}. Therefore, spatial and temporal parameters are both vital to a flexible trajectory representation. A natural problem is to generate a smooth trajectory subject to these parameters. 
We solve \textit{Linear Quadratic Minimum-Time} (LQMT) problems to generate trajectories from spatial-temporal parameters. Although the LQMT problems have extensive studies and applications, only single-stage problems are considered in the literature~\cite{Verriest1991LinearQMT, Mueller2015ACE, Liu2017SearchMP}. We study the multi-stage problems where intermediate points and time vector are fixed in advance for multi-piece trajectories. Consider an $M$-stage control effort minimization without $\mathcal{F}$ and $\mathcal{G}$,
\begin{subequations}
    \label{eq:MultistageMinimumControl}
    \begin{align}
    \min_{z(t)}&~{\int_{t_0}^{t_M}{v(t)\tp \mathbf{W}v(t)}\df{t}},\\
    \mathit{s.t.}~&~z^{(s)}(t)=v(t),~\forall t\in[t_0,t_M],\\
    \label{eq:InitialTerminalConditions}
    &~z^{[s-1]}(t_0)=\bar{z}_o,~z^{[s-1]}(t_M)=\bar{z}_f,\\
    \label{eq:IntermediateConditions}
    &~z^{[d_i-1]}(t_i)=\bar{z}_i,~1\leq i<M,\\
    &~t_{i-1}<t_i,~1\leq i\leq M.
    \end{align}
\end{subequations}
The time interval $[t_0,t_M]$ is split into $M$ stages by $M+1$ fixed timestamps, with constant boundary conditions $\bar{z}_o,\bar{z}_f\in\mathbb{R}^{ms}$. Intermediate conditions $\bar{z}_i\in\mathbb{R}^{md_i}$ with $d_i\leq s$ specify the value of $z(t_i), \dot{z}(t_i),\dots,z^{(d_i-1)}(t_i)$, where $d_i$ is number of derivatives fixed at $t_i$. For example, if $z(t)$ is only required to pass a given position at $t_i$, then $d_i=1$ because $\bar{z}_i$ contains the $0$-order derivative and nothing else.

Existing works focus on the necessary conditions for special cases of (\ref{eq:MultistageMinimumControl}). In aerial robotics area, the QP formulation~\cite{Mellinger2011MinimumST} and the closed-form one~\cite{Bry2015AggressiveFO} implicitly or explicitly optimize unknown knot derivatives, taking parameterization as \textit{a priori}.
This extra computation actually makes them less efficient.
In control area, a special case where $d_i=1$ is also studied in~\cite{Zhang1997SplinesLCT} and~\cite{Egerstedt2009ControlTS} via controllability Gramian. The result is for general linear systems with possibly nonpolynomial solutions while it is less intuitive considering the computational aspect. These necessary conditions can cause potential degeneracy in trajectory representation and sensitivity, if further parametric optimization on spatial-temporal parameters is needed.

\subsection{Optimality Conditions}
We propose necessary and sufficient optimality conditions for (\ref{eq:MultistageMinimumControl}) with all possible settings of $d_i$, $\bar{z}_i$, and $t_i$. Thus, an optimal trajectory can be directly constructed from spatial-temporal parameters. Furthermore, the existence and uniqueness of the optimal trajectory are always guaranteed.

We transform (\ref{eq:MultistageMinimumControl}) into the Mayer form~\cite{Betts2010PracticalNOC} in which a new state $y\in\mathbb{R}^{ms+1}$ augmented by $\tilde{y}\in\mathbb{R}$ is defined as
\begin{equation}
y=\begin{pmatrix} z^{[s-1]} \\ \tilde{y} \end{pmatrix}.
\end{equation}
The augmented system $\dot{y}=\hat{f}\rbrac{y,v}$ has the structure
\begin{equation}
\label{eq:AugmentedSystem}
\dot{y}=
\begin{pmatrix} \bar{\mathbf{A}} & \mathbf{0} \\ \mathbf{0}\tp & 0 \end{pmatrix}
y+
\begin{pmatrix} \mathbf{0} \\ v \\ v\tp \mathbf{W}v \end{pmatrix},
\end{equation}
where
\begin{equation}
\bar{\mathbf{A}} =
\begin{pmatrix}
\mathbf{0} & \mathbf{I}_{m(s-1)} \\
\mathbf{0}_{m\times m} & \mathbf{0}\tp
\end{pmatrix}
\in\mathbb{R}^{ms\times ms}.
\end{equation}

We design a running process for the augmented system in $M$ stages, of which the $i$-th is $\Delta_i=[t_{i-1}, t_i]$. It is worth noting that state switching occurs in this running process. Strictly speaking, the state switching only occurs on $\tilde{y}$ at the beginning of each stage.
Denote by $y_{[i]}:\Delta_i\mapsto\mathbb{R}^{ms+1}$ the augmented state trajectory in the $i$-th stage, which consists of two parts, $z^{[s-1]}_{[i]}$ and $\tilde{y}_{[i]}$.
At each timestamp $t_i$, the state transfers from $y_{[i]}$ to $y_{[i+1]}$, and the part $\tilde{y}$ is reset as
\begin{equation}
\tilde{y}_{[i+1]}(t_{i})=0,~0\leq i<M,
\end{equation}
thus switching the partial state from $\tilde{y}_{[i]}(t_{i})$ to $0$.
The $z^{[s-1]}$ part remains continuous between stages, which means
\begin{equation}
\label{eq:IntermediateContinuity}
z^{[s-1]}_{[i]}(t_i)=z^{[s-1]}_{[i+1]}(t_i),~1\leq i<M.
\end{equation}
The conditions in (\ref{eq:InitialTerminalConditions}) and (\ref{eq:IntermediateConditions}) are still satisfied, i.e.,
\begin{align}
\label{eq:InitialTerminalState}
&z^{[s-1]}_{[1]}(t_0)=\bar{z}_o,~z^{[s-1]}_{[M]}(t_M)=\bar{z}_f,\\
\label{eq:IntermediateState}
&z^{[d_i-1]}_{[i]}(t_i)=\bar{z}_i,~1\leq i<M.
\end{align}
In this process, the cost functional in (\ref{eq:MultistageMinimumControl}) is converted into the sum of terminal cost of each stage for the augmented system, i.e., $\sum_{i=1}^{M}{\tilde{y}_{[i]}(t_i)}$.
Therefore, the optimal trajectories for the augmented system and the original one are identical in $z^{[s-1]}$.

We utilize the \textit{Hybrid Maximum Principle}~\cite{Dmitruk2008HMP} to derive necessary conditions for the optimal solution.
\begin{theorem}[\textbf{Hybrid Maximum Principle}]
    \label{thm:HybridMaximumPrinciple}
    Let $t_0<\dots<t_M$ be real numbers and $\Delta_k=[t_{k-1},t_k]$. For any collection of absolute continuous functions $x_k:\Delta_k\mapsto\mathbb{R}^{n_k}$, define a vector, $x_\Sigma\in\mathbb{R}^{\bar{n}}$ where $\bar{n}=2\sum_{k=1}^M{n_k}$, as
    \begin{equation}
    \label{eq:JointJunctionState}
    x_\Sigma=\rBrac{x_1\tp(t_0),x_1\tp(t_1),\dots,x_M\tp(t_{M-1}),x_M\tp(t_M)}\tp.
    \end{equation}
    On the time interval $\Delta=[t_0,t_M]$ consider the problem
    \begin{subequations}
        \label{eq:HybridOptimalControl}
        \begin{align}
        \min_{u_k,x_k}&~{J(x_\Sigma)},\\
        \mathit{s.t.}~&~\dot{x}_k(t)=f_k\rbrac{x_k(t),u_k(t)},\\
        &~u_k(t)\in U_k\subseteq\mathbb{R}^{r_k},\\
        &~\forall t\in\Delta_k,~k=1,\dots,M,\\
        &~\eta(x_\Sigma)=\mathbf{0},
        \end{align}
    \end{subequations}
    where $f_k:\mathbb{R}^{n_k}\times\mathbb{R}^{r_k}\mapsto\mathbb{R}^{n_k}$, $J:\mathbb{R}^{\bar{n}}\mapsto\mathbb{R}$ and $\eta:\mathbb{R}^{\bar{n}}\mapsto\mathbb{R}^q$ are continuously differentiable, $u_k:\mathbb{R}\mapsto\mathbb{R}^{r_k}$ are measurable and bounded on the corresponding $\Delta_k$.

    Denote an optimal process for (\ref{eq:HybridOptimalControl}) by $\rbrac{x^*(t),u^*(t)}$. Then, there exists a collection $(\alpha, \gamma, \psi_1,\dots,\psi_M)$, where $\alpha\geq0$, $\gamma\in\mathbb{R}^q$ and $\psi_k:\Delta_k\mapsto\mathbb{R}^{n_k}$ are Lipschitz continuous. It generates $M$ Pontryagin functions
    \begin{equation}
    H_k(\psi_k,x_k,u_k)=\psi_k\tp f_k(x_k,u_k),~t\in\Delta_k,
    \end{equation}
    and a Lagrange function $L(x_\Sigma)=\alpha J(x_\Sigma)+\gamma\tp\eta(x_\Sigma)$. The following conditions are satisfied for all $k=1,\dots,M$.
    \begin{itemize}
        \item Nontriviality condition:
        \begin{equation}
        \label{eq:NontrivialityCondition}
        (\alpha, \gamma\tp)\neq\mathbf{0};
        \end{equation}
        \item Adjoint equations:  for almost all $t\in\Delta_k$,
        \begin{equation}
        \label{eq:AdjointEquations}
        \dot{\psi}_k(t)=-\frac{\partial H_k}{\partial x_k}(\psi_k(t),x_k^*(t),u_k^*(t));
        \end{equation}
        \item Transversality conditions:
        \begin{equation}
        \label{eq:TransversalityConditions}
        \left\{
        \begin{gathered}
        \psi_k(t_{k-1})=L_{x_k(t_{k-1})}(x_\Sigma^*),\\
        \psi_k(t_k)=-L_{x_k(t_k)}(x_\Sigma^*);
        \end{gathered}
        \right.
        \end{equation}
        \item Maximality conditions: for all $t\in\Delta_k$,
        \begin{align}
        \label{eq:MaximalityCondition}
        &~H_k(\psi_k(t),x_k^*(t),u_k^*(t))\nonumber\\
        =&\sup_{u_k\in U_k}{H_k(\psi_k(t),x_k^*(t),u_k)}\\
        =&~0.\nonumber
        \end{align}
    \end{itemize}
\end{theorem}
\begin{proof}
    The proof can be directly adapted from Theorem 4 by Dmitruk and Kaganovich~\cite{Dmitruk2008HMP}. Here we only consider each system $f_k$ to be time-invariant and all intervals $\Delta_k$ to be fixed. Besides, no inequality constraints are specified on $x_\Sigma$.
\end{proof}

According to Theorem~\ref{thm:HybridMaximumPrinciple}, the costate $\psi_{[i]}:\Delta_i\mapsto\mathbb{R}^{ms+1}$ in the $i$-th stage is defined as
\begin{equation}
\label{eq:CostateDefinition}
\psi_{[i]}=\begin{pmatrix} \lambda_{[i]} \\ \mu_{[i]} \end{pmatrix}
=\rBrac{{\lambda_{[i]}}_1, {\lambda_{[i]}}_2, \dots, {\lambda_{[i]}}_s, \mu_{[i]}}\tp,
\end{equation}
where $\mu_{[i]}:\Delta_i\mapsto\mathbb{R}$.
${\lambda_{[i]}}_j:\Delta_i\mapsto\mathbb{R}^m$ is the $j$-th map in $\lambda_{[i]}:\Delta_i\mapsto\mathbb{R}^{ms}$. The $i$-th Pontryagin function of (\ref{eq:AugmentedSystem}) is
\begin{align}
&H_i(\psi_{[i]}, y_{[i]}, v_{[i]})=\psi_{[i]}\tp \hat{f}(y_{[i]}, v_{[i]})\\
&=\lambda_{[i]}\tp\bar{\mathbf{A}}z_{[i]}^{[s-1]}+{\lambda_{[i]}\tp}_s v_{[i]}+\mu_{[i]}v_{[i]}\tp\mathbf{W}v_{[i]}.\nonumber
\end{align}

By applying the adjoint equation (\ref{eq:AdjointEquations}) for $\mu_{[i]}$, we have $\dot{\mu}_{[i]}=0$, which means $\mu_{[i]}(t)=\bar{\mu}_i\in\mathbb{R}$ is a constant in $\Delta_i$. Therefore, $H_i$ is always a quadratic function of $v_{[i]}$,
\begin{equation}
\label{eq:PontryaginFunction}
H_i(\psi_{[i]}, y_{[i]}, v_{[i]})=\lambda_{[i]}\tp\bar{\mathbf{A}}z_{[i]}^{[s-1]}+{\lambda_{[i]}\tp}_s v_{[i]}+\bar{\mu}_i v_{[i]}\tp\mathbf{W}v_{[i]}.
\end{equation}
By applying the adjoint equation for $\lambda_{[i]}$, we obtain
\begin{equation}
\label{eq:CostateLambdaSystem}
\dot{\lambda}_{[i]}=-\bar{\mathbf{A}}\tp\lambda_{[i]},
\end{equation}
which is expanded as
\begin{equation}
{{}\dot{\lambda}_{[i]}}_j=\begin{cases} \mathbf{0} & \mathit{if}~j=1, \\ -{\lambda_{[i]}}_{j-1} & \mathit{if}~2\leq j\leq s. \end{cases}
\end{equation}
It is obvious that ${\lambda_{[i]}}_s(t)$ is an $s-1$ degree polynomial.

According to maximality conditions (\ref{eq:MaximalityCondition}), the supremum of $H_i$ is always $0$ in $\Delta_i$. Thus the positive definiteness of $\mathbf{W}$ implies $\bar{\mu}_i\leq0$. If $\bar{\mu}_i=0$, then (\ref{eq:PontryaginFunction}) becomes a linear function of $v_{[i]}$. The zero supremum means that ${\lambda_{[i]}}_s(t)=\mathbf{0}$ in $\Delta_i$. As the result of (\ref{eq:CostateLambdaSystem}), $\psi_{[i]}(t)=\mathbf{0}$ holds for all $t$ in $\Delta_i$. In such a case, a contradiction occurs that the nontriviality condition (\ref{eq:NontrivialityCondition}) and the transversality conditions (\ref{eq:TransversalityConditions}) cannot be satisfied at the same time. Therefore, $\bar{\mu}_i<0$ always holds in the whole $\Delta_i$. The optimal control $v_{[i]}^*$ can be obtained from
\begin{equation}
\frac{\partial H_i}{\partial v_{[i]}}(\psi_{[i]}, y_{[i]}^*, v_{[i]}^*)={\lambda_{[i]}}_s+2\bar{\mu}_i\mathbf{W}v_{[i]}^*=\mathbf{0},
\end{equation}
i.e.,
\begin{equation}
\label{eq:OptimalControlFromCostate}
v_{[i]}^*(t)=-\frac{1}{2\bar{\mu}_i}\mathbf{W}^{-1}{\lambda_{[i]}}_s(t),~\forall t\in\Delta_i.
\end{equation}
Then, $z_{[i]}^*$ produced by a chain of $s$-integrators from ${\lambda_{[i]}}_s(t)$, is a $2s-1$ degree polynomial.

To further explore structures of the solution, we generate the Lagrange function using the cost of augmented system along with all constraints in (\ref{eq:IntermediateContinuity}), (\ref{eq:InitialTerminalState}) and (\ref{eq:IntermediateState}). We have
\begin{align}
L(y_\Sigma)=&~\alpha\sum_{i=1}^{M}{\tilde{y}_{[i]}(t_i)}+\sum_{i=0}^{M-1}{\gamma_i\tilde{y}_{[i+1]}(t_{i})}\\
+&~\sum_{i=1}^{M-1}{\rbrac{\zeta_i\tp,\sigma_i\tp}\rbrac{z^{[s-1]}_{[i]}(t_i)-z^{[s-1]}_{[i+1]}(t_i)}}\nonumber\\
+&~\theta_o\tp\rbrac{z^{[s-1]}_{[1]}(t_0)-\bar{z}_o}+\theta_f\tp\rbrac{z^{[s-1]}_{[M]}(t_M)-\bar{z}_f}\nonumber\\
+&~\sum_{i=1}^{M-1}{\theta_i\tp\rbrac{z^{[d_i-1]}_{[i]}(t_i)-\bar{z}_i}},\nonumber
\end{align}
where $\gamma_i\in\mathbb{R}$, $\zeta_i\in\mathbb{R}^{md_i}$, $\sigma_i\in\mathbb{R}^{m(s-d_i)}$, $\theta_o\in\mathbb{R}^{ms}$, $\theta_f\in\mathbb{R}^{ms}$ and $\theta_i\in\mathbb{R}^{md_i}$ are all constant coefficients of corresponding constraints, $y_\Sigma$ is defined as in (\ref{eq:JointJunctionState}). Following transversality conditions (\ref{eq:TransversalityConditions}), taking the derivative of $L$ w.r.t. $y_\Sigma$ gives the boundary values of costates $\psi_{[i]}$ and $\psi_{[i+1]}$, i.e.,
\begin{equation}
\label{eq:LambdaValues}
\lambda_{[i]}(t_i)=-\begin{pmatrix} \zeta_i+\theta_i \\\sigma_i \end{pmatrix},~\lambda_{[i+1]}(t_i)=-\begin{pmatrix} \zeta_i \\\sigma_i \end{pmatrix},
\end{equation}
\begin{equation}
\mu_{[i]}(t_i)=\mu_{[i+1]}(t_{i+1})=-\alpha.
\end{equation}
Because $\mu_{[i+1]}(t)=\bar{\mu}_{i+1}$ in $\Delta_{i+1}$,  we have
\begin{equation}
\label{eq:KappaValues}
\bar{\mu}_i=-\alpha,~1\leq i\leq M.
\end{equation}
Finally, by substituting (\ref{eq:CostateLambdaSystem}), (\ref{eq:LambdaValues}) and (\ref{eq:KappaValues}) into (\ref{eq:OptimalControlFromCostate}), we obtain that the optimal controls of two consecutive stages satisfy
\begin{equation}
\label{eq:ContinuousControl}
{v^*}_{[i]}^{(j)}(t_i)={v^*}_{[i+1]}^{(j)}(t_i),~ 0\leq j<(s-d_i).
\end{equation}

We finally know that the optimal control of the problem (\ref{eq:MultistageMinimumControl}) is actually $s-d_i-1$ times continuously differentiable at the timestamp $t_i$. Accordingly, the optimal state trajectory, consisting of $M$ polynomials with $2s-1$ degree, is indeed $2s-d_i-1$ times continuously differentiable at $t_i$.

Now we conclude the conditions derived from both (\ref{eq:OptimalControlFromCostate}) and (\ref{eq:ContinuousControl}) in the following theorem, which are proved to be necessary and sufficient optimality conditions of (\ref{eq:MultistageMinimumControl}).
\begin{theorem}
    [\textbf{Optimality Conditions}]
    \label{thm:OptimalityConditions}
    A trajectory, denoted by $z^*(t):[t_0, t_M]\mapsto\mathbb{R}^{m}$, is optimal for the problem~(\ref{eq:MultistageMinimumControl}), if and only if the following conditions are satisfied:
    \begin{itemize}
        \item The map $z^*(t):[t_{i-1}, t_i]\mapsto\mathbb{R}^{m}$ is parameterized as a $2s-1$ degree polynomial for any $1\leq i\leq M$;
        \item The boundary conditions in (\ref{eq:InitialTerminalConditions});
        \item The intermediate conditions in (\ref{eq:IntermediateConditions});
        \item $z^*(t)$ is $\bar{d}_i-1$ times continuously differentiable at $t_i$ for any $1\leq i<M$ where $\bar{d}_i=2s-d_i$.
    \end{itemize}
    Moreover, a unique trajectory exists for these conditions.
\end{theorem}
\begin{proof}[Proof]
    The proof of necessity is evident in the analyses from (\ref{eq:CostateDefinition}) to (\ref{eq:ContinuousControl}) that are directly derived from Theorem~\ref{thm:HybridMaximumPrinciple}. The proof of sufficiency is sketched below: (a) The first and fourth conditions always determine a linear spline space of dimension $2s+\sum_{i=1}^{M-1}{d_i}$ for any sequence of $d_i$; (b) The second and third conditions are shown to form a square coefficient matrix on a basis of the spline space; (c) The matrix is proved to be nonsingular since $t_{i-1}<t_i$ for each $i$, implying the existence and uniqueness of solution; (d) The existence and uniqueness for the necessary conditions yield their sufficiency. This proof of sufficiency is detailed in Appendix~\ref{apd:OptimalityConditionsProof}.
\end{proof}

To further explain the optimality conditions, we take the multi-stage jerk minimization as an example.
In this example, the position, velocity and acceleration are states of the jerk-controlled system ($s=3$). There are intermediate points ($d_i=1$) that the trajectory should pass through at certain timestamps.
The continuity of state only requires the continuity up to acceleration of the minimum-jerk trajectory, while jerk and snap of the optimal trajectory are also continuous everywhere. Accordingly, if we enforce all these continuity conditions, then Theorem~\ref{thm:OptimalityConditions} guarantees that only one trajectory exists, which is exactly the optimal one.

\subsection{Minimization Without Cost Functional}
Theorem~\ref{thm:OptimalityConditions} provides a direct way to construct the unique optimal trajectory. The computation enjoys linear complexity in time and space. It does not even require explicit or implicit evaluation of the cost functional or its gradient.

Consider an $m$-dimensional trajectory whose $i$-th piece is denoted by an $N=2s-1$ degree polynomial:
\begin{equation}
\label{eq:SinglePiece}
p_i(t)=\mathbf{c}_i\tp\beta(t-t_{i-1}),~t\in[t_{i-1},t_i],
\end{equation}
where $\beta(x)=\rbrac{1,x,\dots,x^N}\tp$ is the basis and $\mathbf{c}_i\in\mathbb{R}^{2s\times m}$ the coefficients. For simplicity, we use the timeline relative to $t_0=0$. The trajectory is described by a coefficient matrix $\mathbf{c}\in\mathbb{R}^{2Ms\times m}$ and a time vector $\mathbf{T}\in\mathbb{R}_{>0}^M$ defined by
\begin{equation}
\mathbf{c}=\rBrac{\mathbf{c}_1\tp, \dots, \mathbf{c}_M\tp}\tp,~\mathbf{T}=\rBrac{T_1, \dots, T_M}\tp,
\end{equation}
where $T_i$ means the duration of the $i$-th piece. Then we have the timestamp $t_i=\sum_{j=1}^{i}{T_j}$ and the total duration $T=\Norm{\mathbf{T}}_1$. The $M$-piece trajectory $p:[0,T]\mapsto\mathbb{R}^m$ is defined by
\begin{equation}
\label{eq:TrajectoryDefinition}
p(t)=p_i(t),~\forall t\in[t_{i-1},t_i),~\forall i\in\cbrac{1,\dots,M}.
\end{equation}

To compute the unique solution for (\ref{eq:MultistageMinimumControl}), we directly enforce optimality conditions on the coefficient matrix $\mathbf{c}$.
Denote by $\mathbf{D}_0,\mathbf{D}_M\in\mathbb{R}^{s\times m}$ and $\mathbf{D}_i\in\mathbb{R}^{d_i\times m}$ the specified derivatives at boundaries and intermediate timestamp $t_i$, respectively. Each column of $\mathbf{D}_i$ is related to one dimension. Then, conditions at $t_i$ are formulated by $\mathbf{E}_i, \mathbf{F}_i\in\mathbb{R}^{2s\times 2s}$:
\begin{equation}
\begin{pmatrix} \mathbf{E}_i & \mathbf{F}_i \end{pmatrix}
\begin{pmatrix} \mathbf{c}_i \\ \mathbf{c}_{i+1} \end{pmatrix}
=
\begin{pmatrix} \mathbf{D}_i \\ \mathbf{0}_{\bar{d}_i\times m} \end{pmatrix},
\end{equation}
\begin{align}
\label{eq:EiDefinition}
\mathbf{E}_i=(&\beta(T_i), \dots,\beta^{(d_i-1)}(T_i),\\
&\beta(T_i), \dots, \beta^{(\bar{d}_i-1)}(T_i))\tp \nonumber,\\
\label{eq:FiDefinition}
\mathbf{F}_i=(&\mathbf{0},-\beta(0), \dots, -\beta^{(\bar{d}_i-1)}(0))\tp.
\end{align}
Specially, define $\mathbf{F}_0, \mathbf{E}_M\in\mathbb{R}^{s\times 2s}$ as
\begin{align}
\label{eq:F0Definition}
\mathbf{F}_0=&\rbrac{\beta(0), \dots, \beta^{(s-1)}(0)}\tp,\\
\label{eq:EMDefinition}
\mathbf{E}_M=&\rbrac{\beta(T_M), \dots, \beta^{(s-1)}(T_M)}\tp.
\end{align}
The linear system for the optimal coefficient matrix is
\begin{equation}
\label{eq:OptimalConditionLinearEquationSystem}
\mathbf{M}\mathbf{c}=\mathbf{b}
\end{equation}
where $\mathbf{M}\in\mathbb{R}^{2Ms\times2Ms}$ and $\mathbf{b}\in\mathbb{R}^{2Ms\times m}$ are
\begin{equation}
\mathbf{M}=
\begin{pmatrix}
\mathbf{F}_0 & \mathbf{0} & \mathbf{0} & \cdots & \mathbf{0} \\
\mathbf{E}_1 & \mathbf{F}_1 & \mathbf{0} & \cdots & \mathbf{0} \\
\mathbf{0} & \mathbf{E}_2 & \mathbf{F}_2 & \cdots & \mathbf{0} \\
\vdots & \vdots & \vdots & \ddots & \vdots \\
\mathbf{0} & \mathbf{0} & \mathbf{0} & \cdots & \mathbf{F}_{M-1} \\
\mathbf{0} & \mathbf{0} & \mathbf{0} & \cdots & \mathbf{E}_M \\
\end{pmatrix},
\end{equation}
\begin{equation}
\label{eq:bDefinition}
\mathbf{b}=\rbrac{\mathbf{D}_0\tp, \mathbf{D}_1\tp, \mathbf{0}_{m\times \bar{d}_1}, \dots, \mathbf{D}_{M-1}\tp, \mathbf{0}_{m\times \bar{d}_{M-1}}, \mathbf{D}_M\tp}\tp.
\end{equation}

It is essential that the uniqueness in Theorem~\ref{thm:OptimalityConditions} ensures the nonsingularity of $\mathbf{M}$ for any time vector $\mathbf{T}\succ\mathbf{0}$. Consequently, the unique solution $\mathbf{c}$ can be obtained via linear equation system (\ref{eq:OptimalConditionLinearEquationSystem}) with a banded matrix $\mathbf{M}$, i.e., a \textit{banded system}.
As for a nonsingular banded system, its \textit{Banded PLU Factorization} always exists~\cite{Horn2012MatrixA}, which can be employed to compute the solution with $O(M)$ time and space complexity~\cite{Golub2013MatrixCMP}. Therefore, without the need of cost functional, the unique solution of problem (\ref{eq:MultistageMinimumControl}) is obtained in the lowest complexity, by directly applying our optimality conditions.

\subsection{MINCO Trajectories With Spatial-Temporal Deformation}
For multicopters, there are often task-specific requirements apart from feasibility, such as the perception quality in active SLAM~\cite{Zhang2018PerceptionARHN} or the occlusion rate in aerial videography~\cite{Nageli2017RealMPAV}. These user-defined requirements majorly need to flexibly and adaptively deform both the spatial and temporal profile of a trajectory. Therefore, we select the intermediate points and the time vector as two salient parameters in (\ref{eq:MultistageMinimumControl}). Fortunately, the existence and uniqueness of solution guarantee the smoothness of sensitivity for them. An iterative procedure is then designed to conduct the spatial-temporal deformation with the lowest computation complexity per iteration.

We denote the intermediate points by $\mathbf{q}=\rbrac{q_1,\dots,q_{M-1}}$ where $q_i\in\mathbb{R}^m$ is a specified $0$-order derivative at $t_i$. Denote by $\mathbf{T}=\rbrac{T_1, \dots, T_M}\tp$ the time vector where $T_i\in\mathbb{R}_{>0}$. For any pair of $\mathbf{q}$ and $\mathbf{T}$, Theorem~\ref{thm:OptimalityConditions} naturally determines a trajectory belonging to a class of control effort minimizers, named MINCO hereafter. The MINCO trajectory class, denoted by $\mathfrak{T}_{\mathrm{MINCO}}$, is defined as
\begin{align*}
\mathfrak{T}_{\mathrm{MINCO}} = \Big\{&p(t):[0, T]\mapsto\mathbb{R}^m~\Big|~\mathbf{c}=\mathbf{c}(\mathbf{q},\mathbf{T})\text{ determined} ~\\ &\text{by Theorem~\ref{thm:OptimalityConditions},}~\forall~\mathbf{q}\in\mathbb{R}^{m\times(M-1)},~\mathbf{T}\in\mathbb{R}_{>0}^M\Big\}.
\end{align*}

The dimension $m$, the system order $s$, initial and terminal conditions are omitted here for brevity.
Intuitively, all trajectories in $\mathfrak{T}_{\mathrm{MINCO}}$ are compactly parameterized by only $\mathbf{q}$ and $\mathbf{T}$. Evaluating an element in $\mathfrak{T}_{\mathrm{MINCO}}$ directly follows our linear-complexity formulation.

We denote any user-defined objective (or constraint) on a trajectory by a $C^2$ function $\mathcal{K}(\mathbf{c},\mathbf{T})$ with available gradient. This objective on $\mathfrak{T}_{\mathrm{MINCO}}$ can be computed as
\begin{equation}
\mathcal{W}(\mathbf{q},\mathbf{T})=\mathcal{K}(\mathbf{c}(\mathbf{q},\mathbf{T}), \mathbf{T}).
\end{equation}
To accomplish deformation of $\mathfrak{T}_{\mathrm{MINCO}}$, the function $\mathcal{W}$ together with its gradient $\partial\mathcal{W}/\partial\mathbf{q}$ and $\partial\mathcal{W}/\partial\mathbf{T}$ are needed for a high-level optimizer to optimize the objective.
Obviously, evaluating $\mathcal{W}$ shares the same complexity as evaluating any trajectory in $\mathfrak{T}_{\mathrm{MINCO}}$.
The key procedure is to compute the gradient.
Now we give a linear-complexity scheme to compute $\partial\mathcal{W}/\partial\mathbf{q}$ and $\partial\mathcal{W}/\partial\mathbf{T}$ from the given $\partial\mathcal{K}/\partial\mathbf{c}$ and $\partial\mathcal{K}/\partial\mathbf{T}$.
We first rewrite the linear equation system (\ref{eq:OptimalConditionLinearEquationSystem}) as
\begin{equation}
\label{eq:ParameterizedOptimalConditionLinearEquationSystem}
\mathbf{M}(\mathbf{T})\mathbf{c}(\mathbf{q},\mathbf{T})=\mathbf{b}(\mathbf{q}).
\end{equation}
Without causing ambiguity, we omit parameters in $\mathbf{M},\mathbf{b},\mathbf{c},\mathcal{K}$ and $\mathcal{W}$ temporarily for simplicity. Any notation involving $\mathbf{c}$ is interpreted as $\mathbf{c}(\mathbf{q},\mathbf{T})$.
Denote by $q_{i,j}$ the $j$-th entry in $q_i$.

As for the gradient w.r.t. $\mathbf{q}$, we first differentiate both sides of (\ref{eq:ParameterizedOptimalConditionLinearEquationSystem}) w.r.t. $q_{i,j}$, which gives
\begin{equation}
\frac{\mathbf{\partial c}}{\partial q_{i,j}}=\mathbf{M}^{-1}\frac{\partial \mathbf{b}}{\partial q_{i,j}}.
\end{equation}
Then,
\begin{align}
\frac{\partial\mathcal{W}}{\partial q_{i,j}}&=\trace\cBrac{\rBrac{\frac{\mathbf{\partial c}}{\partial q_{i,j}}}\tp\frac{\partial\mathcal{K}}{\partial\mathbf{c}}} \nonumber\\
&=\trace\cBrac{\rBrac{\mathbf{M}^{-1}\frac{\partial \mathbf{b}}{\partial q_{i,j}}}\tp\frac{\partial\mathcal{K}}{\partial\mathbf{c}}} \nonumber\\
&=\trace\cBrac{\rBrac{\frac{\partial \mathbf{b}}{\partial q_{i,j}}}\tp\rBrac{\mathbf{M}\invtp\frac{\partial\mathcal{K}}{\partial\mathbf{c}}}},
\end{align}
where $\trace{(\cdot)}$ is the trace operation.
The definition of $\mathbf{b}(\mathbf{q})$ in (\ref{eq:bDefinition}) implies that $\partial\mathbf{b}/\partial q_{i,j}$ only has a single nonzero entry $1$ at its $(2i-1)s+1$ row and $j$ column.
Thus, stacking all the resultant scalars gives
\begin{equation}
\frac{\partial\mathcal{W}}{\partial q_i}=\rBrac{\mathbf{M}\invtp\frac{\partial\mathcal{K}}{\partial\mathbf{c}}}\tp e_{(2i-1)s+1},
\end{equation}
where $e_j$ is the $j$-th column of $\mathbf{I}_{2Ms}$. Now that we have already conducted the banded PLU factorization for $\mathbf{M}$ when we compute $\mathbf{c}$. We can reuse the factorization to avoid inverting $\mathbf{M}\tp$. Define a matrix $\mathbf{G}\in\mathbb{R}^{2Ms\times m}$ as
\begin{equation}
\mathbf{M}\tp\mathbf{G}=\frac{\partial\mathcal{K}}{\partial\mathbf{c}}.
\end{equation}
We only need to compute $\mathbf{G}$ once to obtain $\partial\mathcal{W}/\partial q_i$ for all $1\leq i<M$. Denote the factorization of $\mathbf{M}$ as $\mathbf{M}=\mathbf{P}\mathbf{L}\mathbf{U}$. Specifically, $\mathbf{L}$ is a banded matrix with zero upper bandwidth and all-ones diagonal entries. $\mathbf{U}$ is a banded matrix with zero lower bandwidth and nonzero diagonal entries because of the nonsingularity of $\mathbf{M}$. The pivoting matrix $\mathbf{P}$ simply changes the row order of the operand, satisfying $\mathbf{P}\tp\mathbf{P}=\mathbf{I}$. Consequently, the transpose also has a \textit{Banded LUP Factorization}~\cite{Horn2012MatrixA}. Specifically, $\mathbf{M}\tp=\bar{\mathbf{L}}\bar{\mathbf{U}}\mathbf{P}\tp$, where
\begin{equation}
\bar{\mathbf{L}}=\mathbf{U}\tp\rBrac{\mathbf{U}\circ\mathbf{I}}^{-1},~\bar{\mathbf{U}}=\rBrac{\mathbf{I}\circ\mathbf{U}}\mathbf{L}\tp,
\end{equation}
where the inverse is only done for a diagonal matrix and $\circ$ the Hadamard product. Then, $\mathbf{G}$ can be also computed in linear complexity through such a factorization. For convenience, we partition $\mathbf{G}$ into
\begin{equation}
\mathbf{G}=\rBrac{\mathbf{G}_0\tp,\mathbf{G}_1\tp,\dots,\mathbf{G}_{M-1}\tp,\mathbf{G}_M\tp}\tp
\end{equation}
such that $\mathbf{G}_0,\mathbf{G}_M\in\mathbb{R}^{s\times m}$ and $\mathbf{G}_i\in\mathbb{R}^{2s\times m}$ for $1\leq i<M$. After that, the gradient of $\mathcal{W}$ w.r.t. $\mathbf{q}$ is determined as
\begin{equation}
\label{eq:GradientQ}
\frac{\partial\mathcal{W}}{\partial\mathbf{q}}=\rBrac{\mathbf{G}_1\tp e_1,\dots,\mathbf{G}_{M-1}\tp e_1},
\end{equation}
where $e_1$ is the first column of $\mathbf{I}_{2s}$. This operation takes out $M-1$ specific columns in $\mathbf{G}\tp$.

As for the gradient w.r.t. $\mathbf{T}$, differentiating both sides of (\ref{eq:ParameterizedOptimalConditionLinearEquationSystem}) w.r.t. $T_i$ gives
\begin{equation}
\frac{\partial\mathbf{M}}{\partial T_i}\mathbf{c}+\mathbf{M}\frac{\mathbf{\partial c}}{\partial T_i}=\mathbf{0}.
\end{equation}
Thus,
\begin{align}
\frac{\partial\mathcal{W}}{\partial T_i}&=\frac{\partial\mathcal{K}}{\partial T_i}+\trace\cBrac{\rBrac{\frac{\mathbf{\partial c}}{\partial T_i}}\tp\frac{\partial\mathcal{K}}{\partial\mathbf{c}}} \nonumber\\
&=\frac{\partial\mathcal{K}}{\partial T_i}-\trace\cBrac{\rBrac{\frac{\partial\mathbf{M}}{\partial T_i}\mathbf{c}}\tp\mathbf{M}\invtp\frac{\partial\mathcal{K}}{\partial\mathbf{c}}} \nonumber\\
&=\frac{\partial\mathcal{K}}{\partial T_i}-\trace\cBrac{\mathbf{G}\tp\frac{\partial\mathbf{M}}{\partial T_i}\mathbf{c}}
\end{align}
The banded structure of $\mathbf{M}$ implies that
\begin{equation}
\mathbf{G}\tp\frac{\partial\mathbf{M}}{\partial T_i}\mathbf{c}=\mathbf{G}_i\tp\frac{\partial\mathbf{E}_i}{\partial T_i}\mathbf{c}_i.
\end{equation}
Then we obtain the gradient w.r.t. $T_i$ computed as
\begin{equation}
\label{eq:GradientTi}
\frac{\partial\mathcal{W}}{\partial T_i}=\frac{\partial\mathcal{K}}{\partial T_i}-\trace\cBrac{\mathbf{G}_i\tp\frac{\partial\mathbf{E}_i}{\partial T_i}\mathbf{c}_i},
\end{equation}
where $\partial\mathbf{E}_i/\partial T_i$ can be analytically derived from (\ref{eq:EiDefinition}).
Computing (\ref{eq:GradientTi}) for every $1\leq i\leq M$ gives $\partial\mathcal{W}/\partial\mathbf{T}$.

Finally, we finish the computation of $\partial\mathcal{W}/\partial\mathbf{q}$ and $\partial\mathcal{W}/\partial\mathbf{T}$.
It can be verified from both (\ref{eq:GradientQ}) and (\ref{eq:GradientTi}) that the gradient propagation is also done in $O(M)$ complexity.
As for $\mathcal{K}$, we make no assumption on its concrete form. Actually, the smoothness of $\mathcal{K}$ is not even needed if only the resultant $\mathcal{W}$ is $C^2$. In other word, the linear-complexity gradient propagation enjoys both efficiency and flexibility. By incorporating it into common optimizers, we can accomplish the spatial-temporal deformation of $\mathfrak{T}_{\mathrm{MINCO}}$ for a wide range of planning purposes while maintaining the local smoothness of trajectories.

\section{\texorpdfstring{Geometrically Constrained Flight \\ Trajectory Optimization}{Geometrically Constrained Flight Trajectory Optimization}}
In this section, we provide a unified framework for flight trajectory optimization with different time regularization $\rho(T)$, spatial constraints $\tilde{\mathcal{F}}$ and continuous-time constraints $\mathcal{G}$.
This framework indeed relaxes the original problem into $\mathfrak{T}_{\mathrm{MINCO}}$.
The spatial-temporal deformation is utilized to meet various feasibility requirements. Lightweight schemes are specially designed to eliminate geometrical constraints such that the trajectory can be freely deformed.
For continuous-time constraints, a time integral penalty functional is proposed to ensure the feasibility without sacrificing the scalability.
Finally, our framework transforms the constrained trajectory optimization into a sparse unconstrained one which can be reliably solved.

\subsection{Temporal Constraint Elimination}
\begin{figure}[ht]
	\centering
    \includegraphics[width=0.9\columnwidth]{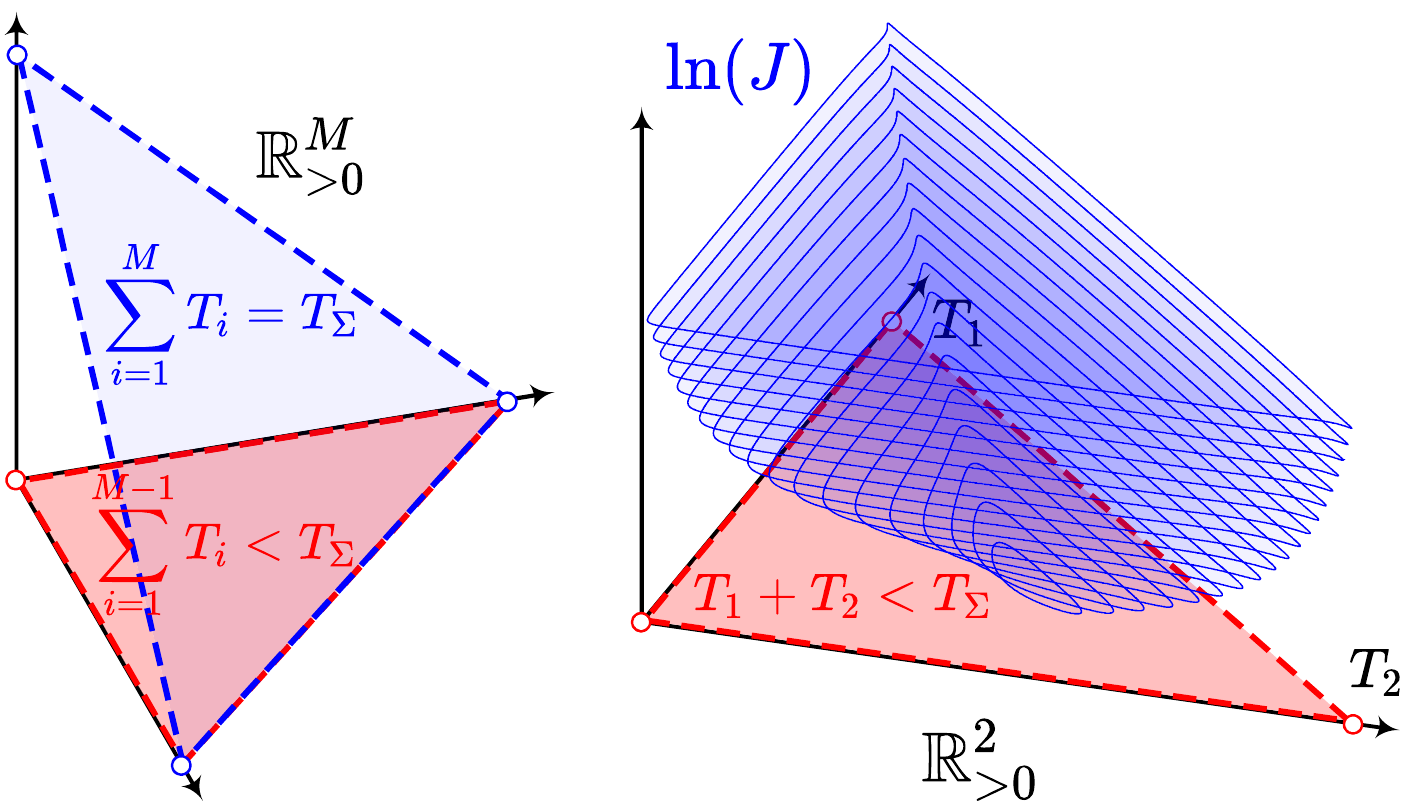}
    \caption{Left: Domain of $J$ on an $M$-piece trajectory with total time fixed as $T_\Sigma$. The domain is indeed the relative interior of an $(M-1)$-simplex in $\mathbb{R}^M_{>0}$. Right: Contour of $\ln{J}$ with $M=3$. The function goes to infinity as the time vector approaches the boundary of the open domain in $\mathbb{R}^2_{>0}$.\label{fig:LogObjectiveLevelSet}}
    \vspace{-0.0cm}
\end{figure}

Deforming MINCO needs standard optimizers that are often designed for Euclidean spaces. However, the trajectory definition and cost functional~(\ref{eq:TrajectoryOptimization}) both restrict the domain of $\mathbf{T}$ to simple manifolds, on which frequent retractions are needed during optimization. We give explicit diffeomorphisms for $\mathbf{T}$ such that surrogate variables are in Euclidean spaces. Thus, common efficient optimizers can be conveniently applied.

For polynomial splines, the control effort in (\ref{eq:TrajectoryOptimization}) is a function of $\mathbf{c}$ and $\mathbf{T}$, denoted by $J_c(\mathbf{c},\mathbf{T})$. Analytical expressions of $J_c$, $\partial J_c/\partial\mathbf{c}$, and $\partial J_c/\partial\mathbf{T}$ are available in~\cite{Bry2015AggressiveFO}. Now that $\mathfrak{T}_{\mathrm{MINCO}}$ are polynomial splines with coefficients determined by $\mathbf{c}(\mathbf{q},\mathbf{T})$, the cost functional of (\ref{eq:TrajectoryOptimization}) can be written as
\begin{equation}
\label{eq:EnergyTimeCost}
J(\mathbf{q},\mathbf{T})=J_q(\mathbf{q},\mathbf{T})+\rho(\norm{\mathbf{T}}_1),
\end{equation}
where $J_q$ is defined as $J_q(\mathbf{q},\mathbf{T})=J_c(\mathbf{c}(\mathbf{q},\mathbf{T}),\mathbf{T})$. Obviously, computing $J_q$, $\partial J_q/\partial\mathbf{q}$, and $\partial J_q/\partial\mathbf{T}$ from any provided $J_c$, $\partial J_c/\partial\mathbf{c}$, and $\partial J_c/\partial\mathbf{T}$ can be done in $O(M)$ complexity, as already shown in deformation of $\mathfrak{T}_{\mathrm{MINCO}}$.

It is natural to optimize $\mathbf{T}$ via $\partial J/\partial\mathbf{T}$.
However, $J_q(\mathbf{q},\mathbf{T})$ has its definition over $\mathbf{T}\in\mathbb{R}_{>0}^M$. It becomes unbounded when any $T_i$ approaches zero and no consecutively repeating points appear in $\mathbf{q}$. Besides, $\rho_f$ defined in (\ref{eq:HardDuration}) further restricts the domain of $J$ to $\sum_{i=1}^{M-1}{T_i}<T_\Sigma$, as shown in Fig.~\ref{fig:LogObjectiveLevelSet}.

We use diffeomorphisms to eliminate constraints for $\rho_f$ and $\rho_s$. Consider the domain of $\rho_f$ in (\ref{eq:HardDuration}),
\begin{equation}
\label{eq:TemporalDomain}
\mathcal{T}_f=\cBrac{\mathbf{T}\in\mathbb{R}^M~\Big|~\Norm{\mathbf{T}}_1=T_\Sigma,~\mathbf{T}\succ\mathbf{0}}.
\end{equation}
It is clear that $J(\mathbf{q},\cdot)$ is finite for a nontrivial $\mathbf{q}$ if and only if $\mathbf{T}\in \mathrm{RelInt}(\mathcal{T}_f)$, i.e., the relative interior of $\mathcal{T}_f$.
\begin{proposition}
    \label{ps:TemporalDiffeomorphism}
    $\mathcal{T}_f$ defined by (\ref{eq:TemporalDomain}) is diffeomorphic to $\mathbb{R}^{M-1}$. Denote by $\boldsymbol\tau=\rbrac{\tau_1,\dots,\tau_{M-1}}$ an element in $\mathbb{R}^{M-1}$. A $C^\infty$ diffeomorphism is given by the map below for $1\leq i < M$:
    \begin{equation}
    \label{eq:TDiffTransformation}
    T_i=\frac{e^{\tau_i}}{1+\sum_{j=1}^{M-1}{e^{\tau_j}}}T_\Sigma,~~T_M=T_\Sigma-\sum_{j=1}^{M-1}{T_j}.
    \end{equation}
\end{proposition}
By exploiting the explicit diffeomorphism (\ref{eq:TDiffTransformation}), we directly minimize the cost function $J$ over $\mathbb{R}^{M-1}$ via $\boldsymbol\tau$, because the domain constraints are satisfied by default.

Optimizing $\boldsymbol\tau$ requires gradient propagation. We partition the gradient as $\partial J_q/\partial \mathbf{T}=\rbrac{g_a\tp,g_b}\tp$, where $g_a\in\mathbb{R}^{M-1}$ and $g_b\in\mathbb{R}$.
Differentiating the layer in (\ref{eq:TDiffTransformation}) yields the gradient of $J$ w.r.t. $\boldsymbol\tau$,
\begin{equation}
\frac{\partial J}{\partial\boldsymbol\tau}=\frac{(g_a-g_b\mathbf{1})\circ e^{[\boldsymbol\tau]}}{1+\norm{e^{[\boldsymbol\tau]}}_1}-\frac{\rBrac{g_a\tp e^{[\boldsymbol\tau]}-g_b\norm{e^{[\boldsymbol\tau]}}_1}e^{[\boldsymbol\tau]}}{\rBrac{1+\norm{e^{[\boldsymbol\tau]}}_1}^2},
\end{equation}
where $e^{[\cdot]}$ is the entry-wise exponential map, and $\mathbf{1}$ an all-ones vector. If an initial guess $\mathbf{T}$ is specified, the corresponding $\boldsymbol\tau$ can be computed via the inverse map of the diffeomorphism, given by $\tau_i=\ln{\rBrac{T_i/T_M}}$ for $1\leq i < M$. As for $\rho_s$ in (\ref{eq:SoftDuration}), only $\mathbf{T}\succ\mathbf{0}$ needs to be ensured.
It suffices to use $\mathbf{T}=e^{[\boldsymbol\tau]}$ as the diffeomorphism between $\mathbb{R}^M$ and $\mathbb{R}_{>0}^M$.

For either $\rho_f$ or $\rho_s$, we denote the diffeomorphism by $\mathbf{T}(\boldsymbol\tau)$. Unconstrained optimization on $\boldsymbol\tau$ can be directly conducted to minimize $J(\mathbf{q},\mathbf{T}(\boldsymbol\tau))$.
Although $\mathbf{T}(\boldsymbol\tau)$ does not preserve convexity, the original cost $J(\mathbf{q},\mathbf{T})$ is already nonconvex as given in (\ref{eq:ParameterizedOptimalConditionLinearEquationSystem}).
Thus, the only concern is whether $\mathbf{T}(\boldsymbol\tau)$ brings new local minima in the space of $\boldsymbol\tau$ or eliminates local minima in the space of $\mathbf{T}$.
\begin{proposition}
    \label{ps:DiffeomorphismKeepsLocalMin}
    Denote by $F:\mathbb{D}_\mathrm{F}\mapsto\mathbb{R}$ any $C^2$ function with a convex open domain $\mathbb{D}_\mathrm{F}\in\mathbb{R}^N$. Given any $C^2$ diffeomorphism $\mathbf{G}:\mathbb{R}^N\mapsto\mathbb{D}_\mathrm{F}$, define $H:\mathbb{R}^N\mapsto\mathbb{R}$ as $H(y)=F(\mathbf{G}(y))$ for $y\in\mathbb{R}^N$. For any $x\in\mathbb{D}_\mathrm{F}$ and $y\in\mathbb{R}^N$ satisfying $x=\mathbf{G}(y)$ or equivalently $y=\mathbf{G}^{-1}(x)$, the following statements always hold:
    \begin{itemize}
        \item $\grad{F(x)}=\mathbf{0}$ if and only if $\grad{H(y)}=\mathbf{0}$;
        \item $\hessian{F(x)}$ is positive-definite (or positive-semidefinite) at $\grad{F(x)}=\mathbf{0}$, if and only if $\hessian{H(y)}$ is positive-definite (or positive-semidefinite) at $\grad{H(y)}=\mathbf{0}$.
    \end{itemize}
\end{proposition}
\begin{proof}
    See Appendix~\ref{apd:DiffeomorphismKeepsLocalMin}.
\end{proof}

Proposition~\ref{ps:DiffeomorphismKeepsLocalMin} confirms that $\mathbf{T}(\boldsymbol\tau)$ preserves the first/second-order necessary optimality conditions and second-order sufficient optimality conditions~\cite{Nocedal2006NumericalOP}. It is also applicable to substitute the exponential map in this subsection with any $C^2$ diffeomorphism from $\mathbb{R}$ to $\mathbb{R}_{>0}$ for a better numerical condition.
In the sense of commonly-used optimality conditions, our constraint elimination does not produce extra spurious local minima or cancel any existing one.

\subsection{Spherical Spatial Constraint Elimination}
\begin{figure}[ht]
    \centering
    \includegraphics[width=0.75\columnwidth]{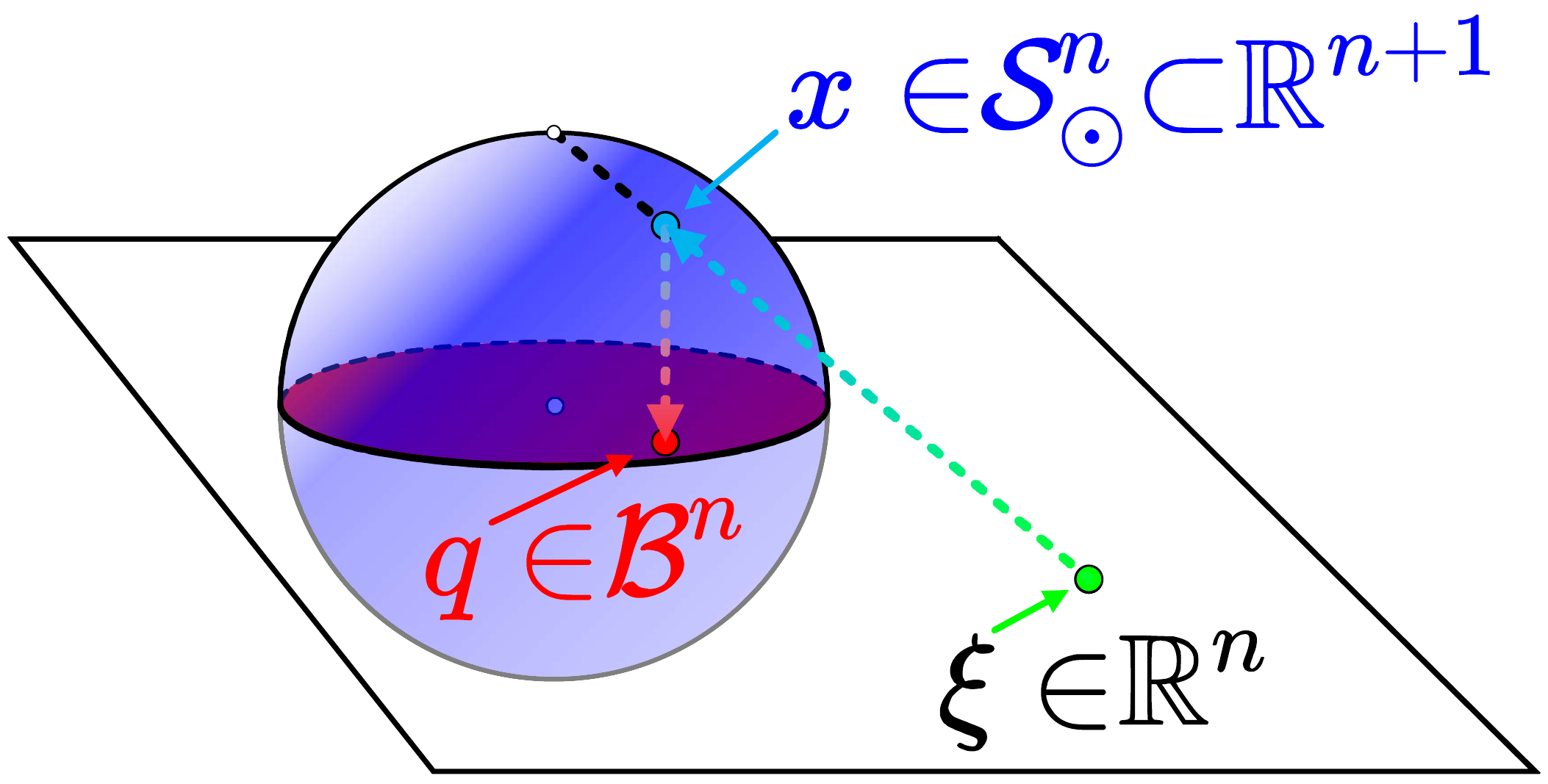}
    \caption{Inverse stereographic projection $f_s$ maps the Euclidean space $\mathbb{R}^n$ onto a sphere without north pole $\mathcal{S}^n_\odot$ in an $(n+1)$-dimensional space. The orthographic projection $f_o$ maps $\mathcal{S}^n_\odot$ onto an $n$-dimensional ball $\mathcal{B}^n$. The variable $\xi$ moves freely in $\mathbb{R}^n$ while the transformed variable $q$ stays in $\mathcal{B}^n$. Optimization on $\xi$ becomes unconstrained when $q$ is constrained by a ball.\label{fig:InverseProjection}}
    \vspace{-0.0cm}
\end{figure}

We enforce motion safety by confining trajectories into the feasible region $\tilde{\mathcal{F}}$. Although $\tilde{\mathcal{F}}$ is nonconvex, it is a union of convex primitives that are sequentially connected. If all pieces have been assigned into these primitives, the safety constraint on each piece becomes convex and thus can be conveniently encoded in $\mathcal{G}$. Owing to the feature of MINCO, the traverse time for every primitive can be directly optimized. Thus, we fix the piece assignment before optimization, rather than resorting to integer variables during optimization~\cite{Tordesillas2019Faster}. Consequently, intermediate points should be contained by the overlap between primitives, forming inequalities. For Inequality Constrained Problems (ICPs), general methods successively approximate the constraints via additional parameters. However, we aim to apply the constraints directly and efficiently. Therefore, we propose spatial constraint elimination to enforce them exactly, leveraging their geometrical properties.

Consider the constraint $q\in\mathcal{P}\subset\mathbb{R}^n$ where $\mathcal{P}$ is a closed ball. Its dimension satisfies $n\leq m$ since a low-dimensional constraint also exists in $\mathbb{R}^m$. If $\mathcal{P}$ is a closed ball $\mathcal{P}^\mathcal{B}$ centered at point $o$ with radius $r$,
\begin{equation}
\mathcal{P}^\mathcal{B}=\cBrac{x\in\mathbb{R}^n~\Big|~\Norm{x-o}_2\leq r},
\end{equation}
We utilize a smooth surjection to map $\mathbb{R}^n$ to $\mathcal{P}^\mathcal{B}$ such that optimization over $\mathbb{R}^n$ implicitly satisfies the constraint $\mathcal{P}^\mathcal{B}$.
As illustrated in Fig.~\ref{fig:InverseProjection}, the map is a composition of the inverse stereographic projection and the orthographic projection.
First, we utilize the inverse stereographic projection to map $\mathbb{R}^n$ to $\mathcal{S}_\odot^n$, where  $\mathcal{S}_\odot^n$ is a unit sphere without north pole, i.e.,
\begin{equation}
\mathcal{S}_\odot^n=\cBrac{x\in\mathbb{R}^{n+1}~\Big|~\Norm{x}_2=1,~x_{n+1}<1}.
\end{equation}
The inverse stereographic projection $f_s$ is define as
\begin{equation}
f_s(x)=\frac{\rbrac{2x\tp,x\tp x-1}\tp}{x\tp x+1}\in\mathcal{S}_\odot^n,~\forall x\in\mathbb{R}^n.
\end{equation}
Note that $f_s$ is a diffeomorphism between $\mathbb{R}^n$ and $\mathcal{S}_\odot^n$~\cite{Lee2012IntroductionSMM}.
We then project $\mathcal{S}_\odot^n$ from $\mathbb{R}^{n+1}$ back in $\mathbb{R}^n$ to obtain
\begin{equation}
\mathcal{B}^n=\cBrac{x\in\mathbb{R}^n~\Big|~\Norm{x}_2\leq1}.
\end{equation}
The map is described by
\begin{equation}
f_o(x)=\rbrac{x_1,\dots,x_n}\tp\in\mathcal{B}^n,~\forall x\in\mathcal{S}_\odot^n,
\end{equation}
which is indeed a smooth surjection onto $\mathcal{B}^n$.
Each point in $\mathcal{B}^n$, except the center, is paired with two points in $\mathcal{S}_\odot^n$. The composition of $f_s$, $f_o$, and a linear transformation, is a smooth surjection:
\begin{equation}
\label{eq:SmoothSurjectionForClosedBall}
f_\mathcal{B}(x)=o+\frac{2rx}{x\tp x+1}\in\mathcal{P}^{\mathcal{B}},~\forall x\in\mathbb{R}^n.
\end{equation}

The map $f_\mathcal{B}$ introduces a new coordinate, denoted by $\xi$, such that optimizing $\xi$ over $\mathbb{R}^n$ always satisfies the constraint on $q$ described by $\mathcal{P}^\mathcal{B}$. For the $i$-th intermediate point $q_i$, denote by $\xi_i$ the corresponding new coordinate.
Accordingly, denote by $\boldsymbol\xi$ the new coordinate for $\mathbf{q}$. Optimizing $\boldsymbol\xi$ requires gradient propagation for $\partial J/\partial\mathbf{q}$.
Denote by $g_i$ the $i$-th entry $\partial J/\partial q_i$ in $\partial J/\partial\mathbf{q}$. Differentiating the layer $f_\mathcal{B}$ gives the gradient
\begin{equation}
\frac{\partial J}{\partial \xi_i}=\frac{2r_i g_i}{\xi_i\tp\xi_i+1}-\frac{4r_i(\xi_i\tp g_i)\xi_i}{(\xi_i\tp\xi_i+1)^2}.
\end{equation}

If the optimization needs to start from an initial guess $\mathbf{q}$, the backward evaluation of $\boldsymbol\xi$ can be done by using a local inverse of $f_\mathcal{B}$, given by $\xi_i$ for $1\leq i < M$:
\begin{equation}
\xi_i=\frac{r_i-\sqrt{r_i^2-\Norm{q_i-o_i}_2^2}}{\Norm{q_i-o_i}_2^2}(q_i-o_i).
\end{equation}

\begin{figure}[t]
	\centering
	\includegraphics[width=1.0\columnwidth]{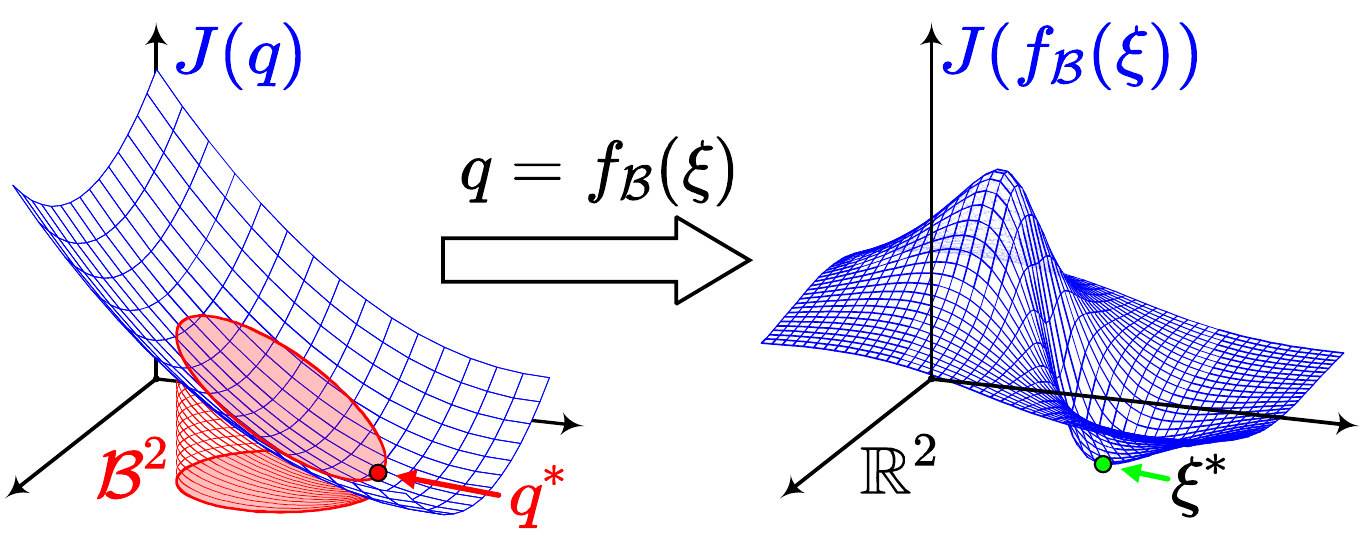}
	\caption{Constrained minimum $q^*$ of a convex function $J(q)$ within a 2-D ball. Transformed by $f_\mathcal{B}$, the resultant function $J(f_\mathcal{B}(\xi))$ becomes nonconvex but it preserves the local minimum $\xi^*$ satisfying $q^*=f_\mathcal{B}(\xi^*)$ with no additional local minimum introduced. \label{fig:TransformedQuadratic}}
	\vspace{-0.0cm}
\end{figure}

Similarly, we analyze influences that the smooth surjection $f_\mathcal{B}$ imposes on the constrained local minima in $\mathcal{P}^\mathcal{B}$. Although $f_\mathcal{B}$ lacks the one-to-one correspondence as diffeomorphisms possess, its components are all well-formed. Firstly, $f_o$ only takes the first $n$ entries of a point. This operation preserves at least the first-order necessary conditions for local minima in either $\mathcal{B}^n$ or $\mathcal{S}_\odot^n$. Secondly, $f_s$ is a diffeomorphism between $\mathcal{S}_\odot^n$ and $\mathbb{R}^n$, thus satisfying Proposition~\ref{ps:DiffeomorphismKeepsLocalMin}. Therefore, we can also confirm that $f_\mathcal{B}$ does not produce extra spurious local minima or cancel any existing one. As shown in Fig.~\ref{fig:TransformedQuadratic}, the constrained minimum within a 2-D ball is transformed into an unconstrained minimum.

\subsection{Polyhedral Spatial Constraint Elimination}

\begin{figure}[ht]
    \centering
    \includegraphics[width=1.0\columnwidth]{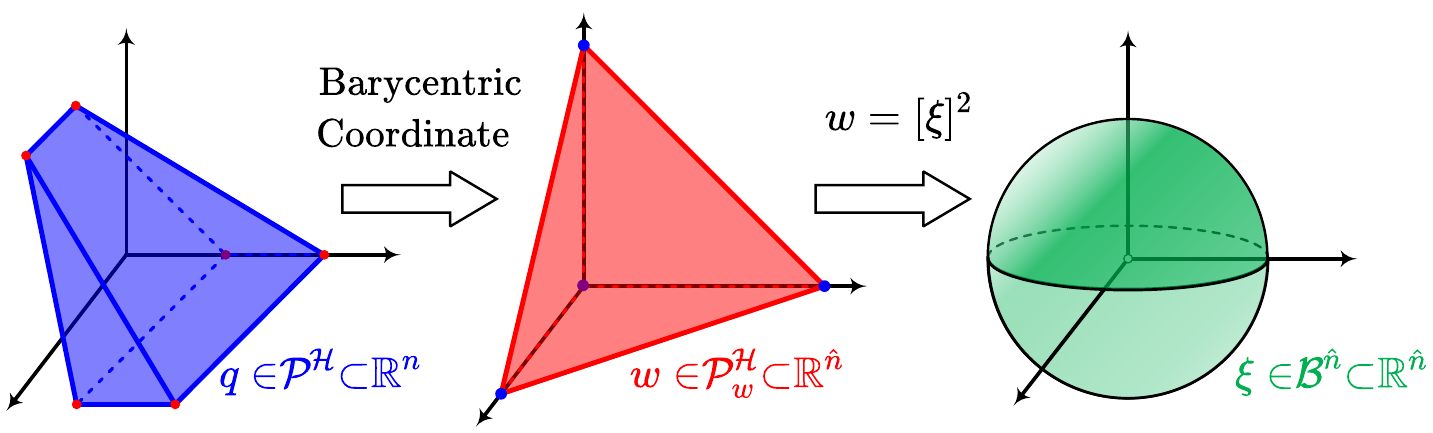}
    \caption{Transformations on a convex polytope. A convex polytope $\mathcal{P}^\mathcal{H}$ with $\hat{n}+1$ vertices is indeed a standard $\hat{n}$-simplex in the barycentric coordinate. The simplex $\mathcal{P}^\mathcal{H}_w$ is then the image of an entry-wise square map $[\cdot]^2$ with ball-shaped domain, which can be eliminated as in Fig.~\ref{fig:InverseProjection}.\label{fig:BarycentricTransform}}
    \vspace{-0.0cm}
\end{figure}

Now we consider the elimination of polyhedral constraints. Specifically, $\mathcal{P}$ is a closed convex polytope $\mathcal{P}^\mathcal{H}$ defined by
\begin{equation}
\label{eq:HPolytope}
\mathcal{P}^\mathcal{H}=\cBrac{x\in\mathbb{R}^n~\Big|~\mathbf{A}x\preceq b}.
\end{equation}
where $\mathrm{Int}\rbrac{\mathcal{P}^\mathcal{H}}\neq\varnothing$ according to (\ref{eq:LocallySequentialConnection}). Common optimization algorithms use the $\mathcal{H}$-representation of $\mathcal{P}^\mathcal{H}$ as linear inequality constraints. In our framework, we exploit their geometrical property to eliminate these constraints so that $\mathfrak{T}_{\mathrm{MINCO}}$ can be freely deformed.
To achieve this, we use the $\mathcal{V}$-representation of $\mathcal{P}^\mathcal{H}$ instead, where any $q\in\mathcal{P}^\mathcal{H}$ has a (general) barycentric coordinate, i.e., a convex combination of vertices.
To obtain the vertices, we apply the efficient convex hull algorithm~\cite{Barber1996QuickH} to the dual of $\mathcal{P}^\mathcal{H}$ based on an interior point calculated by Seidel's algorithm~\cite{Seidel1991SmallDLP}. Note that this procedure produces negligible overhead in our case ($n\leq4$).

The procedure to eliminate a polytope constraint is illustrated in the Fig.~\ref{fig:BarycentricTransform}.
We denote all $\hat{n}+1$ vertices of $\mathcal{P}^\mathcal{H}$ by $\rbrac{v_0,\dots,v_{\hat{n}}}$, where $v_i\in\mathbb{R}^n$ for each $i$. The barycentric coordinate of a point $q\in\mathcal{P}^\mathcal{H}$ consists of the weights for these vertices. To obtain a more compact form, define $\hat{v}_i=v_i-v_0$ and $\hat{\mathbf{V}}=\rbrac{\hat{v}_1,\dots,\hat{v}_{\hat{n}}}$, then the position can be calculated as
\begin{equation}
\label{eq:BarycentricEuclideanTransformation}
q=v_0+\hat{\mathbf{V}}w,
\end{equation}
where $w=\rbrac{w_1,\dots,w_{\hat{n}}}\tp\in\mathbb{R}^{\hat{n}}$ is the last $\hat{n}$ entries in the barycentric coordinate. The set of coordinates in convex combinations can also be written as
\begin{equation}
\label{eq:SimplexInterior}
\mathcal{P}_w^\mathcal{H}=\cBrac{w\in\mathbb{R}^{\hat{n}}~\Big|~w\succeq\mathbf{0},~\norm{w}_1\leq1}.
\end{equation}
The \textit{Main Theorem of Polytope Theory} in~\cite{Toth2017HandbookDCG} confirms the equivalence between $\mathcal{P}_w^\mathcal{H}$ and $\mathcal{P}^\mathcal{H}$ under (\ref{eq:BarycentricEuclideanTransformation}). The polytope is exactly converted into a standard $(\hat{n}+1)$-simplex by simply adding auxiliary variables and applying a linear map to $q$.
This process does not produce additional nonlinearity in the optimization problem except that the dimension of decision variables is increased. Therefore, we only consider the decision variables on $q$ as the corresponding $w$ hereafter.

The simplex (\ref{eq:SimplexInterior}) can be eliminated by nonlinear transformations. We first use an entry-wise square map $[\cdot]^2:\mathbb{R}^{\hat{n}}\mapsto\mathbb{R}^{\hat{n}}$ proposed in~\cite{Sisser1981EliminationBOP} to eliminate nonnegativity constraints using $w=[x]^2$. Then, the constraint $\mathcal{P}_w^\mathcal{H}$ on $w$ is transformed into a closed unit ball $\mathcal{B}^{\hat{n}}$ on $x$,
\begin{equation}
\label{eq:ClosedUnitHyperball}
\mathcal{B}^{\hat{n}}=\cBrac{x\in\mathbb{R}^{\hat{n}}~\Big|~\norm{x}_2\leq1}.
\end{equation}
Consequently, we can utilize the smooth surjection $f_\mathcal{B}$ in (\ref{eq:SmoothSurjectionForClosedBall}) again. The composition of (\ref{eq:BarycentricEuclideanTransformation}), $[\cdot]^2$, and $f_\mathcal{B}$ yields a smooth surjection $f_\mathcal{H}$ from $\mathbb{R}^{\hat{n}}$ onto $\mathcal{P}^{\mathcal{H}}$:
\begin{equation}
f_\mathcal{H}(x)=v_0+\frac{4\hat{\mathbf{V}}[x]^2}{(x\tp x+1)^2}\in\mathcal{P}^{\mathcal{H}},~\forall x\in\mathbb{R}^{\hat{n}}.
\end{equation}

A new coordinate $\xi$ is introduced by $f_\mathcal{H}$, where any $\xi\in\mathbb{R}^{\hat{n}}$ ensures $q\in\mathcal{P}^\mathcal{H}$. The boundary of $\mathcal{P}^\mathcal{H}$ is also attainable. Similarly, $\boldsymbol\xi$ is the new coordinate for $\mathbf{q}$. Optimizing $\boldsymbol\xi$ requires gradient propagation. Denote by $g_i$ the $i$-th gradient $\partial J/\partial q_i$ in $\partial J/\partial\mathbf{q}$, then differentiating the layer $f_\mathcal{H}$ gives
\begin{equation}
\frac{\partial J}{\partial \xi_i}=\frac{8\xi_i\circ \hat{\mathbf{V}}\tp g_i}{(\xi_i\tp\xi_i+1)^2}-\frac{16g_i\tp \hat{\mathbf{V}}[\xi_i]^2}{(\xi_i\tp\xi_i+1)^3}\xi_i.
\end{equation}

If an initial guess $\mathbf{q}$ is specified, the corresponding $\boldsymbol\xi$ can be computed via the local inverse of $f_\mathcal{H}$. The barycentric coordinate of each $q_i$ can be obtained using the analytic approach by Warren et al.~\cite{Warren2007Barycentric}. After that the analytic local inverses of $[\cdot]^2$ and $f_\mathcal{B}(\cdot)$ give us the desired $\xi_i$. Another flexible way is to directly minimize the squared distance between $f_\mathcal{H}(\xi)$ and the given $q_i$. Both approaches have negligible time consumption but promising results.

The map $[\cdot]^2$ in $f_\mathcal{H}$ presents additional nonlinearity into optimization. Fortunately, variable transformation via $[\cdot]^2$ is a special case of the inequality-to-equality conversion~\cite{Bertsekas2016NPG}. Concretely, the inequality constraints are $-w\preceq\mathbf{0}$. By introducing additional variables $x$, the equivalent equality constraints are $-w+[x]^2=\mathbf{0}$, yielding $w=[x]^2$. Such type of constraint conversion is proved to preserve first/second-order necessary conditions and second-order sufficient conditions for ICPs by Bertsekas as provided in Section 4.3 of~\cite{Bertsekas2016NPG}. We confirm that the additional nonlinearity in $f_\mathcal{H}$ does not exclude the desired minimum or produce undesired minimum practically.

Direct constraints on $\mathbf{q}$ are eliminated for either $\mathcal{P}^\mathcal{B}$ or $\mathcal{P}^\mathcal{H}$ using a smooth surjection $\mathbf{q}(\boldsymbol\xi)$. We can conduct unconstrained optimization on $\boldsymbol\xi$ to minimize $J(\mathbf{q}(\boldsymbol\xi),\mathbf{T}(\boldsymbol\tau))$ hereafter.

\subsection{Time Integral Penalty Functional}
After eliminating direct constraints, $\mathfrak{T}_{\mathrm{MINCO}}$ can be freely deformed to meet the continuous-time constraints $\mathcal{G}$. However, enforcing $\mathcal{G}$ over the entire trajectory involves infinitely many inequalities that cannot be solved by constrained optimization. It further needs temporal discretization that usually produces a large number of decision variables. To preserve the sparsity of trajectory parameterization, we decouple the resolution of constraint evaluation from the number of decision variables. Inspired by the constraint transcription~\cite{Jennings1990ComputationalAFQCO}, we transform $\mathcal{G}$ into finite constraints by integral of constraint violations.

For a trajectory $p:[0,T]\mapsto\mathbb{R}^m$, we define
\begin{equation}
\label{eq:ConstraintTranscription}
I_\mathcal{G}^k\sbrac{p}=\int_0^T{\max{\sbrac{\mathcal{G}\rbrac{p(t),\dots,p^{(s)}(t)},\mathbf{0}}^k}~\df{t}},
\end{equation}
where $k\in\mathbb{R}_{>0}$ and $\max\sBrac{\cdot,\mathbf{0}}^k$ is the composition of the entry-wise maximum and an entry-wise power function. Specifically, smoothing is needed if $k\leq1$. The functional-type constraint is then equivalent to equality constraints $I_\mathcal{G}^k\sbrac{p}=\mathbf{0}$. Actually, $I_\mathcal{G}^k\sbrac{p}$ is a function of trajectory parameters, which we adopt as penalty terms. If $k=1$, it forms a nonsmooth but exact penalty. If $k>1$, it forms a differentiable strictly convex penalty. Thus either $I_\mathcal{G}^3\sbrac{p}$ or a smoothing approximation of $I_\mathcal{G}^1\sbrac{p}$ can be adopted. For simplicity, we utilize $I_\mathcal{G}^3\sbrac{p}$ hereafter unless otherwise specified. There are two reasons for choosing a penalty function method. Firstly, the integral in (\ref{eq:ConstraintTranscription}) can only be evaluated numerically, making the constraint approximation inevitable.
Secondly, penalty methods have no requirement on a feasible initial guess which is nontrivial to construct.

We define the time integral penalty functional for $p(t)$ as
\begin{equation}
I_\mathcal{G}\sbrac{p}=\chi\tp I_\mathcal{G}^k\sbrac{p}.
\end{equation}
where $\chi\in\mathbb{R}^{n_g}_{\geq0}$ is a weight vector. Normally, $\chi$ should contain large constants. If no constraint is violated, $I_\mathcal{G}\sbrac{p}$ remains zero. Otherwise, if any part on $p(t)$ violates any constraint in $\mathcal{G}$, the penalty functional $I_\mathcal{G}\sbrac{p}$ grows rapidly. By incorporating $I_\mathcal{G}\sbrac{p}$ into the cost functional, continuous-time constraints are enforced within an acceptable tolerance.

Practically, $I_\mathcal{G}\sbrac{p}$ can only be evaluated by quadrature. To conduct the quadrature, we first define a sampled function $\mathcal{G}_\tau:\mathbb{R}^{2s\times m}\times\mathbb{R}_{>0}\times[0,1]\mapsto\mathbb{R}^{n_g}$ as
\begin{equation}
\mathcal{G}_\tau(\mathbf{c}_i,T_i,\tau)=\mathcal{G}\rBrac{\mathbf{c}_i\tp\beta(T_i\cdot\tau),\dots,\mathbf{c}_i\tp\beta^{(s)}(T_i\cdot\tau)},
\end{equation}
where $\tau\in[0,1]$ is a normalized stamp. Then the quadrature for $I_\mathcal{G}\sbrac{p}$, denoted by $I:\mathbb{R}^{2Ms\times m}\times\mathbb{R}_{>0}^M\mapsto\mathbb{R}_{>0}$, is computed as a weighted sum of the sampled penalty,
\begin{equation}
\label{eq:TimeIntegralPenalty}
I(\mathbf{c},\mathbf{T})=\sum_{i=1}^{M}{\frac{T_i}{\kappa_i}\sum_{j=0}^{\kappa_i}\bar{\omega}_j\chi\tp\max{\sbrac{\mathcal{G}_\tau\rbrac{\mathbf{c}_i,T_i,\frac{j}{\kappa_i}},\mathbf{0}}^k}},
\end{equation}
where $\kappa_i$ controls the resolution. We choose the trapezoidal rule $\rbrac{\bar{\omega}_0,\bar{\omega}_1,\dots,\bar{\omega}_{\kappa_i-1},\bar{\omega}_{\kappa_i}}=\rbrac{1/2,1,\dots,1,1/2}$ because of its reliable performance for ill-shaped $C^2$ integrands in our practice. Intuitively, $I(\mathbf{c},\mathbf{T})$ is a differentiable approximation to $I_\mathcal{G}[p]$, whose precision is adjustable through $\kappa_i$. The value and gradient at most timestamps can be parallelly computed then directly combined as one.

\subsection{Trajectory Optimization via Unconstrained NLP}
Due to $\mathcal{G}$ and $\mathcal{F}$ in (\ref{eq:TrajectoryOptimization}), the optimal trajectory parameterization is generally hard to know.
Unlike traditional methods approximating solutions via a large number of variables~\cite{Betts2010PracticalNOC}, we propose to solve a lightweight relaxed optimization via unconstrained NLP, where the spatial-temporal deformation of $\mathfrak{T}_{\mathrm{MINCO}}$ is applied. The relaxation to (\ref{eq:TrajectoryOptimization}) is defined as
\begin{equation}
\label{eq:Relaxation}
\min_{\boldsymbol\xi, \boldsymbol\tau}~{J(\mathbf{q}(\boldsymbol\xi),\mathbf{T}(\boldsymbol\tau))+I(\mathbf{c}(\mathbf{q}(\boldsymbol\xi),\mathbf{T}(\boldsymbol\tau)),\mathbf{T}(\boldsymbol\tau))},
\end{equation}
where $J$ is the time-regularized control effort (\ref{eq:EnergyTimeCost}) for $\mathfrak{T}_{\mathrm{MINCO}}$ and $I$ is the quadrature of penalty functional (\ref{eq:TimeIntegralPenalty}). Note that any task-specific requirement, either objectives or constraints, can be combined in (\ref{eq:Relaxation}) without affecting its structure.

To generate trajectories for a flat multicopter, we first parameterize its flat-output trajectory as $\mathfrak{T}_{\mathrm{MINCO}}$. After assigning a fixed number of pieces into each $\mathcal{P}_i$, variable transformations are applied to eliminate all direct constraints.
User-defined $\mathcal{G}_\mathcal{D}$ are also transformed into $\mathcal{G}$ via $\Psi_x$ and $\Psi_u$.
Finally, we obtain the cost function (\ref{eq:Relaxation}).
Apparently, the gradient propagation is derived for all layers except $\Psi_x$ and $\Psi_u$. One can either apply \textit{Automatic Differentiation} (AD)~\cite{Griewank2008EvaluatingD} to $\Psi_x$ and $\Psi_u$ or derive the gradient propagation analytically by following the reverse-mode AD. The efficiency is the same as the flatness map as ensured by \textit{Baur-Strassen Theorem}~\cite{Baur1983complexityPD}. The differentiation is only needed for the given flat dynamics once and for all. With available gradient, the relaxation (\ref{eq:Relaxation}) is then solved by the L-BFGS algorithm~\cite{Liu1989LBFGS}.

\section{Applications}
\subsection{Large-Scale Unconstrained Control Effort Minimization}
\begin{figure}[ht]
    \centering
    \includegraphics[width=0.95\columnwidth]{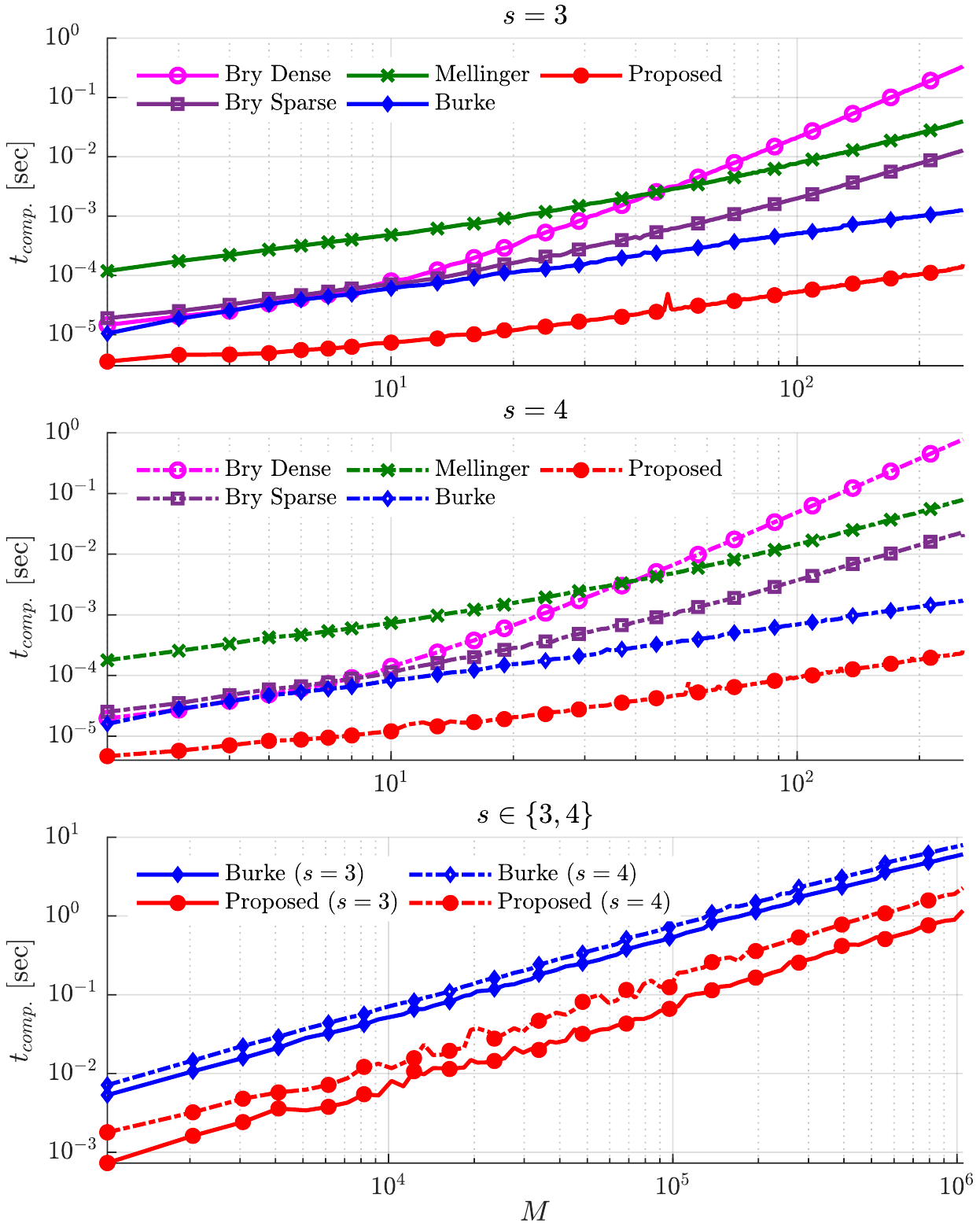}
    \caption{Computation time $t_{comp.}$ under different piece numbers $M$. The top and middle figures give the performance for jerk minimization ($s=3$) and snap minimization ($s=4$), respectively. The bottom figure shows the efficiency of two linear-complexity schemes for very large-scale problems.\label{fig:BenchmarkSubproblems}}
    \vspace{-0.0cm}
\end{figure}

We benchmark several existing schemes over problem (\ref{eq:MultistageMinimumControl}), including the QP formulation by Mellinger and Kumar~\cite{Mellinger2011MinimumST}, the closed-form solution by Bry et al.~\cite{Bry2015AggressiveFO}, and the linear-complexity scheme by Burke et al.~\cite{Burke2020GeneratingMSTRF}. We implement all these schemes in C++11 without any explicit hardware acceleration. Mellinger's scheme is implemented using OSQP~\cite{Stellato2020OSQP}. Bry's solution is evaluated by both a dense solver and a sparse one~\cite{Demmel1999SupernodalLU}. Burke's scheme is re-implemented here for fairness, which is faster than the original one~\cite{Burke2020GeneratingMSTRF}. The benchmark is conducted on an Intel Core i7-8700 CPU under Linux.

The performance is reported in Fig.~\ref{fig:BenchmarkSubproblems}. Both jerk $s=3$ and snap $s=4$ are minimized as defined in (\ref{eq:MultistageMinimumControl}). Mellinger's scheme~\cite{Mellinger2011MinimumST} only performs better than the dense evaluation of Bry's closed-form solution~\cite{Bry2015AggressiveFO} on middle-scale problems ($10^1<M<10^3$). Burke's scheme~\cite{Burke2020GeneratingMSTRF} benefits from its linear complexity, thus it can solve large-scale problems ($10^4<M<10^6$). Our scheme improves the computation speed by orders of magnitude against the others at any problem scale while retaining $O(M)$ complexity.

In conclusion, our optimality conditions provide a practical way to directly construct the solution of problem (\ref{eq:MultistageMinimumControl}), which possesses simplicity, efficiency, stability and scalability.
The trajectory class $\mathfrak{T}_{\mathrm{MINCO}}$ can serve as a reliable submodule of our optimization framework.

\subsection{Trajectory Generation Within Safe Flight Corridors}
\begin{figure}[ht]
    \centering
    \includegraphics[width=1.0\columnwidth]{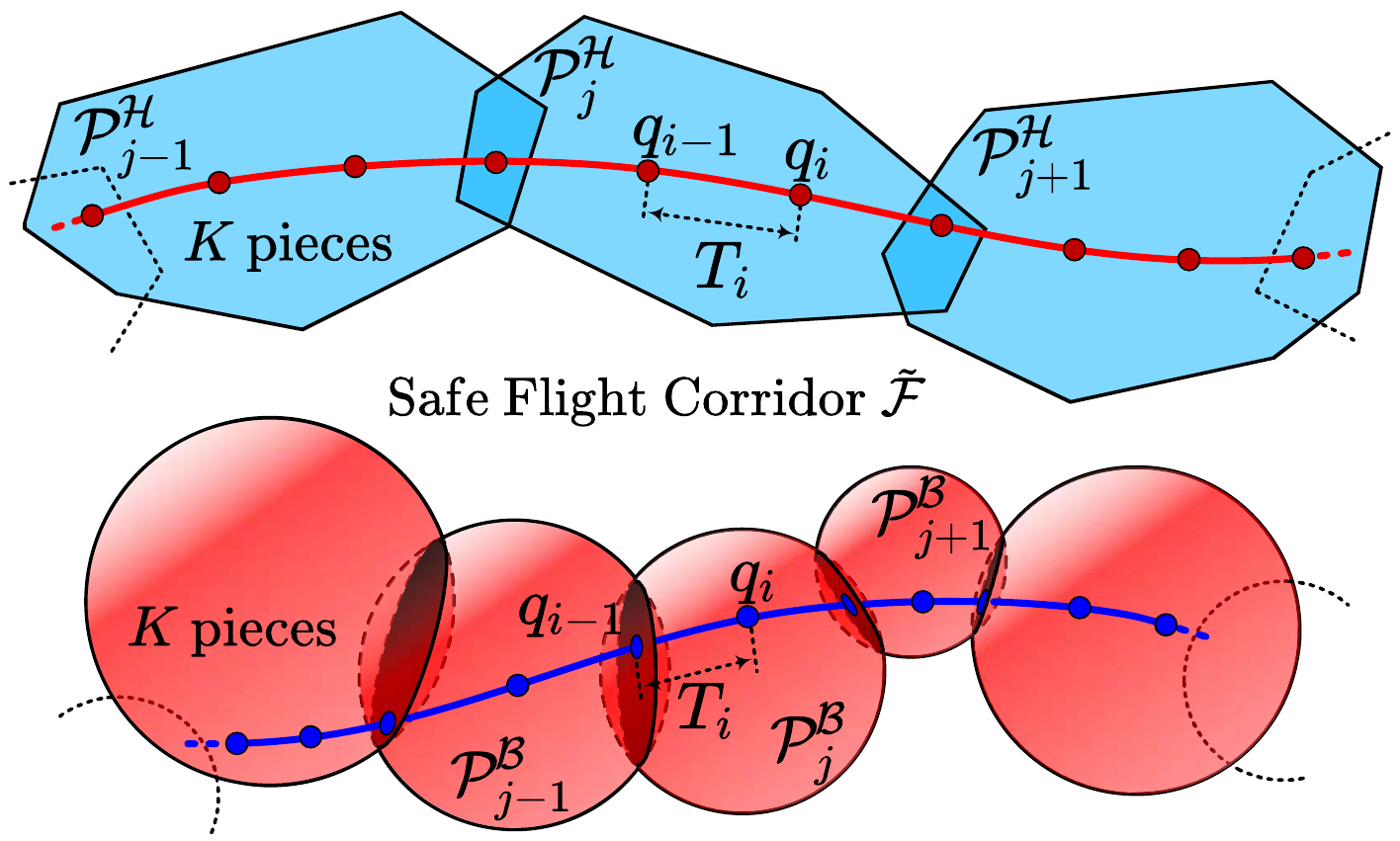}
    \caption{Piece assignment for a trajectory within different kinds of safe flight corridors in $\mathbb{R}^n$. Each geometrical primitive is assigned with $K$ pieces. An intermediate point $q_i$ is assigned to $\mathcal{P}_{\ceil{i/K}}\cap\mathcal{P}_{\ceil{(i+1)/K}}$. For ball-shaped corridors, a point is further anchored to an $(n-1)$-dimensional disk if it is assigned to the intersection of two $n$-dimensional balls.\label{fig:TrajSFCs}}
    \vspace{-0.0cm}
\end{figure}

\begin{figure}[ht]
    \centering
    \includegraphics[width=1.0\columnwidth]{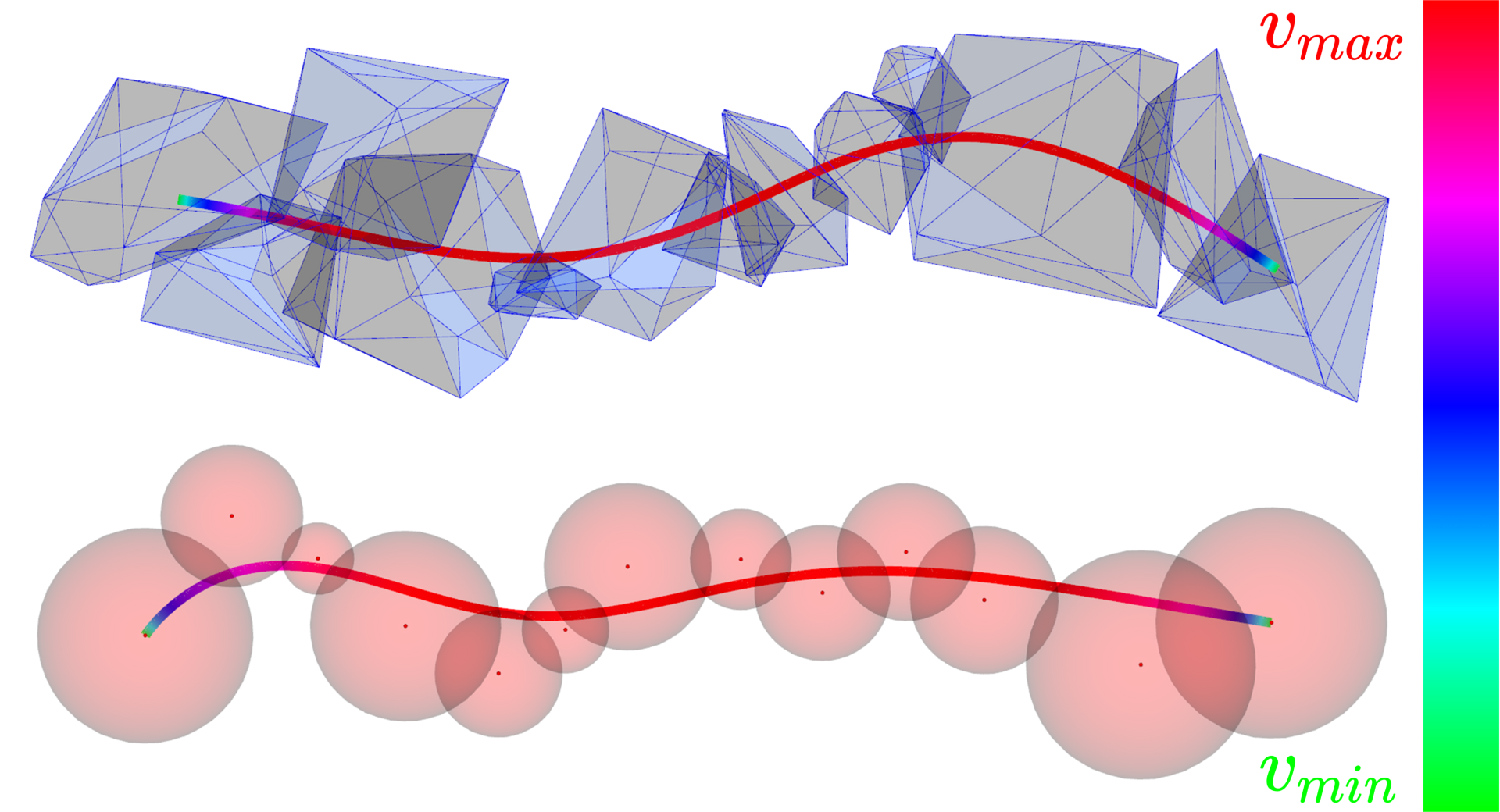}
    \caption{Optimized trajectories within different kinds of 3-D SFCs. The speed profile is colored according to its magnitude. The proposed method generates smooth trajectories within randomly generated SFCs. The speed persistently attains the maximum even if SFCs are narrow and twisted.\label{fig:TrajectoryInSFC}}
    \vspace{-0.0cm}
\end{figure}

As a special case of problem (\ref{eq:TrajectoryOptimization}), trajectory generation within 3-D Safe Flight Corridors (SFCs) has been widely adopted in real-world applications such as~\cite{Gao2020TeachRepeatReplanAC},~\cite{Honig2018TrajectorySWARM}, and~\cite{Gao2019FlyingPC}. The SFCs are usually generated by the \textit{front end} of a trajectory planning framework as an abstraction of the concerned configuration space, such as the Parallel Convex Cluster Inflation (PCCI)~\cite{Gao2020TeachRepeatReplanAC}, the Regional Inflation by Line Search (RILS)~\cite{Liu2017PlanningDF}, the Safe‐Region RRT* Expansion~\cite{Gao2019FlyingPC}, or the Iterative Regional Inflation by Semidefinite programming (IRIS)~\cite{Deits2015ComputingIRIS}. We assume that an SFC, either polyhedron-shaped or ball-shaped, is already obtained here as in (\ref{eq:FreeSpaceConvexDecomposition}) and (\ref{eq:LocallySequentialConnection}). Optimizing dynamically feasible trajectories within SFCs is usually taken as a \textit{back end} of such kind of frameworks.

\begin{figure}[t]
    \centering
    \includegraphics[width=1.0\columnwidth]{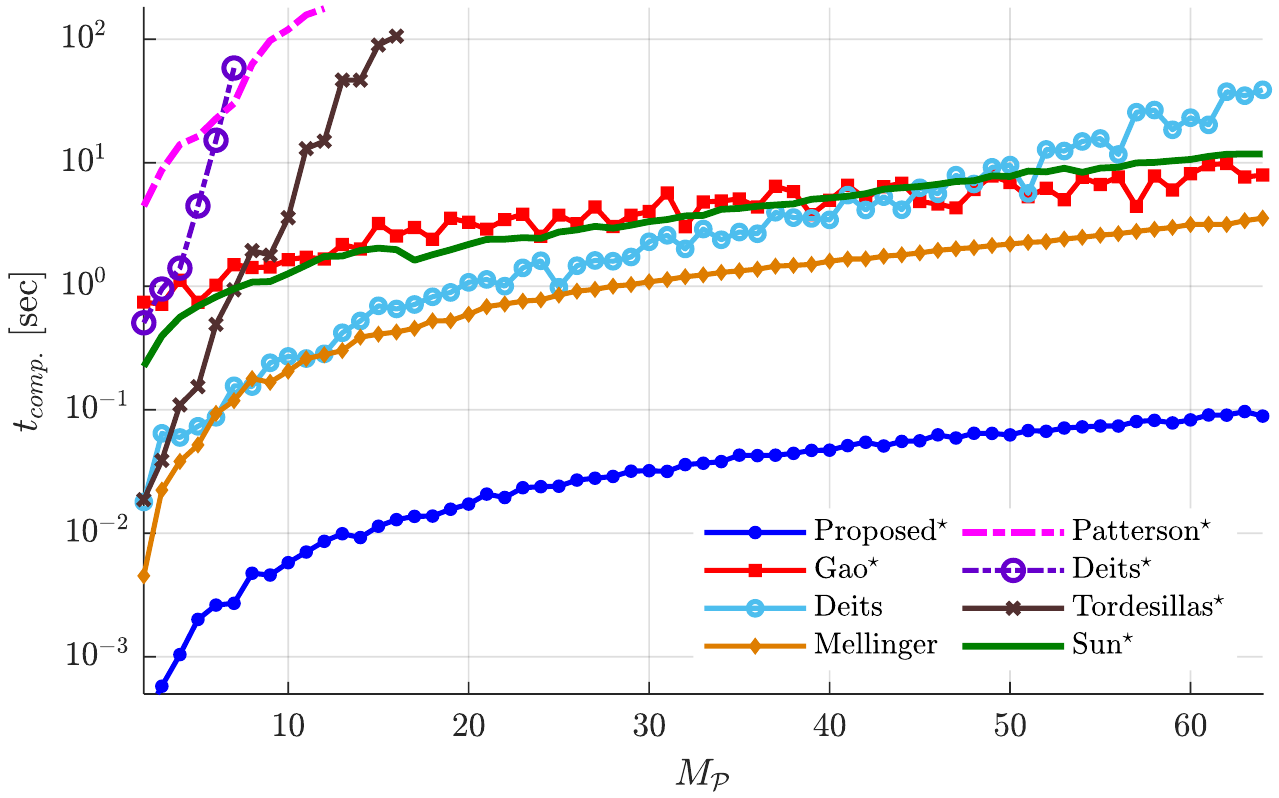}
    \caption{Benchmark on computation efficiency. The Proposed$^\star$ one outperforms other methods by orders of magnitudes. Methods from Tordesillas$^\star$ and Deits$^\star$ suffer from combinatorial explosion, but they are faster than Patterson$^\star$ on small-scale problems. Methods not supporting time or interval optimization consume less computation time at the sacrifice of quality. \label{fig:PolytopesSFCOptBenchmark}}
    \vspace{-0.0cm}
\end{figure}

\begin{figure}[t]
    \centering
    \includegraphics[width=0.95\columnwidth]{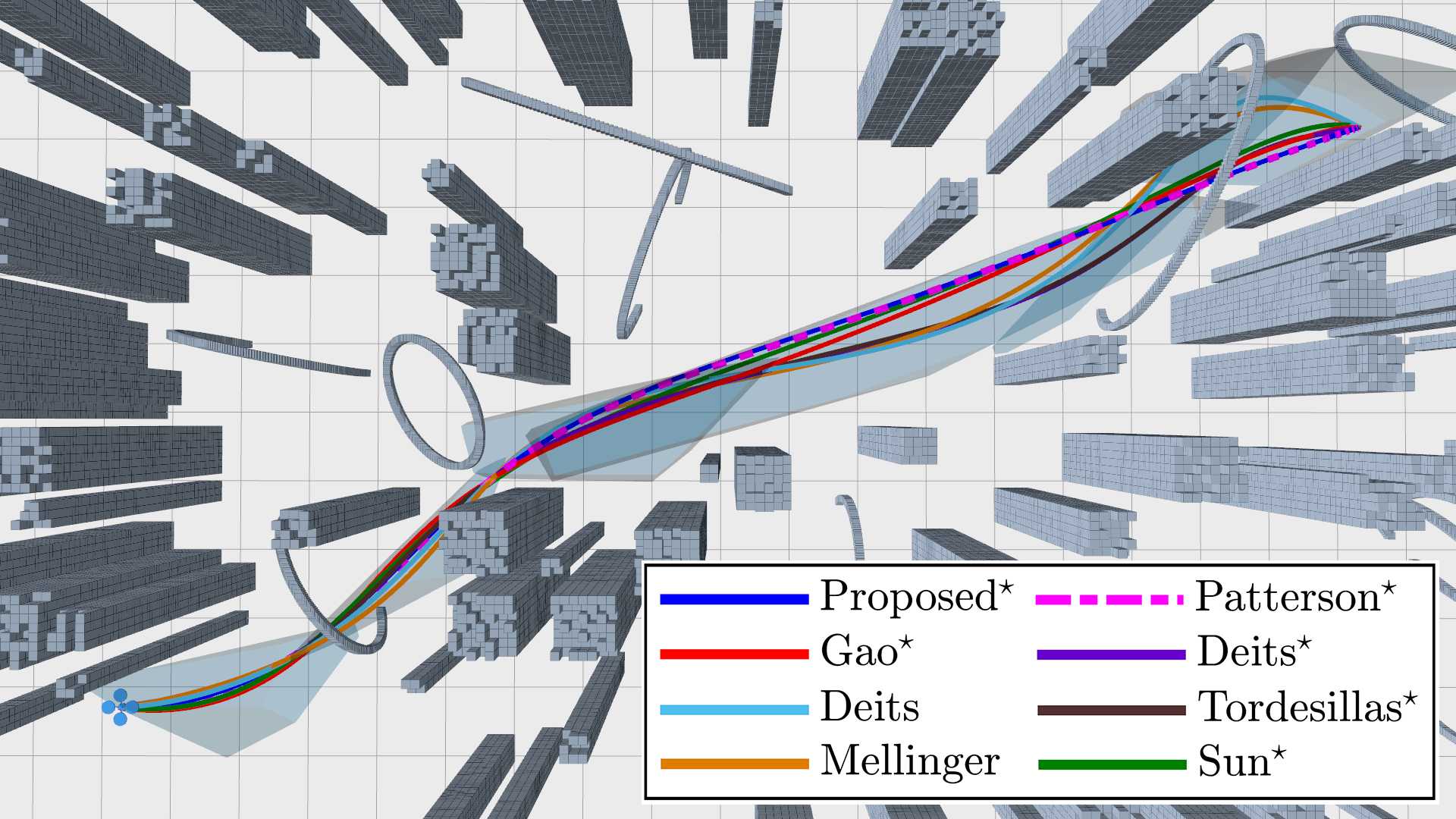}
    \caption{Geometrical profiles of trajectories generated by different methods in a random environment. The trajectory from the Proposed$^\star$ one is closer to the ground truth from Patterson$^\star$ than all other specialized ones.\label{fig:sfcTrajOptBenchmarkViz}}
        \vspace{-0.2cm}
\end{figure}

As is illustrated in Fig.~\ref{fig:TrajSFCs}, we consider two kinds of SFCs. Each convex primitive is assigned with $K$ trajectory pieces, thus $M=M_\mathcal{P}K$. The $i$-th trajectory piece $p_i(t):[0,T_i]\mapsto\mathbb{R}^3$ is assigned to $\mathcal{P}_{\ceil{i/K}}$. The intermediate point assignment of $\mathfrak{T}_{\mathrm{MINCO}}$ is also determined. Applying the constraint elimination, direct constraints on $\mathbf{T}$ and $\mathbf{q}$ are automatically satisfied, such as $\mathbf{T}\in\mathbb{R}_{>0}$ for $\rho_s(T)=k_\rho T$, $\norm{\mathbf{T}}_1<T_\Sigma$ for $\rho_f$, as well as $q_i\in\mathcal{P}_{\ceil{i/K}}\cap\mathcal{P}_{\ceil{(i+1)/K}}$ for all $i$. Constraints $\mathcal{G}$ are specified as follows to ensure both safety and dynamic limits:
\begin{equation}
\label{eq:SimpleContinuousConstraints}
\begin{cases}
p_i(t)\in\mathcal{P}_{\ceil{i/K}}, &\forall t\in[0,T_i],~\forall 1\leq i\leq M,\\
\norm{p_i^{(1)}(t)}^2\leq v_{max}^2, &\forall t\in[0,T_i],~\forall 1\leq i\leq M,\\
\norm{p_i^{(2)}(t)}^2\leq a_{max}^2, &\forall t\in[0,T_i],~\forall 1\leq i\leq M,\\
\end{cases}
\end{equation}
where $v_{max}$ and $a_{max}$ are dynamic limits. Then, the trajectory generation in $\tilde{\mathcal{F}}$ can be accomplished by solving the unconstrained NLP in (\ref{eq:Relaxation}). We show some optimization results in Fig.~\ref{fig:TrajectoryInSFC} for randomly generated SFCs. Both the polyhedron-shaped and ball-shaped SFCs are handled.

To further evaluate the performance of our method, we benchmark several existing methods over polyhedron-shaped SFCs. Technical details for all methods are listed as here:

\begin{itemize}
    \item Proposed$^\star$: Jerk energy minimization is conducted with either linear time regularization or fixed total time. Constraints in (\ref{eq:SimpleContinuousConstraints}) are enforced.
    \item Patterson$^\star$~\cite{Patterson2014GPOPS2}: The LQMT problem of a jerk-controlled system is solved using Gauss pseudospectral method. Each trajectory phase is confined within one polytope. Dynamic limits are enforced through path constraints.
    \item Gao$^\star$~\cite{Gao2020TeachRepeatReplanAC}: A geometrical curve is optimized via QP formed by jerk energy cost and linear safety constraints on control points of B\'ezier curves. Its temporal profile is then optimized by an SOCP for TOPP under (\ref{eq:SimpleContinuousConstraints}).
    \item Deits$^\star$~\cite{Deits2015EfficientMISOS}: The jerk energy and interval allocation of a trajectory is optimized by an MISOCP. Safety constraints and dynamic limits on $L_1$-norm of trajectory derivatives are exactly enforced through SOS conditions. Each trajectory piece is a $3$-degree polynomial.
    \item Deits: Details are the same as Deits$^\star$ except that intervals are allocated heuristically. No integer variable exists.
    \item Tordesillas$^\star$~\cite{Tordesillas2019Faster}: Details are the same as Deits$^\star$ except that safety is ensured by linear constraints on control points of B\'ezier curves. An MIQP is solved instead. The total time is determined by a well-designed algorithm.
    \item Mellinger~\cite{Mellinger2011MinimumST}: A trajectory is optimized in a QP formed by quadratic cost on jerk and linear safety constraints on sampled points. Its time allocation is generated with trapezoidal velocity profiles. Dynamic limits in (\ref{eq:SimpleContinuousConstraints}) are enforced by time scaling \cite{Liu2017PlanningDF}.
    \item Sun$^\star$~\cite{Sun2020BilevelTO}: A trajectory is optimized in a bilevel framework. The low-level QP is exact the same as Tordesillas$^\star$ except that $6$-degree polynomials are used. Its time allocation is optimized in the upper level optimization using analytical sensitivity of the low-level one.
\end{itemize}

\begin{figure}[t]
    \centering
    \includegraphics[width=1.0\columnwidth]{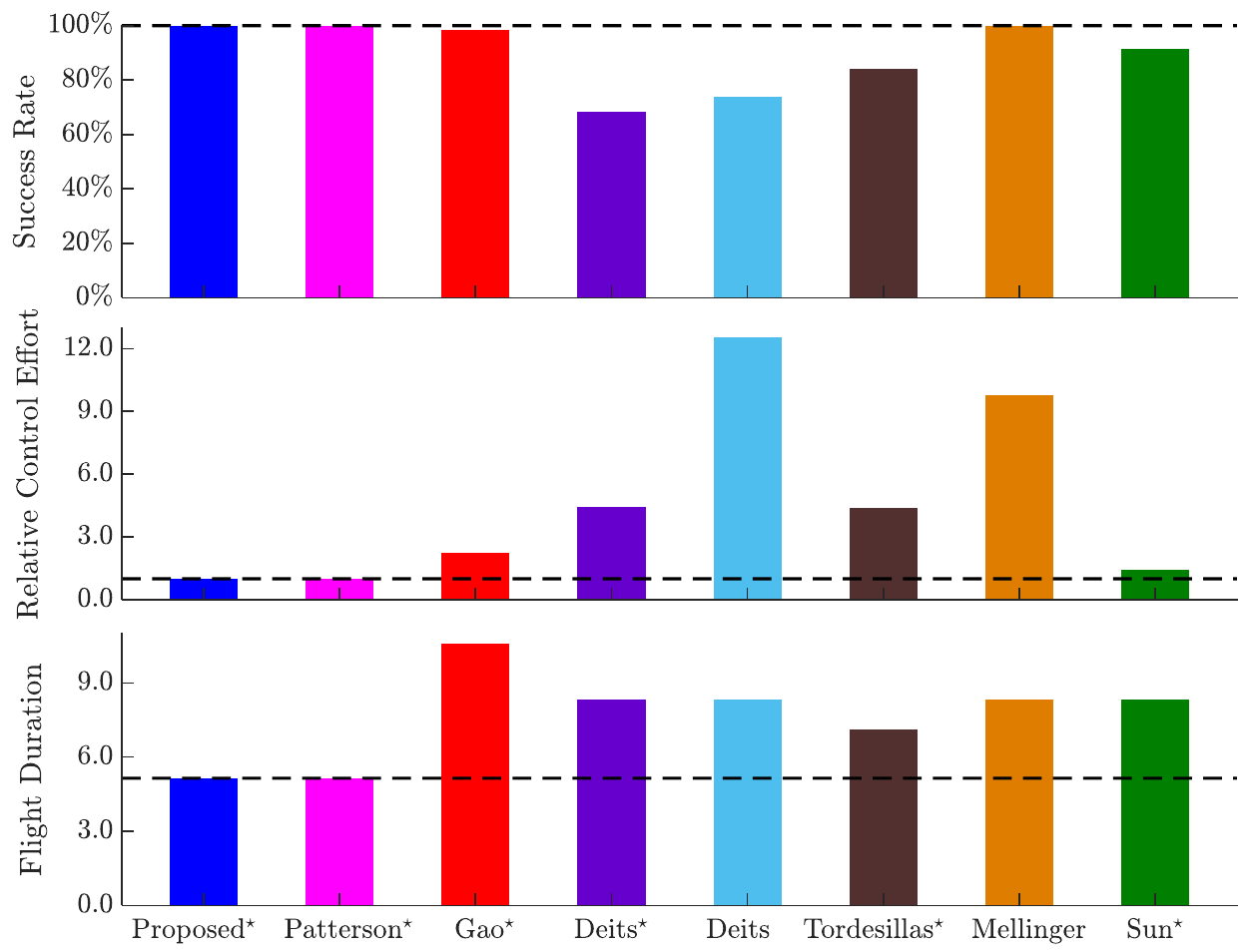}
    \caption{Benchmark on success rates, relative control effort, and flight durations. Methods from Deits$^\star$ and Tordesillas$^\star$ have relatively low success rates because they optimize interval allocation which involves integer variables. Methods from Deits and Mellinger have relatively large control effort because optimization on time or interval allocation is not supported. Note that some methods need preassigned total flight duration. \label{fig:sfcOptOverviewPlot}}
    \vspace{-0.2cm}
\end{figure}

\begin{figure}[ht]
    \begin{center}
        \subfigure[\label{fig:LongRangeFlightViz} Trajectories from different methods within a long SFC of the office.]
        {\includegraphics[width=1.0\columnwidth]{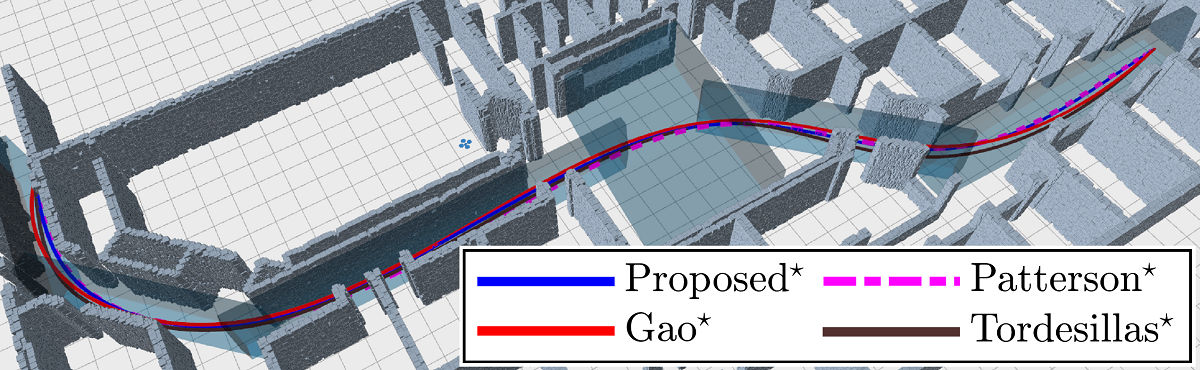}}
        \subfigure[\label{fig:ConstrainedVelAccNorm} The velocity and acceleration magnitude for different methods.]
        {\includegraphics[width=1.0\columnwidth]{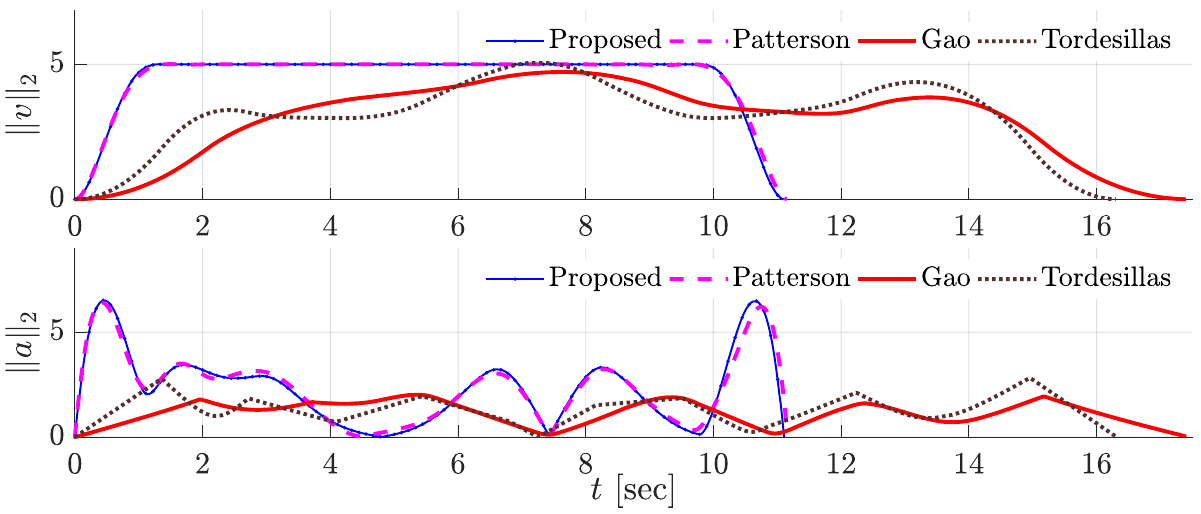}}
    \end{center}
    \caption{\label{fig:LongRangeFlightProfile} Trajectory profiles with large weight on time regularization. Only the Proposed$^\star$ one and the ground truth from Patterson$^\star$ generate persistently tight trajectories, considering the continuous-time constraints on norms of derivatives.}
    \vspace{-0.2cm}
\end{figure}

A method is asterisked if it supports optimizing time allocation or interval allocation. Dynamic limits are treated as the same for either $L_1$-norm or $L_2$-norm. Thus, constraints are indeed much tighter on methods from the Proposed$^\star$, Gao$^\star$, Mellinger, and Patterson$^\star$, that restrict $L_2$-norm of derivatives. As for total time, Deits$^\star$ and Sun$^\star$ need preassigned values, thus, we set their total time using trapezoidal velocity profiles. Patterson$^*$ handles the original problem directly, taking advantage of the exponential convergence of global collocation~\cite{Patterson2014GPOPS2}. Therefore, we take its trajectory as the ground truth.

The benchmark is conducted in randomly generated environments, one of which is shown in Fig.~\ref{fig:sfcTrajOptBenchmarkViz}. The corridor size $M_\mathcal{P}$ ranges from $2$ to $64$ where $10$ SFCs are generated for each size. The facet number of $\mathcal{P}^\mathcal{H}_i$ ranges from $8$ to $30$. We set $K=1$, $k_\rho = 1024.0$, $v_{max}=5.0m/s$, $a_{max}=7.0m/s^2$, $\kappa_i = 16$, the timeout as $3$ minutes, and the relative tolerance as $10^{-4}$. Static boundary conditions are assumed. As for programs, methods from the Proposed$^\star$ and Mellinger are both implemented in C++11 with a single thread for sequential computing. The general-purpose solver~\cite{Patterson2014GPOPS2} is directly adopted for Patterson$^\star$. A C++11 re-implementation of the original MATLAB one~\cite{Deits2015EfficientMISOS} is adopted for both Deits$^\star$ and Deits. Methods from Gao$^\star$, Tordesillas$^\star$, and Sun$^\star$ are taken from their open-source implementations. Besides, the commercial solver Gurobi~\cite{GurobiOPT} is used by Deits$^\star$, Deits, and Tordesillas$^\star$ with $6$ threads enabled for parallel computing. The commercial solver MOSEK~\cite{MosekOPT} is used by both Gao$^\star$ and Sun$^\star$.

The computation efficiency is provided in Fig.~\ref{fig:PolytopesSFCOptBenchmark}. Clearly, Deits$^\star$ and Tordesillas$^\star$ have to optimize integer variables, thus possessing approximately exponential complexity as $M_\mathcal{P}$ grows. Nonetheless, Tordesillas$^\star$ achieves acceptable performance for small $M_\mathcal{P}$ by using a more conservative but easier constraints than Deits$^\star$. Methods from Deits and Mellinger achieve satisfactory performance by tackling time allocation or interval allocation heuristically. Methods from Gao$^\star$ and Sun$^\star$ performs well in their scalability while the overhead for small $M_\mathcal{P}$ does not suit real-time applications. The method from Patterson$^\star$ suits offline scenarios where computation time is far less important than solution quality. The Proposed$^\star$ method improves the speed by more than an order of magnitude, while retaining optimization on time allocation.

The geometrical profile of trajectories is provided in Fig.~\ref{fig:sfcTrajOptBenchmarkViz}. Methods that do not optimize time or interval allocation are more likely to deviate from the ground truth. Trajectories by Deits$^\star$ and Tordesillas$^\star$ also deviate a lot from the ground truth because of the limited resolution of intervals. The success rates, relative control effort, and flight durations are all given in Fig.~\ref{fig:sfcOptOverviewPlot}. Interval allocation based methods have relatively low success rates. All control effort are normalized by that of the Proposed$^\star$ one, whose total time is fixed accordingly for fairness. Clearly, heuristic time or interval allocation causes relatively high control effort. Besides, the flight duration from the Proposed$^\star$ method is the closest to the ground truth.

To explore the temporal profile, we also test four complete methods in a long-distance flight as shown in Fig.~\ref{fig:LongRangeFlightProfile}. The trajectory from Gao$^\star$ is less aggressive than the others. The trajectory from Tordesillas$^\star$ has discontinuous jerk since $3$-degree polynomials are used. The results from the Proposed$^\star$ one have nearly the same quality as the ground truth. Profiting from the effectiveness of the penalty functional, our method can also achieve the maximum speed persistently.

In simulations, our method achieves comparable trajectory quality to the collocation based method~\cite{Patterson2014GPOPS2} in both the geometrical and temporal profile, while having superior computational speed against all benchmarked ones.

\begin{figure}[t]
    \begin{center}
        \subfigure[\label{fig:HeavyFPV}Hardware settings of the vehicle.]
        {\includegraphics[width=0.49\columnwidth]{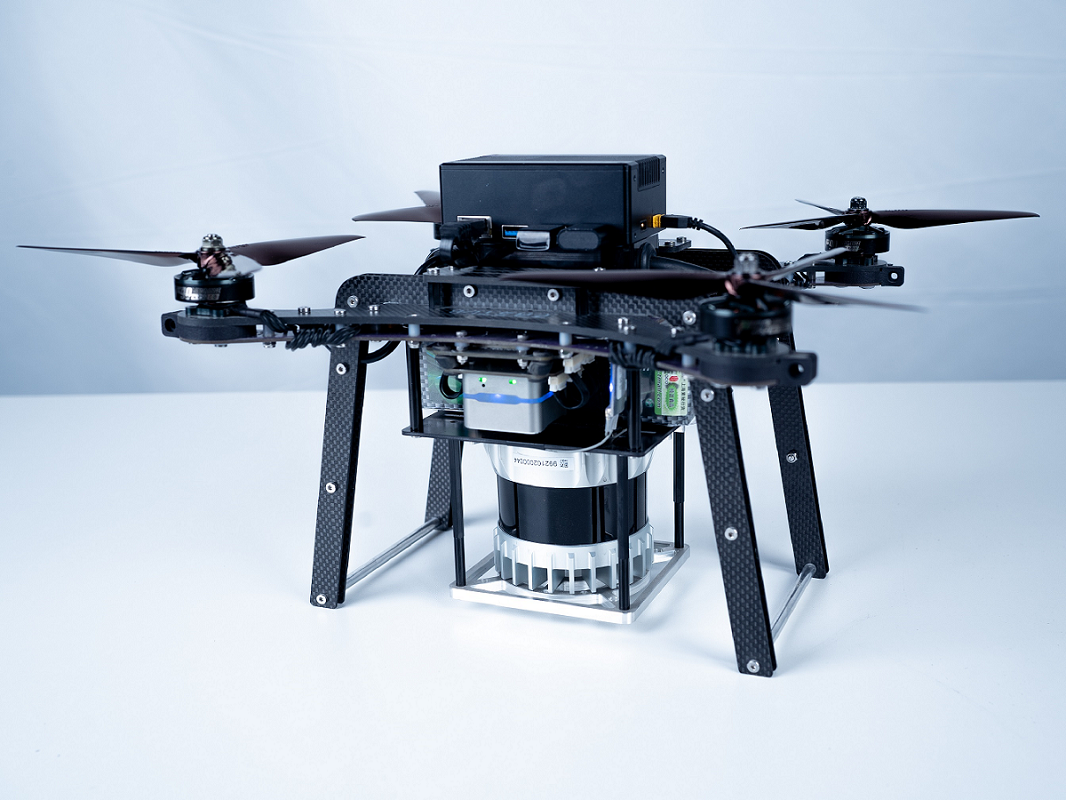}}%
        \hspace{0.02cm}
        \subfigure[\label{fig:HighSpeedFirstPersonView}An onboard camera image.]
        {\includegraphics[width=0.49\columnwidth]{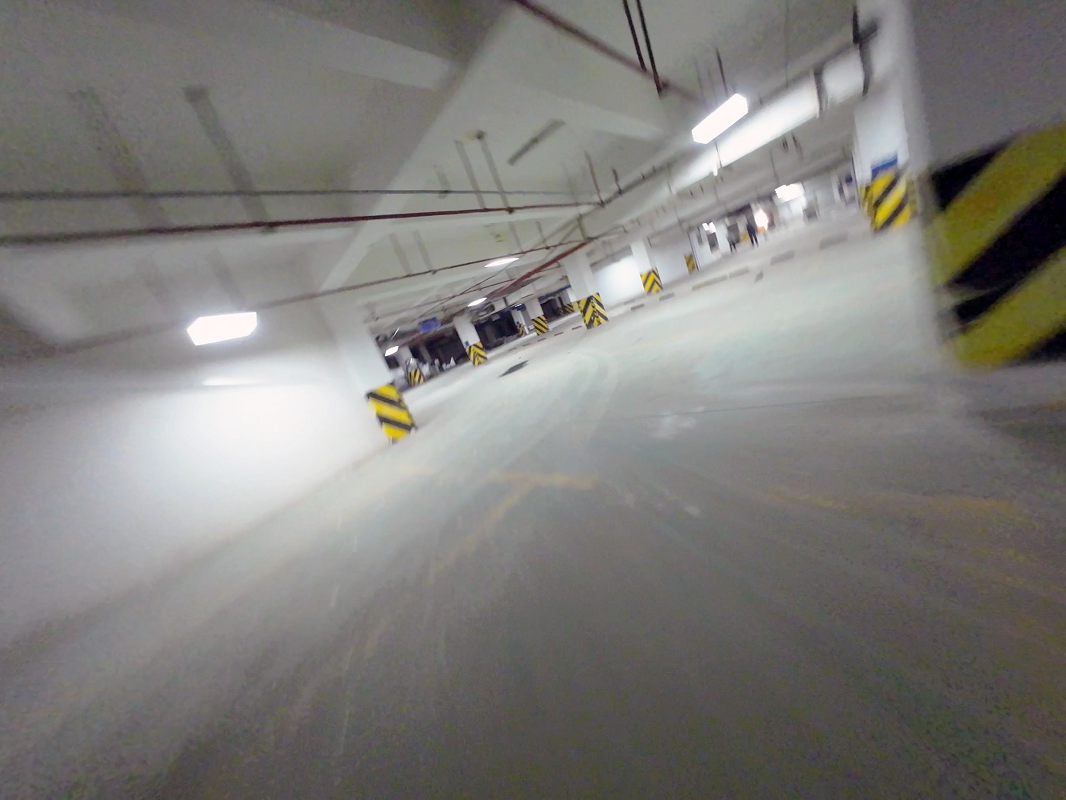}}
    \end{center}
    \caption{\label{fig:HighSpeedFlight}Left: Our autonomous multicopter equipped with an onboard computer and a LiDAR. Right: A snapshot of the first person view during our high-speed flight experiment in a garage.}
    \vspace{-0.2cm}
\end{figure}

We conduct experimental validation of our framework by enabling high-speed autonomous flights of a multicopter in an underground garage. All computations are performed by an onboard computer with an Intel Core i7-8550U CPU, which is shown in Fig.~\ref{fig:HighSpeedFlight}. We utilize FAST-LIO2~\cite{Xu2021FastLIO2} for highly robust LiDAR-based localization. Polyhedron-shaped safe flight corridors are generated by following~\cite{Gao2020TeachRepeatReplanAC}. Our method generates a $343.57m$ global trajectory in only $0.29s$ in the first track. The planning results are provided in Fig.~\ref{fig:HighSpeedGlobalView}. We believe this computation time validates our framework's efficiency even for long-distance trajectory planning. In this experiment, the vehicle speed reaches $12.0m/s$ while ensuring its safety among obstacles and keeping a low thrust-to-weight ratio. We further compare planning results for different parameters on $v_{max}$. It turns out that our method can always squeeze the capability of $v_{max}$ and $a_{max}$ if $k_\rho$ is large. More details about this experiment are given in the attached multimedia.

\subsection{\texorpdfstring{$\mathrm{SE}(3)$}{SE(3)} Motion Planning in Quotient Space}

\begin{figure*}[ht]
    \begin{center}
        \subfigure[\label{fig:LayoutForPolytopes}SFC layout]
        {\includegraphics[width=0.30\columnwidth]{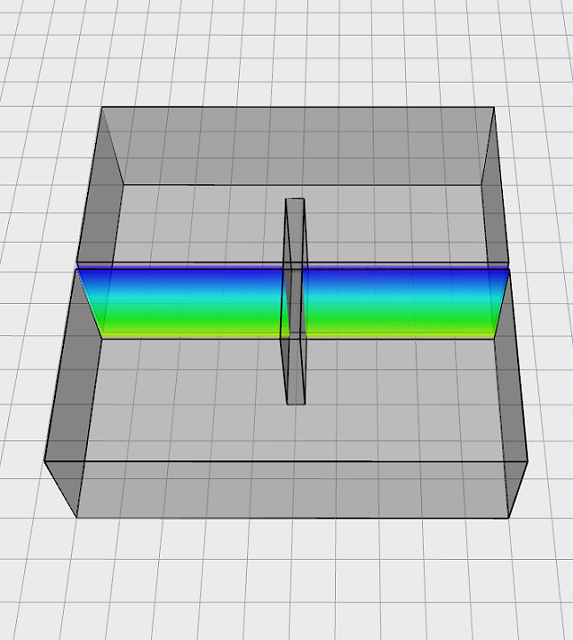}}%
        \hspace{0.01cm}
        \subfigure[\label{fig:NarrowGap30Deg}$\phi_{gap}=30^\circ$]
        {\includegraphics[width=0.30\columnwidth]{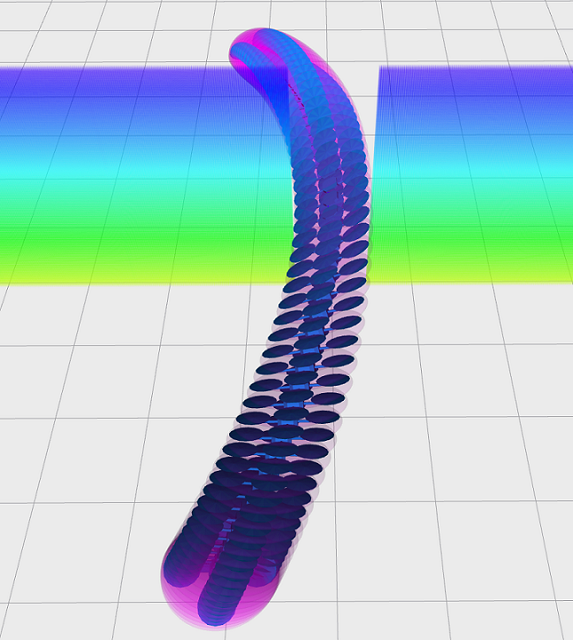}}%
        \hspace{0.01cm}
        \subfigure[$\phi_{gap}=45^\circ$]
        {\includegraphics[width=0.30\columnwidth]{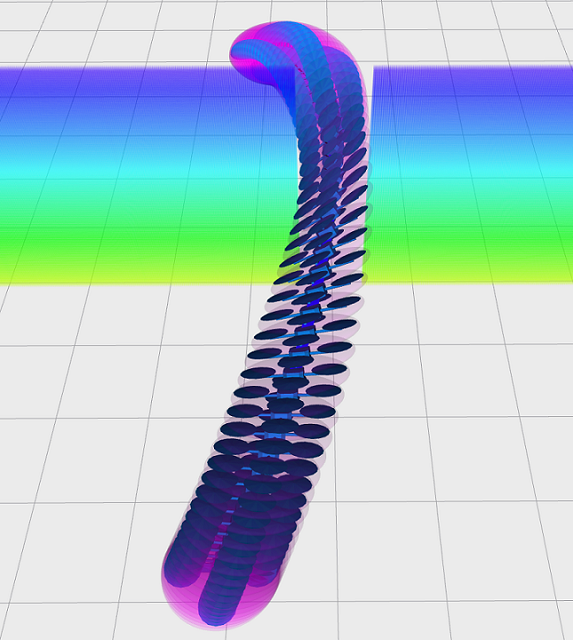}}%
        \hspace{0.01cm}
        \subfigure[$\phi_{gap}=60^\circ$]
        {\includegraphics[width=0.30\columnwidth]{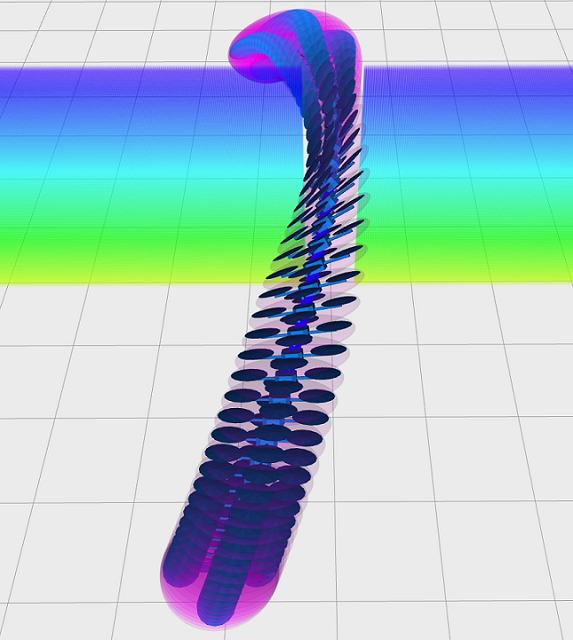}}%
        \hspace{0.01cm}
        \subfigure[$\phi_{gap}=75^\circ$]
        {\includegraphics[width=0.30\columnwidth]{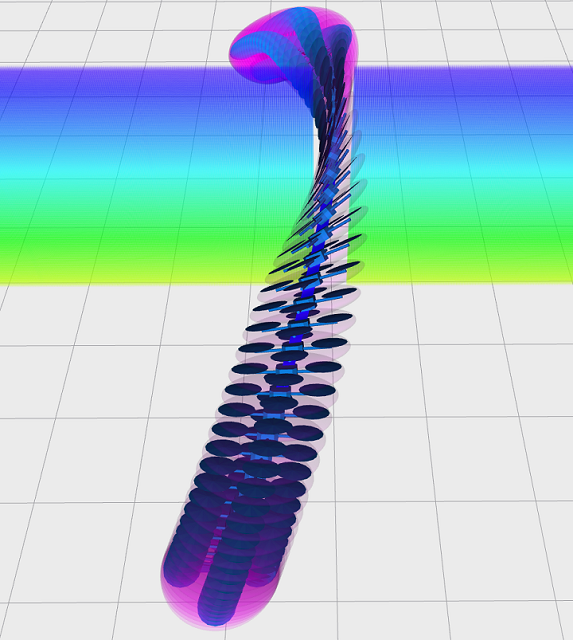}}%
        \hspace{0.01cm}
        \subfigure[\label{fig:NarrowGap85Deg}$\phi_{gap}=85^\circ$]
        {\includegraphics[width=0.30\columnwidth]{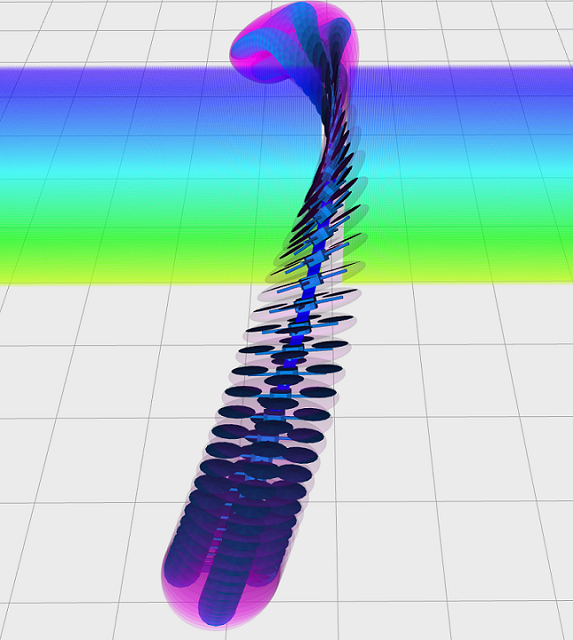}}
        \vspace{0.0cm}
        \subfigure[\label{fig:SE3TrajectoryConstraintProfile}The magnitude of angular velocity and the normalized thrust for different $\mathrm{SE}(3)$ trajectories.]
        {\includegraphics[width=1.95\columnwidth]{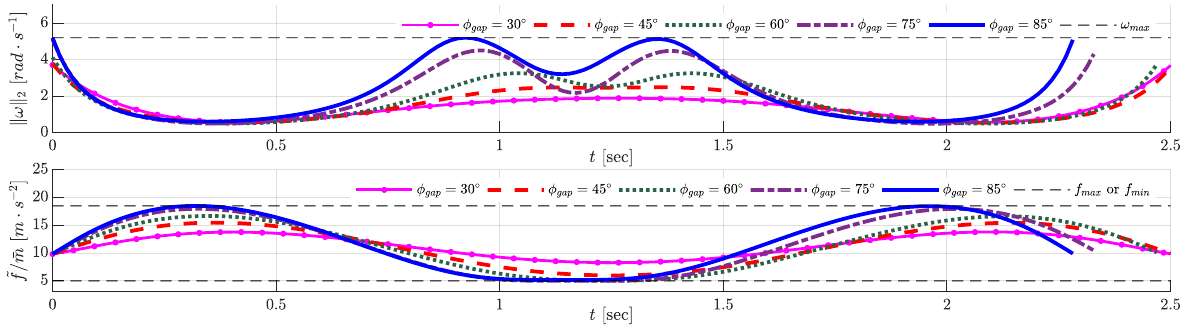}}
    \end{center}
    \caption{\label{fig:SE3PlanningForNarrowGaps} SFC layout for a narrow gap, $\mathrm{SE}(3)$ trajectories under different widths of gaps, and control inputs for different motions. As the gap becomes narrower, larger angular rates and higher thrust are needed for a safe flight. The proposed method persistently enforces limits on these control inputs under different settings while retains millisecond-level computation time.}
\end{figure*}

In dense obstacle environments, safe motions often do not exist for narrow spaces unless a multicopter agilely adjusts its attitude to avoid collisions. Therefore, we consider $\mathrm{SE}(3)$ motion planning in our framework. An important property for planning in $\mathrm{SE}(3)$ as a manifold with structure $\mathbb{R}^3\times\mathrm{SO}(3)$ is the necessary condition that a feasible pose for a rigid body at least contains a feasible translation for a dimensionless point. The subspace $\mathbb{R}^3$ is referred to as a \textit{Quotient Space}~\cite{Orthey2019REQST}.
Exploiting such a quotient-space decomposition~\cite{Orthey2018QSMP}, we consider the rotational safety based on a translational trajectory, instead of handling them jointly. Therefore, we can relax assumptions for (\ref{eq:FreeSpaceConvexDecomposition}) such that $\tilde{\mathcal{F}}$ is just a free region in the quotient space without considering multicopter's actual size.

We consider simplified quadcopter dynamics whose configuration is defined by its translation $p$ and rotation $\mathbf{R}$:
\begin{equation}
    \label{eq:SimplifiedDynamics}
    \left\{
    \begin{aligned}
    \dot{p}&=v,\\
    \bar{m}\dot{v}&=-\bar{m}\bar{g}e_3+\mathbf{R}\tilde{f}e_3,\\
    \dot{\mathbf{R}}&=\mathbf{R}\hat{\omega}.
    \end{aligned}
    \right.
\end{equation}
where $e_i$ is the $i$-th column of $\mathbf{I}_{3}$, $\bar{g}$ the gravitational acceleration, $\tilde{f}$ the thrust, $\omega$ the body rate input, and $\bar{m}$ the vehicle mass. The hat map $\hat{\cdot}:\mathbb{R}^3\mapsto\mathbb{R}^{3\times3}$ is defined by $\hat{a}b=a\times b$ for all $a,b\in\mathbb{R}^3$. Moreover, we model the geometrical shape of a symmetric multicopter as its outer L\"owner-John ellipsoid~\cite{Toth2017HandbookDCG},
\begin{equation}
\label{eq:EllipsoidTrajectory}
\mathcal{E}(t)=\cBrac{\mathbf{R}(t)\mathbf{Q}x+p(t) ~\Big|~\norm{x}_2\leq 1}
\end{equation}
where $\mathbf{Q}=\Diag\cbrac{r_e,r_e,h_e}$. $r_e$ and $h_e$ are the radius and the height of multicopter, respectively.

A feasible motion satisfies the safety and dynamic limits. By safety we mean $\mathcal{E}(t)\subset\tilde{\mathcal{F}},~\forall t\in[0,T]$, where $T$ is the total time of the motion. However, this safety constraint is indeed hard to enforce. We further make an assumption on $\mathcal{F}$ that all $\mathcal{P}_i^\mathcal{H}$ or their intersections are able to contain at least one ellipsoid of the multicopter. This assumption can be reasonably satisfied when $\tilde{\mathcal{F}}$ is generated incrementally. As a result, we can ensure safety through
\begin{equation}
\label{eq:RelaxedSafetyConstraintSE3}
\forall t\in[0,T],~\exists1\leq i\leq M_p,~\mathit{s.t.}~\mathcal{E}(t)\subset\mathcal{P}^\mathcal{H}_i.
\end{equation}
By dynamic limits we mean the velocity, thrust and body rate should have reasonable magnitude,
\begin{equation}
\label{eq:DynamicLimitsSE3}
\begin{cases}
\norm{p^{(1)}(t)}^2\leq v_{max}^2, &\forall t\in[0,T],\\
f_{min}\leq\tilde{f}(t)\leq f_{max}, &\forall t\in[0,T],\\
\norm{\omega(t)}_2^2\leq \omega_{max}^2, &\forall t\in[0,T].
\end{cases}
\end{equation}

Given a quotient-space trajectory $p(t):[0,T]\mapsto\mathbb{R}^3$, state-control trajectories of $p$, $v$, $\mathbf{R}$, and $\omega$ are all algebraically computed by flatness maps $\Psi_x$ and $\Psi_u$ of the dynamics (\ref{eq:SimplifiedDynamics}). The concrete forms of the algebraic maps are detailed in~\cite{Mellinger2011MinimumST} with fixes on the body rate~\cite{Faessler2018DiffFRD} for simple quadcopters thus are omitted here. Consequently, the entire $\mathrm{SE}(3)$ trajectory is also obtained. Denote by $\mathbf{R}(t)$ its rotational part. To generate $p(t)$ in $\tilde{\mathcal{F}}$, we follow the methodology of our previous experiment but with different constraints here.

The $i$-th trajectory piece $p_i(t):[0,T_i]\mapsto\mathbb{R}^3$ is assigned to the polytope $\mathcal{P}^\mathcal{H}_j$ defined in (\ref{eq:HPolytopeDescription}) with $j=\ceil{i/K}$. We denote by $\mathcal{E}_i(t)$ the ellipsoid induced by $p_i(t)$ and the corresponding $\mathbf{R}_i(t)$ as is defined in (\ref{eq:EllipsoidTrajectory}). As proposed by Wu et. al.~\cite{Wu2021ExternalFRSMP}, ensuring safety by confining the vehicle ellipsoid in a polyhedron also has an analytical form. Specifically,
\begin{equation}
\mathcal{E}_i(t)\in\mathcal{P}^\mathcal{H}_j,~j=\ceil{i/K},~\forall t\in[0,T_i],
\end{equation}
is equivalent to
\begin{subequations}
\label{eq:SafetyOnEllipsoid}
\begin{align}
\sBrac{\sbrac{\mathbf{A}_j\mathbf{R}_i(t)\mathbf{Q}}^2\mathbf{1}}^{\frac{1}{2}}+&\mathbf{A}_j p_i(t)-b_j\preceq\mathbf{0},\\
j=\ceil{i/K},&~\forall t\in[0,T_i],
\end{align}
\end{subequations}
where $\mathbf{1}$ is an all-ones vector with an appropriate length, $[\cdot]^2$ and $[\cdot]^{\frac{1}{2}}$ are entry-wise square and square root, respectively. Finally, we obtained the state-control constraint $\mathcal{G}_\mathcal{D}$ in (\ref{eq:OriginalStateControlConstraints}) for the considered dynamics in (\ref{eq:SimplifiedDynamics}). We choose to minimize $s=3$ because it is the highest derivative order for flatness of (\ref{eq:SimplifiedDynamics}) and also helpful in smoothing the angular rate.

We validate our framework in simulations where a relatively large quadcopter is required to fly through a narrow gap with much smaller width as shown in Fig.~\ref{fig:SE3PlanningForNarrowGaps}. The settings are $r_e=0.5m$, $h_e=0.1m$, $f_{min}/\bar{m}=5.0m/s^2$, $f_{max}/\bar{m}=18.5m/s^2$, $v_{max}=6.5m/s$, and $\omega_{max}=5.2rad/s$. Intuitively, the quadcopter can only achieve no more than $1$ revolution per second ($rps$), making it less agile than small quadcopters~\cite{Kushleyev2013TowardsSAMQ} that can achieve $5rps$. The computation times, required roll angles, and $\mathrm{SE}(3)$ motions for different $d_{gap}$ are shown in the Table~\ref{Tab:ComputationTimeForGaps} and Fig.~\ref{fig:NarrowGap30Deg}-\ref{fig:NarrowGap85Deg}.

\begin{table}[ht]
    \centering
    \caption{\label{Tab:ComputationTimeForGaps}Computation times and roll angles for different gaps}
    \begin{tabular}{c|c|c|c|c|c}
        \hline
        $d_{gap}$           & $0.88m$     & $0.76m$     & $0.60m$     & $0.40m$     & $0.25m$     \\ \hline
        $\phi_{gap}$ & $30^\circ$ & $45^\circ$ & $60^\circ$ & $75^\circ$ & $85^\circ$ \\ \hline
        $t_{comp.}$        & $4.7ms$     & $4.4ms$     & $6.0ms$    & $6.6ms$    & $7.4ms$    \\ \hline
    \end{tabular}
    \vspace{0.0cm}
\end{table}

As the gap becomes narrower, the required roll angle becomes larger and the feasible space becomes smaller in view of dynamic limits. Our method is still able to find all the feasible motions. The superior computation speed makes it possible to solving $\mathrm{SE}(3)$ planning at a high frequency (at least $100Hz$). Constraint functions are visualized in Fig.~\ref{fig:SE3TrajectoryConstraintProfile}. The body rate and thrust satisfy dynamic limits all the time. The continuous-time tightness of $f_{min}$ for $\phi_{gap}\in\cbrac{60^\circ, 75^\circ, 85^\circ}$ shows the effectiveness of our penalty functional.

\begin{figure}[h]
    \begin{center}
        \subfigure[\label{fig:QuadrotorCfg}The custom-made quadcopter.]
        {\includegraphics[width=0.572\columnwidth]{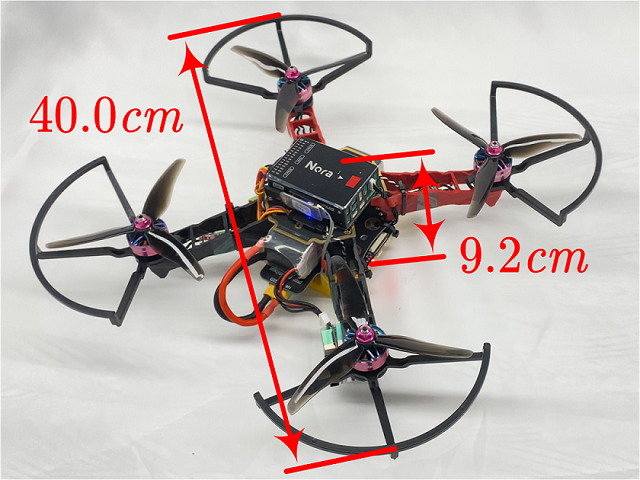}}%
        \hspace{-0.0cm}
        \subfigure[\label{fig:WindowCfg}The window.]
        {\includegraphics[width=0.24\columnwidth]{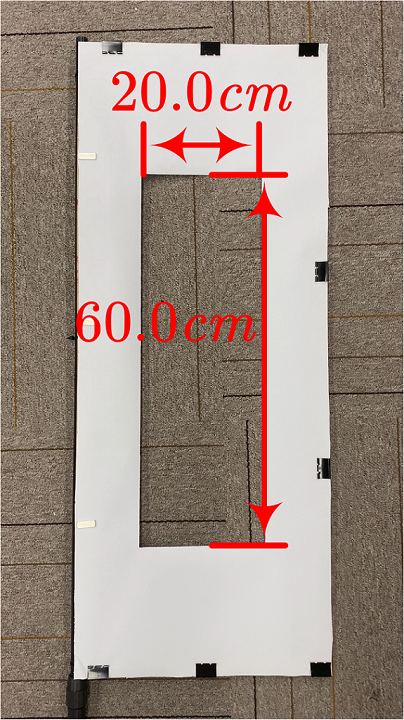}}
    \end{center}
    \caption{\label{fig:QuadrotorWindowCfg}Sizes of the quadcopter and the narrow window.}
\end{figure}

\begin{figure*}[ht]
    \begin{center}
        \subfigure[\label{fig:InteractiveScenario}The interactive scenario.]
        {\includegraphics[width=0.63\columnwidth]{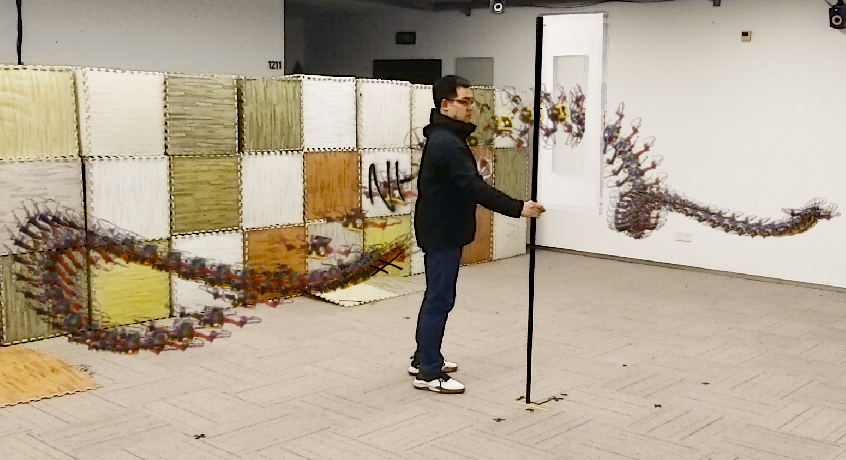}}%
        \hspace{0.01cm}
        \subfigure[\label{fig:DoubleSnapShotsMarked}Flying through two consecutive windows.]
        {\includegraphics[width=0.63\columnwidth]{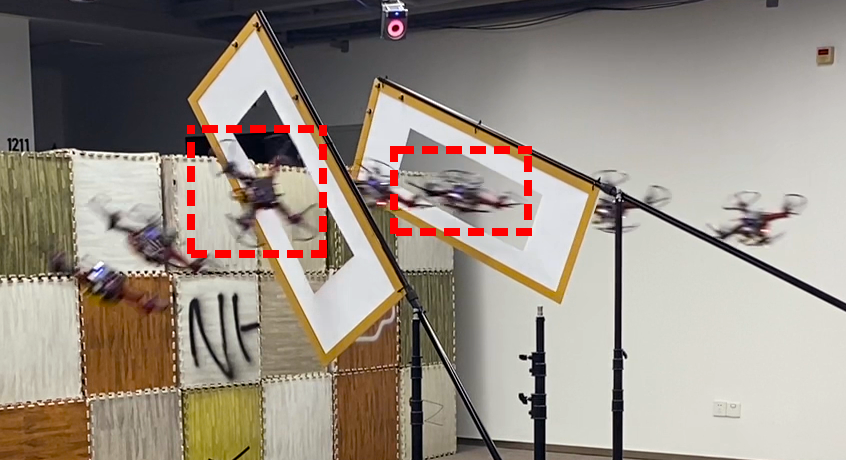}}%
        \hspace{0.01cm}
        \subfigure[\label{fig:TripleSnapShotsMarked}Flying through three consecutive windows.]
        {\includegraphics[width=0.63\columnwidth]{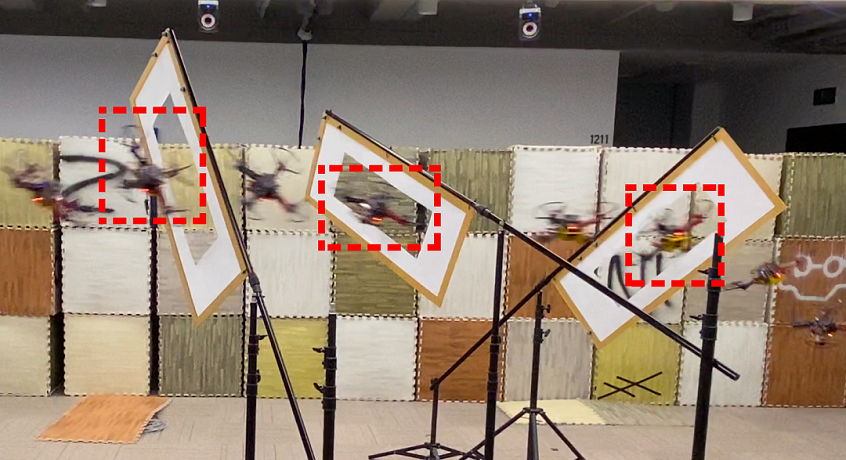}}
        \vspace{0.00cm}
        \subfigure[\label{fig:DoublePlanned} $\mathrm{SE}(3)$ planning for two windows.]
        {\includegraphics[width=0.95\columnwidth]{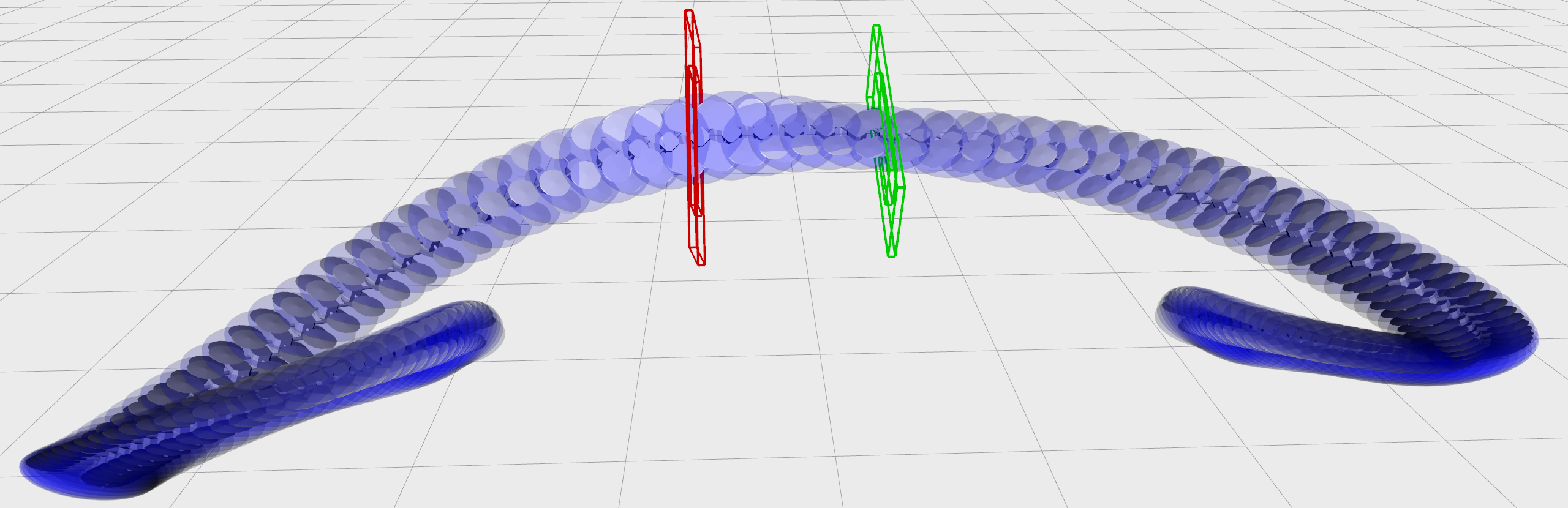}}%
        \hspace{0.01cm}
        \subfigure[\label{fig:TriplePlanned} $\mathrm{SE}(3)$ planning for three windows.]
        {\includegraphics[width=0.95\columnwidth]{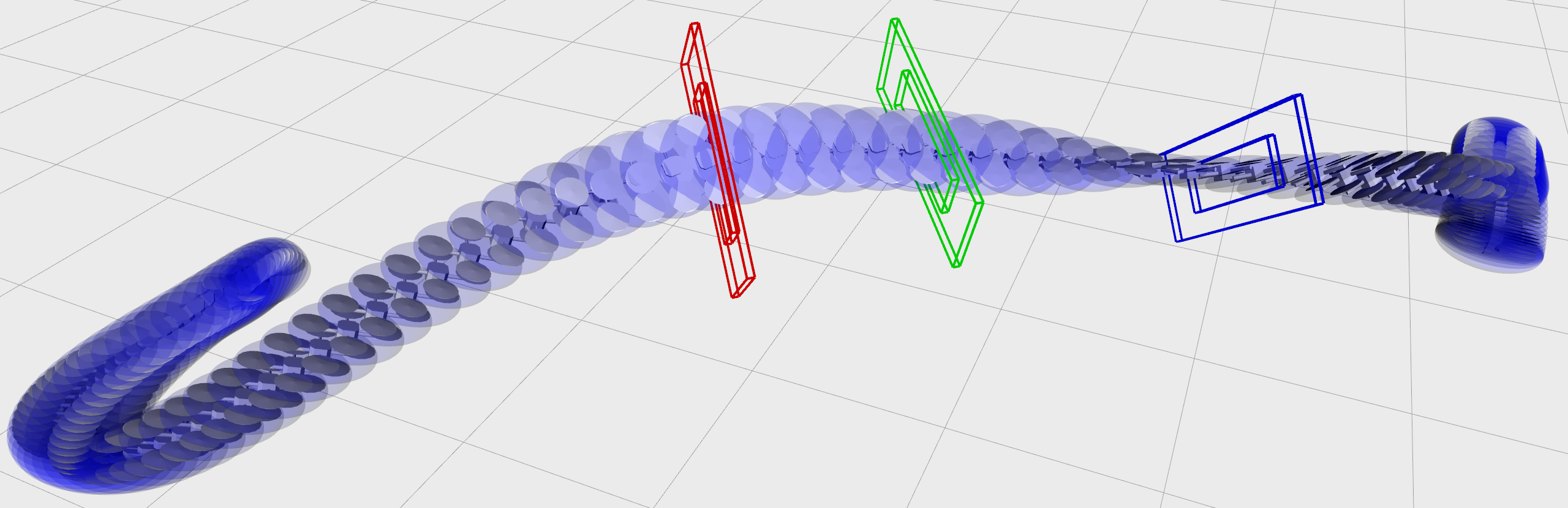}}
        \vspace{0.00cm}
    \end{center}
    \caption{\label{fig:FlyThroughMultipleWindows} Experiment results for three $\mathrm{SE}(3)$ planning scenarios. The Fig.~\ref{fig:InteractiveScenario} gives a snapshot for the interactive scenario. Fig.~\ref{fig:DoubleSnapShotsMarked}-\ref{fig:TripleSnapShotsMarked} show two snapshots for real flights through consecutive windows. Fig.~\ref{fig:DoublePlanned}-\ref{fig:TriplePlanned} show corresponding $\mathrm{SE}(3)$ trajectories generated by the proposed method.}
\end{figure*}

We evaluate the performance of our planner in a real-world experiment where a quadcopter flies through several narrow windows. Sizes of the quadcopter and windows are given in Fig.~\ref{fig:QuadrotorWindowCfg}. The quadcopter weights $794.2g$. The safety margin of the short side is only $5.4cm$, implying that the feasible motion space is extremely small. The settings are $r_e = 20.0cm$, $h_e = 4.6cm$, $v_{max}=4.0m/s$, $f_{min}/\bar{m} = 3.0m/s^2$, $f_{max}/\bar{m} = 18.0m/s^2$, $\omega_{max} = 6.0rad/s$, and $K=2$. The flying space is a restricted volume of $6.5\times6.0\times2.0m^3$. All poses of narrow windows and the quadcopter are provided by a motion capture system running at $100Hz$. The obstacle-free region $\tilde{\mathcal{F}}$ is geometrically computed for multiple narrow windows in the free volume. The planner is run on an offboard computer where a human operator arbitrarily chooses the goal position. We adopt the control algorithm by Faessler et al.~\cite{Faessler2018DiffFRD} for onboard $\mathrm{SE}(3)$ trajectory tracking.

The first scenario contains consecutive windows with roll angles ranging from $30^\circ$ to $90^\circ$. The quadcopter has to fly through them and reach a randomly selected goal as shown in Fig.~\ref{fig:DoubleSnapShotsMarked}-\ref{fig:TriplePlanned}. The second scenario is an interactive one where a human operator randomly holds a narrow window for real-time planning as given in Fig.~\ref{fig:InteractiveScenario}. The third scenario requires the quadcopter persistently fly back and forth through multiple windows for a long duration as shown in Fig.~\ref{fig:Consecutive} and Fig.~\ref{fig:ConsecutiveSE3Trajectory}. Our planner guides the quadcopter to fly back and forth through windows for about $20.0s$ while ensuring the safety and physical limits all the time. More details about this experiment are given in the attached multimedia.

In this experiment, the short distance between consecutive windows, the small acceleration/deceleration space, and the limited vehicle maneuverability are challenges that our planner must confront. We believe that these results constitute a strong evidence for its constraint fidelity, motion quality, computation efficiency, and robustness. However, we do observe the limitation of optimization-based methods. For example, if two $90^\circ$ windows are asymmetrically placed, a multicopter has to pass them in sequence. Each window only allows two roll angles $\pm90^\circ$. The combinations are $4$ locally optimal maneuvers but only one can be the global optimum. Thus, the other three are shallow local minima inevitable for local methods.

\section{Discussion and Conclusion}

\subsection{Extensions}
Profiting from the flexibility and efficiency, our framework has many applicative and algorithmic extensions. First, no assumption is ever made on concrete forms of vehicle dynamics and $\mathcal{G}_\mathcal{D}$. More accurate dynamics such as the rotor drag~\cite{Faessler2018DiffFRD} can be adopted to fully exploit physical limits via real-time high-fidelity planning and control. Time-dependent constraints for moving obstacles can also be supported by $\mathcal{G}_\mathcal{D}$. Second, our framework is inherently parallelizable to further squeeze its performance. Computation-demanding operations on $I_\mathcal{G}\sbrac{p}$ are independent at each timestamp, thus parallelization can effectively speedup our optimization. Moreover, it is possible to extend our methodology to other vehicle types whose flat-output space overlaps the configuration space.
An example is the fixed-wing aircraft in~\cite{Bry2015AggressiveFO} whose flights are mainly restricted by the trajectory curvature. Bry et al. propose Dubins–Polynomial trajectories~\cite{Bry2015AggressiveFO} for this restriction while the curvature constraint is a special case of $\mathcal{G}$ for MINCO.

To demonstrate the extendibility, we apply our framework to a swarm of multicopters to enable their autonomous navigation in unknown environments. All details of the formulation (\ref{eq:TrajectoryOptimization}) and real-world flights are given in a technical report~\cite{Zhou2021DecentralizedSWARM}.

\subsection{Limitations}
Our framework, like most optimization-based ones, focuses on local solutions of trajectory planning, thus suffering from shallow local minima. This can be alleviated by interleaving sampling-based or graph-search-based strategies into our framework, as proposed in~\cite{Zucker2013Chomp, Campos2017HybridOTP, Natarajan2021InterleavingGSTO}. A major limitation of the framework originates from MINCO itself. If $\mathcal{G}$ exist, optimal solutions cannot in general be represented by polynomial splines, let alone MINCO. Thus optimizing MINCO is just a relaxation to the original problem. However, our results show that MINCO can still represent high-quality solutions comparable to the ground truth, but with several orders of magnitudes faster computing. There are also limitations caused by the penalty functional. To achieve zero constraint violations, an unbounded smoothing factor or penalty weight and an unbounded quadrature resolution are both required. However, small constraint violations are empirically acceptable for multicopter navigation. As a reward, this method does not need initial feasible guesses.

\subsection{Conclusion}
In this article, we proposed a flexible multicopter trajectory planning framework powered by several core features, such as the MINCO trajectory based on our optimality conditions, constraint elimination schemes based on smooth maps, the penalty functional method based on constraint transcription, and the backward differentiation of the flatness maps from flat outputs.
All these components enjoy the efficiency and generality originating from low complexity and less preliminary assumptions. We performed extensive benchmarks against many kinds of multicopter trajectory planning methods to show the speedup over orders of magnitude and the top-level solution quality.
A variety of applications demonstrated the versatility of our framework. We also presented further discussions about several unlisted applications or extensions as future work.

\section{Acknowledgment}
The authors would like to thank Shaohui Yang for his profound insight into the experiment design and applications of this framework, Hongkai Ye and Yuwei Wu for their help in the benchmark, and Yuman Gao, Tiankai Yang, and Neng Pan for the hardware platform for high-speed flights.

\appendix

\subsection{Proof of Sufficiency in Theorem \ref{thm:OptimalityConditions}}
\label{apd:OptimalityConditionsProof}
\begin{proof}
    We consider the space of $M$-piece polynomial $2s$-order splines defined over $[t_0, t_M]$ where consecutive pieces on any $x:[t_0, t_M]\mapsto\mathbb{R}$ satisfy $x_{i-1}^{(j)}(t_i)=x_i^{(j)}(t_i)$ for $0\leq j<\bar{d}_i$ and $1\leq i<M$.  In (\ref{eq:MultistageMinimumControl}), $d_i\leq s$ holds for each $i$. For brevity, we define $D_{i,j}$ as $D_{i,j}=i\cdot s+\sum_{k=1}^{j}{d_k}$. According to Theorem 4.4 in~\cite{Schumaker2007SplineFBT}, this spline space is actually a linear space of dimension $\bar{D}=D_{2,M-1}$.

    Moreover, an explicit basis of the space exists. Based on the original partition $t_0<t_1<\dots<t_M$, we define an \textit{extended partition} $\bar{t}_1\leq\bar{t}_2\leq\dots\leq\bar{t}_{\bar{M}}$ of length $\bar{M}=D_{4,M-1}$ as
    \begin{equation}
    \label{eq:ExtendedPartition}
    \bar{t}_i=\begin{cases} t_0 & \mathit{if}~1\leq i\leq D_{2,0}, \\ t_j & \mathit{if}~D_{2,j-1}<i\leq D_{2,j}, \\ t_M & \mathit{if}~D_{2,M-1}<i\leq\bar{M}. \end{cases}
    \end{equation}
    Based on this extended partition, Theorem 4.9 in~\cite{Schumaker2007SplineFBT} explicitly constructs $\bar{D}$ functions $\cbrac{B_i(t):[t_0,t_M]\mapsto\mathbb{R}}_{i=1}^{\bar{D}}$ which form a basis for the considered spline space.

    Now we consider (\ref{eq:InitialTerminalConditions}) and (\ref{eq:IntermediateConditions}) in the spanned linear space. These conditions specify derivative values on timestamps of the original partition to be interpolated by the basis $\cbrac{B_i(t)}_{i=1}^{\bar{D}}$. We only needs the specified orders along with their timestamps instead of the specified derivative values. Denote by $\tau_i$ the $i$-th specified timestamps, where
    \begin{equation}
    \label{eq:SpecifiedStamps}
    \tau_i=\begin{cases} t_0 & \mathit{if}~1\leq i\leq D_{1,0}, \\ t_j & \mathit{if}~D_{1,j-1}<i\leq D_{1,j}, \\ t_M & \mathit{if}~D_{1,M-1}<i\leq\bar{D}. \end{cases}
    \end{equation}
    Denote by $\nu_i$ the specified order at $\tau_i$, written as
    \begin{equation}
    \label{eq:SpecifiedOrderSequence}
    \nu_i=\begin{cases} i-1 & \mathit{if}~1\leq i\leq D_{1,0}, \\ i-1-D_{1,j-1} & \mathit{if}~D_{1,j-1}<i\leq D_{1,j}, \\ i-1-D_{1,M-1} & \mathit{if}~D_{1,M-1}<i\leq\bar{D}. \end{cases}
    \end{equation}
    Then, the conditions (\ref{eq:InitialTerminalConditions}) and (\ref{eq:IntermediateConditions}) generate a linear equation system on the basis, whose coefficient matrix is
    \begin{equation}
    \mathbf{B}=
    \begin{pmatrix}
    B_1^{(\nu_1)}(\tau_1) & B_2^{(\nu_1)}(\tau_1) & \cdots & B_{\bar{D}}^{(\nu_1)}(\tau_1) \\
    B_1^{(\nu_2)}(\tau_2) & B_2^{(\nu_2)}(\tau_2) & \cdots & B_{\bar{D}}^{(\nu_2)}(\tau_2) \\
    \vdots & \vdots & \ddots & \vdots \\
    B_1^{(\nu_{\bar{D}})}(\tau_{\bar{D}}) & B_2^{(\nu_{\bar{D}})}(\tau_{\bar{D}}) & \cdots & B_{\bar{D}}^{(\nu_{\bar{D}})}(\tau_{\bar{D}})
    \end{pmatrix}.
    \end{equation}
    It is obvious that $\mathbf{B}$ is a square matrix for any possible solution to Theorem \ref{thm:OptimalityConditions} in each dimension.

    According to Theorem 4.67 in~\cite{Schumaker2007SplineFBT}, $\mathbf{B}$ is nonsingular if and only if
    \begin{equation}
    \label{eq:NonsingularityCondition}
    \tau_i\in\delta_i=\begin{cases} [\bar{t}_i, \bar{t}_{i+2s}) & \mathit{if}~\nu_i+\alpha_i-2s\geq0, \\ (\bar{t}_i, \bar{t}_{i+2s}) & \mathit{if}~\nu_i+\alpha_i-2s<0, \end{cases}
    \end{equation}
    holds for any $i=1,\dots,\bar{D}$, where $\alpha_i$ is defined as
    \begin{equation}
    \alpha_i=\cbrac{\mathit{max}~j~:~\bar{t}_i=\dots=\bar{t}_{i+j-1}}.
    \end{equation}
    We show that (\ref{eq:NonsingularityCondition}) is always true in our case. It is obvious that $\alpha_i$ can be computed as
    \begin{equation}
    \label{eq:OffsetSequence}
    \alpha_i=\begin{cases} D_{2,0}-i+1 & \mathit{if}~1\leq i\leq D_{2,0}, \\ D_{2,j}-i+1 & \mathit{if}~D_{2,j-1}<i\leq D_{2,j}. \end{cases}
    \end{equation}
    Combining (\ref{eq:SpecifiedOrderSequence}) and (\ref{eq:OffsetSequence}), we know that $\nu_i<s$ and $\alpha_i\leq s$ always hold for $i>s$, which means
    \begin{equation}
    \begin{cases}
    \nu_i+\alpha_i-2s=0 & \mathit{if}~1\leq i\leq s,\\
    \nu_i+\alpha_i-2s<0 & \mathit{if}~s<i\leq\bar{D}.
    \end{cases}
    \end{equation}
    Thus, the interval $\delta_i$ is computed as
    \begin{equation}
    \delta_i=\begin{cases} [\bar{t}_i, \bar{t}_{i+2s}) & \mathit{if}~1\leq i\leq s, \\ (\bar{t}_i, \bar{t}_{i+2s}) & \mathit{if}~s<i\leq\bar{D}. \end{cases}
    \end{equation}
    Consequently, we have
    \begin{equation}
    \label{eq:FirstCaseInCondition}
    \tau_i=t_0\in[t_0,t_1)\subseteq[\bar{t}_i,\bar{t}_{i+2s})=\delta_i,~1\leq i\leq s.
    \end{equation}
    When $i>s$, we denote $\bar{t}_i=t_k$, $\bar{t}_{i+2s}=t_l$ and $\tau_i=t_j$. As is shown in (\ref{eq:ExtendedPartition}) and (\ref{eq:SpecifiedStamps}), we have
    \begin{equation}
    D_{2,k-1}<i,~(i+2s)\leq D_{2,l},~D_{1,j-1}<i\leq D_{1,j}.
    \end{equation}
    Due to the fact that $d_i\leq s$ holds for any $1\leq i<M$, the following two inequalities always hold.
    \begin{equation}
    \label{eq:BoundOnK}
    D_{2,k-1}<i\leq D_{1,j}=(D_{2,j}-s)\leq D_{2,j-1},
    \end{equation}
    \begin{equation}
    \label{eq:BoundOnL}
    D_{2,j}=(D_{1,j}+s)\leq(D_{1,j-1}+2s)<(i+2s)\leq D_{2,l}.
    \end{equation}
    Inequalities (\ref{eq:BoundOnK}) and (\ref{eq:BoundOnL}) imply $k<j$ and $j<l$, thus
    \begin{equation}
    \label{eq:SecondCaseInCondition}
    \tau_i=t_j\in(t_k,t_l)=(\bar{t}_i,\bar{t}_{i+2s})=\delta_i,~s<i\leq\bar{D},
    \end{equation}
    always holds. Combining (\ref{eq:FirstCaseInCondition}) and (\ref{eq:SecondCaseInCondition}) gives (\ref{eq:NonsingularityCondition}). Therefore, the coefficient matrix $\mathbf{B}$ on basis is always nonsingular for settings on the original problem, implying the existence and uniqueness of solution.

    The optimality conditions guarantee one unique solution in each decoupled dimension, which gives its sufficiency.
\end{proof}

\subsection{Proof of Proposition \ref{ps:DiffeomorphismKeepsLocalMin}}
\label{apd:DiffeomorphismKeepsLocalMin}
\begin{proof}
    Denote by $\mathbf{J}$ the Jacobian of $\mathbf{G}$. For any $x\in\mathbb{D}_\mathrm{F}$ and $y\in\mathbb{R}^N$, satisfying $x=\mathbf{G}(y)$ or $y=\mathbf{G}^{-1}(x)$, we have
    \begin{equation}
    \grad{H(y)}=\mathbf{J}(y)\tp\grad{F(x)}.
    \end{equation}
    Then, the nonsingularity of $\mathbf{J}$ implies that the first statement always holds. Denote by $\mathbf{K}_i$ the Hessian of the $i$-th entry in $\mathbf{G}$. If $x$ and $y$ are stationary points, the Hessian of $H$ is
    \begin{align}
    \hessian{H(y)}&=\mathbf{J}(y)\tp\hessian{F(x)}\mathbf{J}(y)+\sum_{i=1}^{N}{\frac{\partial F(x)}{\partial x_i}\mathbf{K}_i(y)} \nonumber\\
    &=\mathbf{J}(y)\tp\hessian{F(x)}\mathbf{J}(y).
    \end{align}
    Then, the nonsingular $\mathbf{J}$ implies that $\hessian{F(x)}$ and $\hessian{H(y)}$ are congruent~\cite{Horn2012MatrixA}. Thus the second statement holds.
\end{proof}

\bibliography{references}

\vspace{-1.2cm}

\begin{IEEEbiography}[{\includegraphics[width=1in,height=1.25in,clip,keepaspectratio]{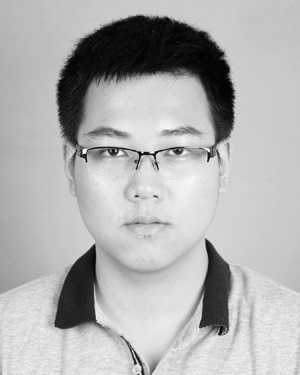}}]{Zhepei Wang}
    received the B.Eng. degree in control science and engineering from Zhejiang University, Hangzhou, China, in 2017.

    He is currently working toward the Ph.D. degree in control science and engineering from Zhejiang University, Hangzhou, China. His research interests include motion planning, discrete and computational geometry, numerical optimization, and autonomous navigation of unmanned vehicles.
\end{IEEEbiography}

\vspace{-1.2cm}

\begin{IEEEbiography}[{\includegraphics[width=1in,height=1.25in,clip,keepaspectratio]{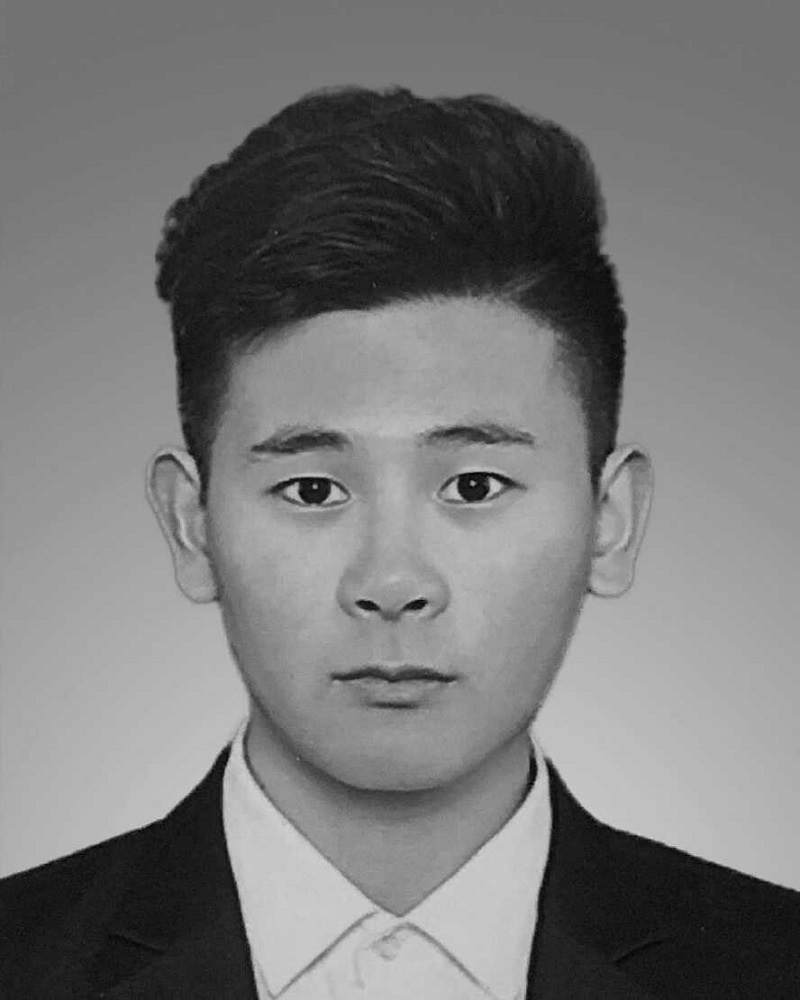}}]{Xin Zhou}
    received the B.Eng. degree in electrical engineering and automation from China University of Mining and Technology, Xuzhou, China, in 2019.

    He is currently working toward the Ph.D. degree in control engineering from Zhejiang University, Hangzhou, China. His research interests include motion planning and mapping for aerial swarm robotics.
\end{IEEEbiography}

\vspace{-1.2cm}

\begin{IEEEbiography}[{\includegraphics[width=1in,height=1.25in,clip,keepaspectratio]{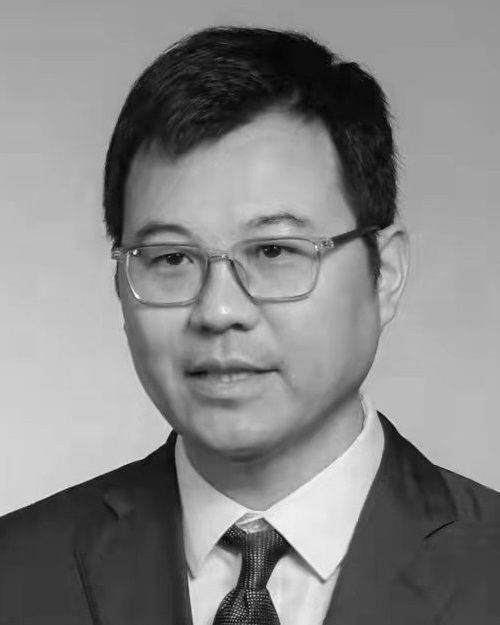}}]{Chao Xu}
    received the Ph.D. degree in mechanical engineering from Lehigh University, Bethlehem, PA, USA, in 2010.

    He is the Professor of Cyber-Systems \& Robotics, the Associate Dean of the College of Control Science and Engineering, and the Founding Dean of Huzhou Institute of Zhejiang University. He founded the FAST (Field Autonomous System and Computing) Lab. His research interests are robot mechanics and control. He is currently the Managing-Editor for the Journal of Industrial and Management Optimization, and the Founding Managing-Editor for IET Cyber-Systems \& Robotics.
\end{IEEEbiography}

\vspace{-1.0cm}

\begin{IEEEbiography}[{\includegraphics[width=1in,height=1.25in,clip,keepaspectratio]{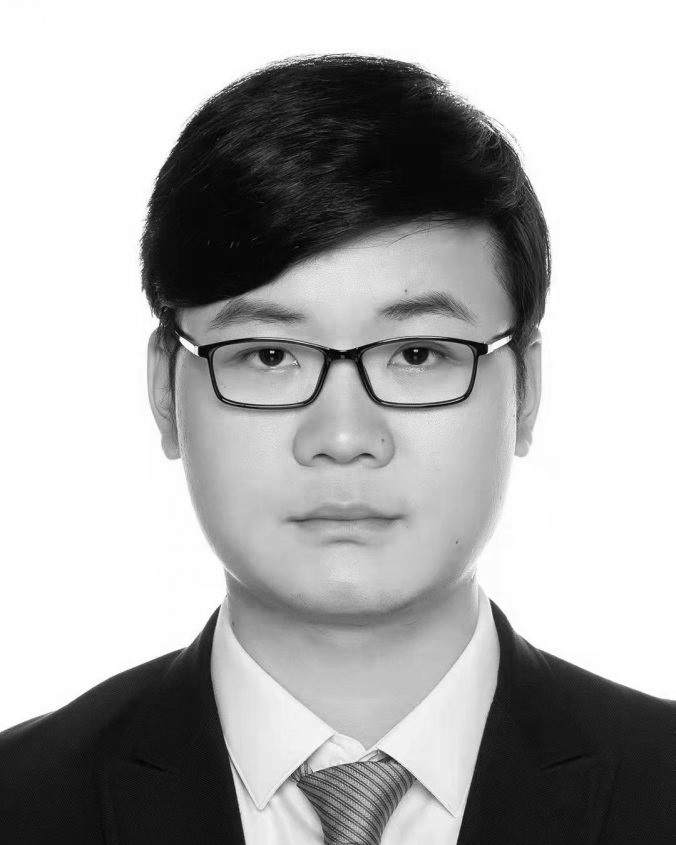}}]{Fei Gao}
    received the Ph.D. degree in electronic and computer engineering from the Hong Kong University of Science and Technology, Hong Kong, in 2019.

    He is currently an Assistant Professor with the College of Control Science and Engineering, Zhejiang University, where he co-directs the FAST (Field Autonomous System and Computing) Lab and leads the FAR (Flying Autonomous Robotics) Group. His research interests include aerial robots, swarms, autonomous navigation, motion planning, and localization and mapping.
\end{IEEEbiography}

\end{document}